%% file: neutral.tex
\documentclass{article}

\usepackage[T1]{fontenc}
\usepackage[htt]{hyphenat}

\usepackage{microtype}
\usepackage{graphicx}
\usepackage{subfigure}
\usepackage{booktabs} 
\usepackage{natbib}
\input{preamble}

\usepackage{breakcites}
\usepackage{authblk}
\usepackage[margin=1in]{geometry}
\usepackage{enumitem}

\usepackage{hyperref}

\hypersetup{
     colorlinks = true,
     linkcolor = brown,
     anchorcolor = blue,
     citecolor = teal,
     filecolor = blue,
     urlcolor = black
     }

\title{A mean-field analysis of two-player zero-sum games}

\author[a]{Carles Domingo-Enrich}
\author[a,d]{Samy Jelassi}
\author[e]{Arthur Mensch}
\author[a]{\\Grant M. Rotskoff}
\author[a, b, c]{Joan Bruna}
\affil[a]{Courant Institute of Mathematical Sciences, New York University, New York}
\affil[b]{Center for Data Science, New York University, New York}
\affil[c]{Institute for Advanced Study, Princeton}
\affil[d]{Princeton University, Princeton}
\affil[e]{ENS Paris, France}

\begin{document}





\maketitle

\begin{abstract}
  \input{abstract.tex}
\end{abstract}

\input{introduction.tex}

\input{related_work.tex}
\input{problem_setup.tex}
\input{LGV_section.tex}

\input{WFR_section.tex}

\input{propagation_chaos.tex}

\input{experiments.tex}

\input{future_work.tex}

\pagebreak

\input{broader_impact.tex}

\bibliography{biblio}

\onecolumn

\appendix
\startcontents[sections]

\input{app_lifted.tex}
\input{app_nikaido_continuity.tex}
\input{app_thm1.tex}
\input{app_thm2.tex}
\input{app_thm3.tex}
\input{app_thm4.tex}
\input{app_thm5.tex}
\input{app_thm6.tex}
\input{app_misc.tex}
\input{app_SDEs_manifolds.tex}



\end{document}

%% file: preamble.tex
\usepackage{amsmath}
\usepackage{amssymb}
\usepackage{graphicx}
\usepackage{mathtools}
\usepackage{amsfonts}
\usepackage{amsthm,bm}
\usepackage{color}
\usepackage[table]{xcolor}
\usepackage{enumerate}
\usepackage{dsfont}
\usepackage{pgfplots}
\pgfplotsset{width=10cm,compat=1.9}
 \usepackage{pifont}
\usepackage{thmtools} 
\usepackage{thm-restate}
\usepackage{sidecap}

\usepackage{algorithm}
\usepackage[noend]{algpseudocode}
\algrenewcommand{\algorithmiccomment}[1]{$\vartriangleright$ #1}
\algrenewcommand{\algorithmicreturn}{\textbf{Return: }}
\algnewcommand\algorithmicinput{\textbf{Input: }}
\algnewcommand\Input{\State \algorithmicinput}

\usepackage{titletoc}

\graphicspath {{figures/}}

\usepackage{hyperref}
\usepackage{caption}
\usepackage[super]{nth}
\newtheorem{lemma}{Lemma}
\newtheorem{definition}{Definition}
\newtheorem{theorem}{Theorem}

\newtheorem{assumption}{Assumption}

\newtheorem{remark}{Remark}
\newtheorem{corollary}{Corollary}

\mathtoolsset{showonlyrefs}

\input{commands.tex}

%% file: commands.tex
\newcommand{\EE}{\mathbb{E}}

\newcommand{\Cc}{\mathcal{C}}

\newcommand{\Gg}{\mathcal{G}}

\newcommand{\Ll}{\mathcal{L}}

\newcommand{\Nn}{\mathcal{N}}

\newcommand{\Pp}{\mathcal{P}}

\newcommand{\Xx}{\mathcal{X}}
\newcommand{\Yy}{\mathcal{Y}}

\newcommand{\xx}{x}
\newcommand{\yy}{y}

\newcommand{\hmu}{\hat{\mu}}

\renewcommand{\phi}{\varphi}

\DeclareMathOperator{\supp}{supp}

\renewcommand{\leq}{\leqslant}
\renewcommand{\geq}{\geqslant}

\renewcommand{\epsilon}{\varepsilon}
\renewcommand{\imath}{\mathrm{i}}

\DeclareMathOperator*{\argmin}{argmin}
\DeclareMathOperator*{\argmax}{argmax}

\usepackage{accents}

\newcommand{\arthur}[1]{{\color{blue}[AM: #1]}}

\newlength{\restsubwidth}
\newlength{\restsubheight}
\newlength{\restsubmoreheight}
\newcommand{\rest}[2]{%
        \settowidth{\restsubwidth}{\ensuremath{#2}}
        \settoheight{\restsubheight}{\ensuremath{{}_{#2}}}
        \ensuremath{{#1\hskip 0.5pt}_{\vrule\kern2pt\parbox[b][%
        4pt][b]{\the\restsubwidth}{%
                        \ensuremath{{}_{#2}}}}}
        }

%% file: abstract.tex
Finding Nash equilibria in two-player zero-sum continuous games is a central problem in machine learning, e.g. for training both GANs and robust models. The existence of pure Nash equilibria requires strong conditions which are not typically met in practice. Mixed Nash equilibria exist in greater generality and may be found using mirror descent. Yet this approach does not scale to high dimensions. To address this limitation, we parametrize mixed strategies as mixtures of particles, whose positions and weights are updated using gradient descent-ascent. We study this dynamics as an interacting gradient flow over measure spaces endowed with the Wasserstein-Fisher-Rao metric. We establish global convergence to an approximate equilibrium for the related Langevin gradient-ascent dynamic. We prove a law of large numbers that relates particle dynamics to mean-field dynamics. Our method identifies mixed equilibria in high dimensions and is demonstrably effective for training mixtures of GANs.

%% file: introduction.tex
\section{Introduction}

Multi-objective optimization problems arise in many fields, from economics to civil engineering. Tasks that require optimizing multiple objectives have also become a routine part of many agent-based machine learning algorithms including generative adversarial networks \citep{goodfellow2014generative}, imaginative agents \citep{racaniere2017imagination}, hierarchical reinforcement learning \citep{wayne2014hierarchical} and multi-agent reinforcement learning \citep{bu2008comprehensive}. It not only remains difficult to carry out the necessary optimization, but also to assess the optimality of a given solution.

Multi-agent optimization is generally cast as finding equilibria in the space of strategies. The classic notion of equilibrium is due to Nash~\citep{nash1951non}: a Nash equilibrium is a set of agent strategies for which no agent can unilaterally improve its loss value. Pure Nash equilibria, in which each agent adopts a single strategy, provide a limited notion of optimality because they exist only under restrictive conditions. On the other hand, mixed Nash equilibria (MNE), where agents adopt a strategy from a probability distribution over the set of all strategies, exist in much greater generality~\citep{glicksberg1952further}. Importantly, MNE exist for games with infinite-dimensional compact strategy spaces, in which each player observes a loss function that is continuous in its strategy. We encounter this setting in different game formulations of machine learning problems, like GANs \citep{goodfellow2014generative}.

Although MNE are guaranteed to exist, it is difficult to identify them. Indeed, worst-case complexity analyses have shown that without additional assumptions on the losses there is no efficient algorithm for finding a MNE, even in the case of two-player finite games~\citep{daskalakis2009complexity}.
Some recent progress has been made; \citet{hsiehfinding2019} proposed a mirror-descent algorithm with convergence guarantees, which is approximately realizable in high-dimension. 

\paragraph{Contributions.} Following \citet{hsiehfinding2019}, we formulate continuous two-player zero-sum games as a multi-agent optimization problem over the space of probability measures on strategies. We describe two gradient descent-ascent dynamics in this space, both involving a transport term. 
\begin{itemize}
    \item We show that the stationary points of a gradient ascent-descent flow with Langevin diffusion over the space of mixed strategies are approximate MNE.
    \item We analyse a gradient ascent-descent dynamics that jointly updates the positions and weights of two mixed strategies to converge to an \textit{exact} MNE. This dynamics corresponds to a gradient descent-ascent flow over the space of measures endowed with a Wasserstein-Fisher-Rao (WFR) metric \citep{chizat_interpolating_2018}.
    \item We discretize both dynamics in space and time to obtain implementable training algorithms. We provide mean-field type consistency results on the discretization. We demonstrate numerically how both dynamics overcome the curse of dimensionality for finding MNE on synthetic games. On real data, we use WFR flows to train mixtures of GANs, that explicitly discover data clusters while maintaining good performance.
\end{itemize}

%% file: related_work.tex
\section{Related work}
\paragraph{Equilibria in continuous games.}
Most of the works that study convergence to equilibria in continuous games or GANs do not frame the problem in the infinite-dimensional space of measures, but on finite-dimensional spaces. That is because they either (i) restrict their attention to games with convexity-concavity assumptions in which pure equilibria exist~\citep{optimistic,lin2018solving,nouiehed2019solving}, or (ii) provide algorithms with convergence guarantees to local notions of equilibrium such as stable fixed points, local Nash equilibria and local minimax points~\citep{heusel2017gans,adolphs2018local,mazumdar2019finding,jin2019minmax,fiez2019convergence,mechanics}. Both approaches differ from ours, which is to give global convergence guarantees without convexity assumptions.  
Some works have studied approximate MNE in infinite-dimensional measure spaces. \citet{arora2017generalization} proved the existence of approximate MNE and studied the generalization properties of this approximate solution; their analysis, however, does not provide a constructive method to identify such a solution. In a more explicit setting, \citet{grnarova2017online} designed an online-learning algorithm for finding a MNE in GANs under the assumption that the discriminator is a single hidden layer neural network. \citet{balandat2016minimizing}
apply the dual averaging algorithm to the minimax problem and show that it
recovers a MNE, but they do not provide any convergence rate nor a
practical algorithm for learning mixed NE. Our framework holds without making any assumption on the architectures of the discriminator and generator and provides explicit algorithms with some convergence guarantees.

\paragraph{Mean-field view of nonlinear gradient descent.}
Our approach is closely related to the mean-field perspective on wide neural
networks
\citep{mei2018mean,rotskoff2018neural,chizat2018global,sirignano2019mean,rotskoff2019global}.
These methods view training algorithms as approximations of Wasserstein gradient flows, which are dynamics on measures over the space of neurons. In our setting, a mixed strategy corresponds to a measure over the space
of strategies. 

\paragraph{Particle approaches for two-player games.} 
Our theoretical work sheds a new light on the results of \citet{hsiehfinding2019}, and rigorously justifies important algorithmic modifications the authors introduced. Specifically, they give rates of convergence for infinite-dimensional mirror descent on measures (i.e. updating strategy weights but not their positions). The straightforward implementation of this algorithm performs poorly unless the dimension is low (\autoref{fig:polygames}), which is why they proposed an `implementable` two-timescale version, in which the inner loop is a transport-based sampling procedure  closely related to our Algorithm \ref{alg:langevin_dynamics}. This implementable version is not studied theoretically, as the two-timescale structure hinders a thorough analysis. Our analysis includes transport on equal footing with mirror descent updates.

%% file: problem_setup.tex
\section{Problem setup and mean-field dynamics}

\paragraph{Notation.}  For a topological space $\mathcal{X}$ we denote by $\mathcal{P}(\mathcal{X})$ the space of Borel probability measures on $\mathcal{X}$, and $\mathcal{M}_{+}(\mathcal{X})$ the space of Borel (positive) measures. For a given measure $\mu \in \mathcal{P}(\mathcal{X})$ that is absolutely continuous with respect to the canonical Borel measure $d x$ of $\mathcal{X}$ and has Radon-Nikodym derivative $\frac{d\mu}{dx} \in \Cc(\Xx)$, we define its differential entropy $H(\mu)=-\int \log(\frac{d\mu}{dx})d\mu$. For measures $\mu, \nu \in \mathcal{P}(\mathcal{X})$, $\mathcal{W}_2$ is the 2-Wasserstein distance.

\subsection{Lifting differentiable games to spaces of strategy distributions} 
 
\paragraph{Differentiable two-player zero-sum games.} We recall the definition of a differentiable zero-sum game, and show how finding a mixed Nash equilibrium to such a game is equivalent to solving a bi-linear game in the infinite dimensional space of distributions on strategies. We will use gradient flow approaches for solving the lifted problem.
\begin{definition}
    A two-player zero-sum game consists of a set of two players with parameters
    $z=(x,y)\in \mathcal{Z} = \mathcal{X}\times\mathcal{Y}$, where players observe a loss functions $\ell_1\colon \mathcal{Z}\rightarrow \mathbb{R}$ and
    $\ell_2\colon \mathcal{Z}\rightarrow \mathbb{R}$ that satisfy for all
    $(x,y)\in \mathcal{Z}$, $\ell_1(x,y)+\ell_2(x,y)=0.$ $\ell\triangleq \ell_1=-\ell_2$ is the loss of the game.
\end{definition}
 The compact finite-dimensional spaces of strategies $\Xx$ and $\Yy$ are endowed with a certain distance function $d$ (which we assume Euclidean in what follows---\autoref{sec:sdes_manifolds} derives our results on arbitrary strategy manifolds). This allows to define differentiable games, amenable to first-order optimization. We make the following mild assumption over the regularity of losses and constraints \citep{glicksberg1952further}.
 
 \begin{assumption}\label{ass:continuous_game}
    The parameter spaces $\mathcal{X}$ and $\mathcal{Y}$ are compact Riemannian
    manifolds without boundary of dimensions $d_x, d_y$ embedded in
    $\mathbb{R}^{D_x}, \mathbb{R}^{D_y}$ respectively. The loss $\ell$
    is continuously differentiable and $L$-smooth with respect to each
    parameter. That is, for all $x,x' \in \mathcal{X}$ and $y, y' \in
    \mathcal{Y}$, ${\|\nabla_x\ell(x,y)-\nabla_x\ell(x',y')\|}_2\leq L(d(x,x') +
    d(y,y')), \ {\|\nabla_y\ell(x,y)-\nabla_y\ell(x',y')\|}_2\leq L(d(x,x') +
    d(y,y'))$.
\end{assumption}



\paragraph{From pure to mixed Nash equilibria.} Assuming that both players play simultaneously, a pure Nash equilibrium point is a pair of strategies 
$(x^*,y^*) \in \Xx \times \Yy$ such that, for all $(x, y) \in \Xx \times \Yy$, $\ell(x^\star, y) \leq \ell(x^\star, y^\star) \leq \ell(x, y^\star)$. Such points do not always exist in continuous games. In contrast,
mixed Nash equilibria (MNE) are guaranteed to exist~\citep{glicksberg1952further} under
\autoref{ass:continuous_game}. Those distributions $(\mu_x^\star, \mu_y^\star) \in
\mathcal{P}(\mathcal{X})\times\mathcal{P}(\mathcal{Y})$ are global saddle points of the expected loss 
$\mathcal{L}(\mu_x,\mu_y) \triangleq \iint \ell(x,y)d\mu_x(x)d\mu_y(y)$. Formally, for all
$\mu_x, \mu_y\in \mathcal{P}(\mathcal{X}) \times \mathcal{P}(\mathcal{Y})$,
\begin{equation}\label{def:mixed_nash}
         \Ll(\mu_x^*, \mu_y) \leq 
        \Ll(\mu_x^*, \mu_y^*) \leq \Ll(\mu_x, \mu_y^*).
\end{equation}
We quantify the accuracy of an
estimation $(\hat \mu_x,\hat \mu_y)$ of a MNE using the \citet{nikaido_note_1955} error
\begin{align} \label{eq:nikaido_def}
    \mathrm{NI}(\hat \mu_x,\hat \mu_y) = \sup_{\mu_y \in \mathcal{P}(\mathcal{Y})} \mathcal{L}(\hat \mu_x,\mu_y)-\!\!\!\inf_{\hat \mu_x \in \mathcal{P}(\mathcal{X})} \mathcal{L}(\mu_x,\hat \mu_y).
\end{align}
We track the evolution of this metric in our
theoretical results (\autoref{sec:wfr_sec}) and in our experiments. We obtain guarantees on finding $\varepsilon$-MNE $(\mu_x^\varepsilon, \mu_y^\varepsilon)$, i.e. distribution pairs such that $\mathrm{NI}(\mu_x^\varepsilon, \mu_y^\varepsilon) \leq \varepsilon$.


\subsection{Training dynamics on discrete mixtures of strategies}
\label{subsec:dynamics}

We study three different dynamics for solving \eqref{def:mixed_nash}. Let us first assume that the two players play \textit{finite} mixtures of $n$ strategies $\mu_x = \sum_{i=1}^n w^{i}_x \delta_{x^i} \in \mathcal{P}(\mathcal{X})$, $\mu_y = \sum_{i=1}^n w^{i}_y \delta_{y^i} \in \mathcal{P}(\mathcal{Y})$, where $\{x^i, y^i\}_{i \in [1:n]}$ are the positions of the strategies and $w^{i}_x, w^{i}_y \geq 0$ are their weights. In the simplest setting, those mixtures are assumed uniform, i.e. $w^{i}_x = w^{i}_y = 1/n$. Finding the best $2n$ strategies involve finding a saddle point of 
$\Ll(\mu_x, \mu_y) = \frac{1}{n^2} \sum_i \sum_j \ell(x_i, y_j)$. Starting from random independent initial strategies $x^i_0 = \xi_i \sim \mu_{x,0}, y^i_0 = \bar{\xi}_i \sim \mu_{y,0}$, we may hope that the gradient descent-ascent dynamics
\begin{equation} \label{eq:IWGF_particle_flow}
    \frac{dx^{i}_t}{dt} = - \frac{1}{n} \sum_{j=1}^{n} \nabla_x \ell(x^i_t, y^j_t), \quad \frac{dy^{i}_t}{dt} = \frac{1}{n} \sum_{j=1}^{n} \nabla_y \ell(x^{j}_t, y^{i}_t),  \quad \forall i \in [1:n]
\end{equation}
finds such a saddle point. Yet this may fail in simple
nonconvex-nonconcave games, as illustrated in \autoref{subsec:failure_IWGF}---the particle distributions collapse to a stationary point that is not a MNE.

To mitigate this convergence problem, we analyse a perturbed dynamics analogous to Langevin gradient descent. Using the same initialization as in \eqref{eq:IWGF_particle_flow}, we add a small amount of noise in the gradient dynamics and obtain the stochastic differential equations  
\begin{align}
\begin{split} \label{eq:particle_system_LGV}
    dX_t^{i} = - \frac{1}{n} \sum_{j=1}^{n} \nabla_x \ell(X_t^{i}, Y_t^{j}) dt + \sqrt{\frac{2}{\beta}} dW_t^{i}, \ dY_t^{i} = \frac{1}{n} \sum_{j=1}^{n} \nabla_y \ell(X_t^{j},Y_t^{i}) dt + \sqrt{\frac{2}{\beta}} d\bar{W}_t^{i},
\end{split}
\end{align}
where $W_t^{i}, \bar{W}_t^{i}$ are independent Brownian motions. The discretization of \eqref{eq:particle_system_LGV} results in \autoref{alg:langevin_dynamics}; it is similar to Alg. 4 in \citet{hsiehfinding2019}.

We propose a second alternative dynamics to \eqref{eq:IWGF_particle_flow}, that updates both the positions and the weights of the particles, using relative updates for weights. We will show that it enjoys better convergence properties in the mean-field limit.
\begin{align}
\begin{split} \label{eq:particle_system_WFR}
    \frac{dx_t^{i}}{dt} &= - \gamma \sum_{j=1}^n w_{y,t}^{j} \nabla_x \ell(x_t^{i}, y_t^{j}), \quad \frac{dw_{x,t}^{i}}{dt} = \alpha \left(- \sum_{j=1}^n w_{y,t}^{j} \ell(x_t^{i}, y_t^{j}) + K(t) \right) w_{x,t}^{i}
\end{split}
\end{align}
and similarly for all $y_t^i$ (flipping the sign of $\ell$). $K(t) \triangleq \sum_{k=1}^n \sum_{j=1}^n w_{y,t}^{j} w_{x,t}^{k} \ell(x_t^{i}, y_t^{j})$ keeps $w_{x,t}$ in the simplex. We use uniform weights for initialization. When $\gamma = 0$ and  $\alpha = 1$, only the weights are updated: this results in the continuous-time version of the infinite-dimensional mirror descent studied by \citet{hsiehfinding2019}. The Euler discretization of \eqref{eq:particle_system_WFR} results in \autoref{alg:WFR_dynamics}.

\begin{algorithm}[t]
\caption{Langevin Descent-Ascent (L-DA).} 
\label{alg:langevin_dynamics}
\begin{algorithmic}[1]
\State \textbf{Input}: IID samples $x_0^{1},\dots,x_0^{n}$ from $\mu_{x,0} \in \mathcal{P}(\mathcal{X})$, IID samples $y_0^{1},\dots,y_0^{n} \in \mathcal{Y}$ from $\mu_{y,0} \in \mathcal{P}(\mathcal{Y})$
\For {$t=0,\dots,T$} \For {$i=1,\dots,n$}
        \State Sample $\Delta W_{t}^i \sim \mathcal{N}(0,I)$, $x_{t+1}^i = x_{t}^i - \frac{\eta}{n} \sum_{j=1}^n \nabla_x \ell(x_t^{i}, y_t^{j}) + \sqrt{2 \eta \beta^{-1}} \Delta W_{t}^i$
        \State Sample $\Delta \bar{W}_{t}^i \sim \mathcal{N}(0,I)$, $y_{t+1}^i = y_{t}^i + \frac{\eta}{n} \sum_{j=1}^n \nabla_y \ell(x_t^{j}, y_t^{i}) + \sqrt{2 \eta \beta^{-1}} \Delta \bar{W}_{t}^i$
    \EndFor
\EndFor
\State \textbf{Return} $\mu_{x,T}^n = \frac{1}{n}\sum_{i=1}^n \delta_{x_T^i}, \quad \mu_{y,T}^n = \frac{1}{n}\sum_{i=1}^n \delta_{y_T^i}$
\end{algorithmic}
\end{algorithm}
\begin{algorithm}[t]
\caption{Wasserstein-Fisher-Rao Descent-Ascent (WFR-DA).}
\label{alg:WFR_dynamics}
\begin{algorithmic}[1]
\State \textbf{Input}: IID samples $x_0^{(1)},\dots,x_0^{(n)}$ from $\nu_{x,0} \in \mathcal{P}(\mathcal{X})$, IID samples $y_0^{(1)},\dots,y_0^{(n)}$ from $\nu_{y,0} \in \mathcal{P}(\mathcal{Y})$. Initial weights: For all $i \in [1:n]$, $w_x^{(i)} = 1, \ w_y^{(i)} = 1$.
\For {$t=0,\dots,T$}
        \State $[x_{t+1}^{(i)}]_{i=1}^{n} = [x_{t}^{(i)} - \eta \sum_{j=1}^n w_{y,t}^{(i)} \nabla_x \ell(x_t^{(i)}, y_t^{(j)})]_{i=1}^{n}$
        \State ${\scriptstyle [\hat{w}_{x,t+1}^{(i)}]_{i=1}^{n} = \left[w_{x,t}^{(i)} \exp\left(- \eta' \sum_{j=1}^n w_{y,t}^{(j)} \ell(x_t^{(i)}, y_t^{(j)}) \right) \right]_{i=1}^{n}, \quad  [w_{x,t+1}^{(i)}]_{i=1}^{n} = [\hat{w}_{x,t+1}^{(i)}]_{i=1}^{n}/\sum_{j=1}^n \hat{w}_{x,t+1}^{(j)}}$
        \State $[y_{t+1}^{(i)}]_{i=1}^{n} = [y_{t}^{(i)} + \eta \sum_{j=1}^n w_{x,t}^{(j)} \nabla_y \ell(x_t^{(j)}, y_t^{(i)})]_{i=1}^{n},$
        \State ${\scriptstyle [\hat{w}_{y,t+1}^{(i)}]_{i=1}^{n} = \left[w_{y,t}^{(i)} \exp\left(\eta' \sum_{j=1}^n w_{x,t}^{(j)} \ell(x_t^{(j)}, y_t^{(i)})\right) \right]_{i=1}^{n}, \quad [w_{y,t+1}^{(i)}]_{i=1}^{n} = [\hat{w}_{y,t+1}^{(i)}]_{i=1}^{n}/\sum_{j=1}^n \hat{w}_{y,t+1}^{(j)}}$
\EndFor
\State \textbf{Return} $\bar{\nu}_{x,T}^n = \frac{1}{T+1} \sum_{t=0}^T \sum_{i=1}^n w_{x,T}^{(i)} \delta_{x_T^{(i)}}, \quad \bar{\nu}_{y,T}^n = \frac{1}{T+1} \sum_{t=0}^T \sum_{i=1}^n w_{y,T}^{(i)} \delta_{y_T^{(i)}}$
\end{algorithmic}
\end{algorithm}

\subsection{Training dynamics as gradient flows on measures} 
The three dynamics that we have introduced at the level of particles induces dynamics on the
associated empirical probability measures. If $\{ x^i_t, y^i_t\}_{i \in [1,n]}$ is a solution of \eqref{eq:IWGF_particle_flow}, then $\mu_x(t)=\frac{1}{n}\sum_{i=1}^n \delta_{x^i_t}$
and $\mu_y(t)=\frac{1}{n}\sum_{i=1}^n \delta_{y^i_t}$ are solutions of the \textit{Interacting Wasserstein Gradient Flow} (IWGF) of $\mathcal{L}$:
\begin{equation}
    \begin{split} \label{eq:interacting_flow}
        \begin{cases} 
        \partial_t \mu_x = \nabla \cdot (\mu_x \nabla_x V_x(\mu_y,x)), \quad \mu_x(0) = \frac{1}{n}\sum_{i=1}^n \delta_{x^i_0},\\
        \partial_t \mu_y = -\nabla \cdot (\mu_y \nabla_y V_y(\mu_x,y)), \quad \mu_y(0) =  \frac{1}{n}\sum_{i=1}^n \delta_{y^i_0}. 
        \end{cases}
    \end{split}
\end{equation} 
The derivation of \eqref{eq:interacting_flow} is provided in \autoref{subsec:link_IWGF}. We use the notation
$V_x(\mu_y,x) \triangleq \frac{\delta \mathcal{L}}{\delta \mu_x}(\mu_x,\mu_y)(x) = \int \ell(x,y)d\mu_y(y)$ for the first variations of the functional $\mathcal{L}(\mu_x,\mu_y)$. Holding $\mu_y$ fixed, the evolution of $\mu_x$ is a Wasserstein gradient flow
on $\mathcal{L}(\cdot, \mu_y)$~\citep{ambrosio_gradient_2005}. We interpret these PDEs in the weak sense, i.e. equality holds when integrating measures against bounded continuous functions.

The distributions $\mu_x(t) = \frac{1}{n} \sum_{i=1}^n \delta_{X^i_t}$
and $\mu_y(t)= \frac{1}{n} \sum_{i=1}^n \delta_{Y^i_t}$, where $\{X^i, Y^i\}_{i \in [1:n]}$ are solutions of~\eqref{eq:particle_system_LGV} follows a \textit{Entropy-Regularized Interacting Wasserstein Gradient Flow (ERIWGF)}:
\begin{align}
    \begin{split} \label{eq:entropic_interacting_flow}
        \begin{cases} 
        \partial_t \mu_x = \nabla_x \cdot (\mu_x \nabla_x V_x(\mu_y,x)) + \beta^{-1} \Delta_x \mu_x, \quad \mu_x(0) = \frac{1}{n}\sum_{i=1}^n \delta_{x^i_0} \\
        \partial_t \mu_y = -\nabla_y \cdot (\mu_y \nabla_y V_y(\mu_x,y)) + \beta^{-1} \Delta_y \mu_y, \quad \mu_y(0) = \frac{1}{n}\sum_{i=1}^n \delta_{y^i_0}
        \end{cases}
    \end{split}
\end{align}
The derivation of \eqref{eq:entropic_interacting_flow} is provided in \autoref{lem:forward_kolmogorov_entropy}. It is a system of coupled nonlinear Fokker-Planck equations, that are the Kolmogorov forward equations of the SDE \eqref{eq:particle_system_LGV}. They correspond to the IWGF of the entropy-regularized loss $\mathcal{L}_{\beta}(\mu_x,\mu_y)\triangleq \mathcal{L}(\mu_x,\mu_y) +\beta^{-1}(H(\mu_y) -H(\mu_x)$. 

Finally, if $\{x^i, y^i, w_x^i, w_y^i\}_{i \in [1:n]}$ solve \eqref{eq:particle_system_WFR}, then $\mu_x(t) =\sum_{i=1}^n w^{i}_{x,t} \delta_{x^i_t}$, $\mu_y(t) = \sum_{i=1}^n w^{i}_{y,t} \delta_{y^i_t}$ solve the \textit{Interacting Wasserstein-Fisher-Rao Gradient Flow (IWFRGF)} of $\mathcal{L}$:
\begin{align}
    \begin{split}
        \begin{cases} \label{eq:WFR_eq}
        \partial_t \mu_x \!\!\!\!\! &= \gamma \nabla_x \cdot (\mu_x \nabla_x V_x(\mu_y,x)) - \alpha \mu_x(V_x(\mu_y,x) - \mathcal{L}(\mu_x,\mu_y)), \: \mu_x(0) = \sum_{i=1}^n w^{i}_{x,0} \delta_{x^i_0},\\
        \partial_t \mu_y &= -\gamma \nabla_y \cdot (\mu_y \nabla_y V_y(\mu_x,y)) + \alpha \mu_y(V_y(\mu_x,y) -  \mathcal{L}(\mu_x,\mu_y)), \: \mu_y(0) = \sum_{i=1}^n w^{i}_{y,0} \delta_{y^i_0}.
        \end{cases}
    \end{split}
\end{align}

The derivation of \eqref{eq:WFR_eq} is provided in \autoref{subsec:lifted_WFR} and \autoref{lem:forward_kolmogorov_wfr}.
The Wasserstein-Fisher-Rao or Hellinger-Kantorovich metric \citep{chizat2015unbalanced, kondratyev2016new,gallout2016AJS} is a metric on the probability space $\mathcal{M}_{+}(\mathcal{X})$ induced by a lifting to the space $\mathcal{P}(\mathcal{X} \times \mathbb{R}^{+})$ of the form $\nu \mapsto \mu = \int_{\mathbb{R}^{+}} w \ d\nu(\cdot,w)$. If we keep $\nu_y$ fixed, the first equation in \eqref{eq:WFR_eq} is a Wasserstein-Fisher-Rao gradient flow (slightly modified by the term $\alpha \mu_x \mathcal{L}(\mu_x,\mu_y)$ to constrain $\mu_x$ in $\mathcal{P}(\mathcal{X})$). The term $- \alpha \mu_x(V_x(\mu_y,x) - \mathcal{L}(\mu_x,\mu_y))$, which also arises in entropic mirror descent, allow mass to `teleport' from bad strategies to better ones with finite cost by moving along the weight coordinate. Wasserstein-Fisher-Rao gradient flows have been used by \citet{ chizat2019sparse, rotskoff2019global, Liero2018} in the context of optimization. 

Initialization of \eqref{eq:interacting_flow}, \eqref{eq:entropic_interacting_flow} and \eqref{eq:WFR_eq} may be done with the measures $\mu_{x,0}$ and $\mu_{y,0}$ from which $\{x_0^i\}, \{y_0^i\}$ are sampled, in which case the measures $\mu_{x}(t)$ and $\mu_{y}(t)$ are not discrete and follow the \textit{mean-field} dynamics. In \autoref{sec:mean-field} we link the dynamics starting from discrete realizations to the mean-field dynamics. 

%% file: LGV_section.tex
\section{Convergence analysis}

We establish convergence results for the entropy-regularized dynamics and the WFR dynamics.

\subsection{Convergence of the entropy-regularized
Wasserstein dynamics} \label{subsec:lgv_analysis}

The following theorem characterizes the stationary points of the entropy-regularized dynamics.

\begin{theorem} \label{thm:langevin_unique}
Suppose that \autoref{ass:continuous_game} holds, that $\ell \in C^2(\mathcal{X} \times \mathcal{Y})$ and that the initial measures $\mu_{x,0}, \mu_{y,0}$ have densities in $L^1(\mathcal{X}), L^1(\mathcal{Y})$. If a solution $(\mu_x(t), \mu_y(t))$ of the ERIWGF \eqref{eq:entropic_interacting_flow} converges in time, it must converge to the point $(\hmu_x,\hmu_y)$ which is the unique fixed point of the problem 
\begin{equation} 
\label{eq:exp}
        \rho_x(x) = \frac{1}{Z_x} e^{- \beta \int \ell(x,y) \ d\mu_y(y)}, \quad
        \rho_y(y) = \frac{1}{Z_y} e^{\beta \int \ell(x,y) \ d\mu_x(x)}.
\end{equation}
$(\hmu_x,\hmu_y)$ is an $\epsilon$-Nash equilibrium of the game given by $\mathcal{L}$ when
$    \beta \geq \frac{4}{\epsilon} \log \left(2 \frac{1-V_{\delta}}{V_{\delta}} (2 K_{\ell}/\epsilon -1) \right),$
where $K_{\ell} := \max_{\xx,\yy} \ell(\xx,\yy) - \min_{\xx,\yy} \ell(\xx,\yy)$ is the length of the range of $\ell$, $\delta := \epsilon/(2\text{Lip}(\ell))$ 
and $V_{\delta}$ is a lower bound on the volume of a ball of radius $\delta$ in $\mathcal{X},\mathcal{Y}$.
\end{theorem}

The proof is in \autoref{sec:pf_thm1}. \autoref{thm:langevin_unique} characterizes the stationary points of the ERIWGF but does not provide a guarantee of convergence in time. It implies that if the
dynamics \eqref{eq:entropic_interacting_flow} converges in time, the limit will
be an $\epsilon$-Nash equilibrium of $\mathcal{L}$, with $\epsilon = \tilde{O}(1/\beta)$ (disregarding log factors). 
The dynamics \eqref{eq:entropic_interacting_flow} correspond to a McKean-Vlasov 
process on the joint probability measure $\mu_x \times \mu_y$. 
While convergence to stationary solutions of such processes have been studied
in the Euclidean case \citep{eberlequantitative2019}l, their results would only 
guarantee convergence for temperatures $\beta^{-1} \gtrsim Lip(\ell)$ in our setup, 
which is not strong enough to certify convergence to arbitrary $\epsilon$-NE.

There is a
trade-off between setting a low temperature $\beta^{-1}$, which yields an
$\epsilon$-Nash equilibrium with small $\epsilon$ but possibly slow or no convergence,
and setting a high temperature, which has the opposite effect. Linear potential
Fokker-Planck equations (that we recover when both players are decoupled) indeed converge exponentially with rate
$e^{-\lambda_{\beta} t}$ for all $\beta$, with $\lambda_{\beta}$ decreasing
exponentially with $\beta$ for nonconvex potentials \citep[sec. 5]{Markowich99onthe}. Entropic regularization also biases the dynamics towards measures with full support and hence precludes convergence to sparse equilibria even if they exist. This problem does not arise in the WFR dynamics.


%% file: WFR_section.tex
\subsection{Analysis of the Wasserstein-Fisher-Rao dynamics}\label{sec:wfr_sec}



\autoref{thm:wfr_main_thm} states that, at a certain time $t_0$, the time averaged measures of the solution $(\nu_x,\nu_y)$ of \eqref{eq:WFR_eq} are an $\epsilon$-MNE, where $\epsilon$ can be made arbitrarily small by adjusting the constants $\gamma, \alpha$ of the dynamics. We define  $\bar{\nu}_{x}(t) = \frac{1}{t} \int_0^t \nu_{x}(s) \ ds$ and $\bar{\nu}_{y}(t) = \frac{1}{t} \int_0^t \nu_{y}(s) \ ds$, where $\nu_x$ and $\nu_y$ are solutions of \eqref{eq:WFR_eq}. 
\begin{theorem} \label{thm:wfr_main_thm}
Let $\epsilon > 0$ arbitrary. Suppose that $\nu_{x,0}, \nu_{y,0}$ are such that their Radon-Nikodym derivatives with respect to the Borel measures of $\mathcal{X},\mathcal{Y}$ are lower-bounded by $e^{-K'_x}, e^{-K'_y}$ respectively. For any $\delta \in (0,1/2)$, there exists a constant $C_{\delta, \mathcal{X},\mathcal{Y},K'_x,K'_y} > 0$ depending on the dimensions of $\mathcal{X}, \mathcal{Y}$, their curvatures and $K'_x, K'_y$, such that if $\gamma/\alpha < 1$,
$\frac{\gamma}{\alpha} \leq \left(\epsilon/C_{\delta, \mathcal{X},\mathcal{Y},K'_x,K'_y}\right)^{\frac{2}{1-\delta}}$
\begin{equation}
\text{NI}(\bar{\nu}_{x}(t_0),\bar{\nu}_{y}(t_0)) \leq \epsilon\quad\text{where}\quad t_0 = (\alpha \gamma)^{-1/2}.
\end{equation}

\end{theorem}

The proof (\autoref{sec:proof_thm4}) builds on the convergence properties of continuous-time mirror descent and closely follows the proof of Theorem 3.8 from \citet{chizat2019sparse}. We explicit the dependency of $C_{\delta, \mathcal{X},\mathcal{Y},K'_x,K'_y}$ on the dimensions of the manifolds and the properties of the loss $\ell$. 
Notice that \autoref{thm:wfr_main_thm} ensures convergence towards an $\epsilon$-Nash equilibrium of the non-regularized game. Following \citet{chizat2019sparse}, it is possible to replace the regularity assumption on the initial measures $\nu_{x,0}, \nu_{y,0}$ by a singular initialisation, at the expense of using $O(\exp(d))$ particles. 
This result is not a convergence result for the measures, but rather on the value of the NI error. Notice that it involves time-averaging and a finite horizon. Similar results are common for mirror descent in convex games \citep{solving}, albeit in the discrete-time setting.


\autoref{thm:wfr_main_thm} does not capture the benefits of transport, as it regards it as a perturbation of mirror descent (which corresponds to $\gamma=0$). When targetting a small error $\epsilon$, we need to set $\gamma \ll \alpha$ because of the bound on $\gamma/\alpha$. In this case, mirror descent is the main driver of the dynamics. However, it is seen empirically that taking much higher ratios $\gamma/\alpha$ (i.e. increasing the importance of the transport term) results in better performance. A satisfying explanation of this phenomenon is still sought after in the simpler optimization setting \citep{chizat2019sparse}.


%% file: propagation_chaos.tex
\subsection{Convergence to mean-field}\label{sec:mean-field}
\label{sec:algorithms}

The following theorem (proof in \autoref{sec:WFR_prop}) links the empirical measures of the systems \eqref{eq:particle_system_LGV}, \eqref{eq:particle_system_WFR} to the solutions of the mean field dynamics~\eqref{eq:entropic_interacting_flow} and~\eqref{eq:WFR_eq} respectively. It can be seen as a law of large numbers. It shows that by \autoref{thm:convergence_mean_field}, \autoref{alg:langevin_dynamics} and \autoref{alg:WFR_dynamics} approximate the mean-field dynamics studied in \autoref{subsec:lgv_analysis} and \autoref{sec:wfr_sec}.
\begin{theorem} \label{thm:convergence_mean_field}
(i) Let $\mu_x^n = \frac{1}{n} \sum_{i=1}^n \delta_{X^{(i)}} \in \mathcal{C}([0,T],\mathcal{P}(\mathcal{X})), \mu_y^n = \frac{1}{n} \sum_{i=1}^n \delta_{Y^{(i)}} \in \mathcal{C}([0,T],\mathcal{P}(\mathcal{Y}))$ be the empirical measures of a solution of \eqref{eq:particle_system_LGV} up to an arbitrary time $T$. Let $\mu_x \in \mathcal{C}([0,T],\mathcal{P}(\mathcal{X})), \mu_y \in \mathcal{C}([0,T],\mathcal{P}(\mathcal{Y}))$ be a solution of the ERIWGF \eqref{eq:entropic_interacting_flow} with mean-field initial conditions $\mu_x(0) = \mu_{x,0}, \mu_y(0) = \mu_{y,0}$. Then, 
\begin{align}
    \mathbb{E}[\mathcal{W}_{2}^2(\mu_{x,t}^n, \mu_{x,t}) + \mathcal{W}_{2}^2(\mu_{y,t}^n, \mu_{y,t})] \xrightarrow{n \rightarrow \infty} 0, \quad \mathbb{E}[|\text{NI}(\mu_{x,t}^n, \mu_{y,t}^n)-\text{NI}(\mu_{x,t}, \mu_{y,t})|] \xrightarrow{n \rightarrow \infty} 0,
\end{align}
uniformly over $t \in [0,T]$. $\text{NI}$ is the Nikaido-Isoda error defined in \eqref{eq:nikaido_def}.

(ii) Let $\nu_x^n = \sum_{i=1}^n w_{x,t}^{i} \delta_{X^{(i)}} \in \mathcal{C}([0,T],\mathcal{P}(\mathcal{X})),$ $\mu_y^n = \sum_{i=1}^n w_{y,t}^{i} \delta_{Y^{(i)}} \in \mathcal{C}([0,T],\mathcal{P}(\mathcal{Y}))$ be the (projected) empirical measures of a solution of \eqref{eq:particle_system_WFR} up to an arbitrary time $T$. Let $\nu_x \in \mathcal{C}([0,T],\mathcal{P}(\mathcal{X})), \nu_y \in \mathcal{C}([0,T],\mathcal{P}(\mathcal{Y}))$ be a solution of \eqref{eq:WFR_eq} with mean-field initial conditions $\mu_x(0) = \mu_{x,0}, \mu_y(0) = \mu_{y,0}$. Then,
\begin{align}
    \mathbb{E}[\mathcal{W}_{2}^2(\nu_{x,t}^n, \nu_{x,t}) + \mathcal{W}_{2}^2(\nu_{y,t}^n, \nu_{y,t})] \xrightarrow{n \rightarrow \infty} 0, \quad \mathbb{E}[|\text{NI}(\bar{\nu}_{x,t}^n, \bar{\nu}_{y,t}^n)-\text{NI}(\bar{\nu}_{x,t}, \bar{\nu}_{y,t})|] \xrightarrow{n \rightarrow \infty} 0,
\end{align}
uniformly over $t \in [0,T]$. $\bar{\nu}_{x,t}, \bar{\nu}_{y,t}, \bar{\nu}_{x,t}^n, \bar{\nu}_{y,t}^n$ are the time-averaged measures, as in \autoref{thm:wfr_main_thm}.
\end{theorem}

%% file: experiments.tex
\section{Numerical Experiments}\label{sec:num_exp}
 We show that WFR and Langevin dynamics outperform mirror descent in high dimension, on synthetic games. We then show the interests of using WFR-DA for training GANs. Code has been made available for reproducibility. 

\begin{figure}[t]
    \begin{minipage}[c]{0.5\textwidth}
    \includegraphics[width=\linewidth]{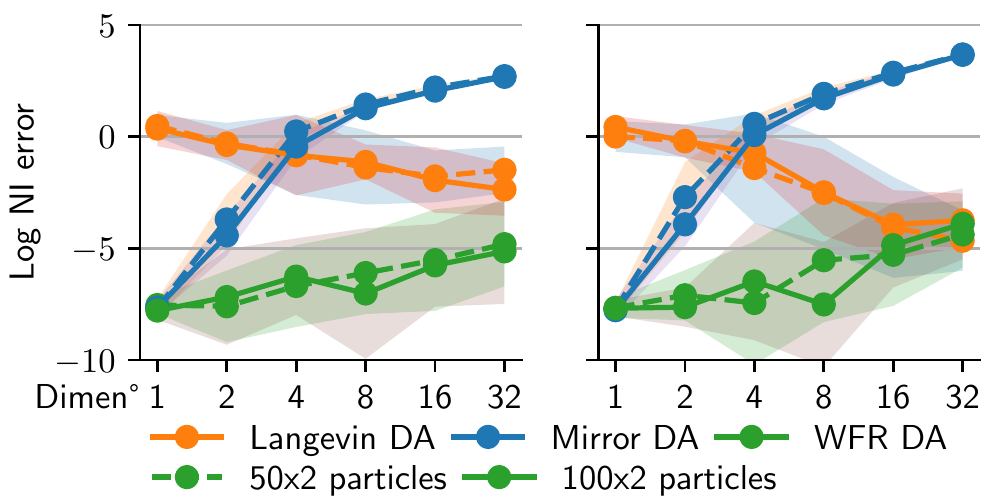}
    \end{minipage}\hfill
    \begin{minipage}[c]{0.43\textwidth}
    \caption{Nikaido-Isoida errors for L-DA, WFR-DA and mirror descent, as a function of the problem dimension, for a nonconvex loss $\ell_a$ (left) and convex loss $\ell_b$ (right). L-DA and WFR-DA outperforms mirror descent for large dimensions.
    Values averaged over 20 runs after 30000 iterations. Error bars show standard deviation across runs.} \label{fig:polygames}
    \end{minipage}
\end{figure}

\begin{figure*}[t]
    \centering
    \includegraphics[scale=.54]{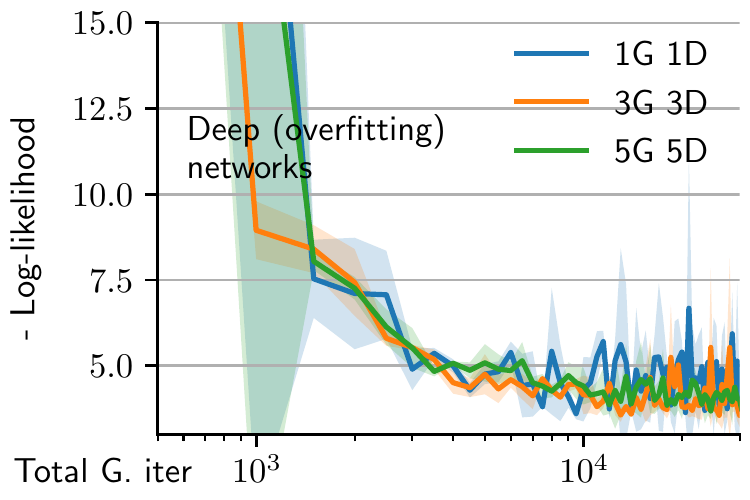}
    \includegraphics[scale=.54]{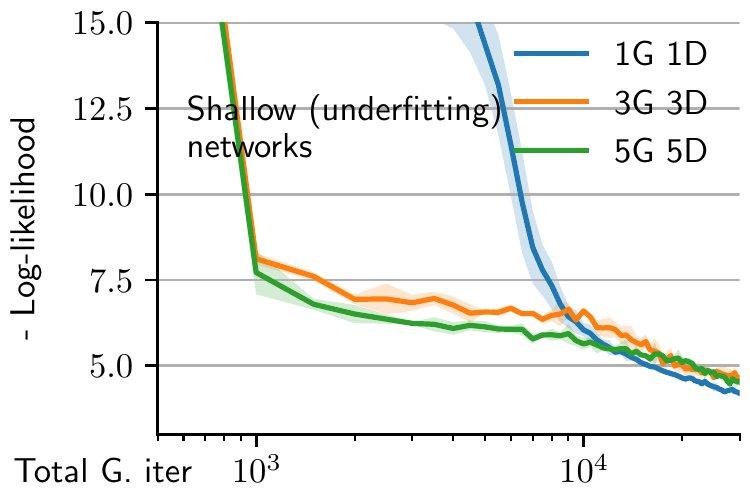}
    \includegraphics[scale=.54]{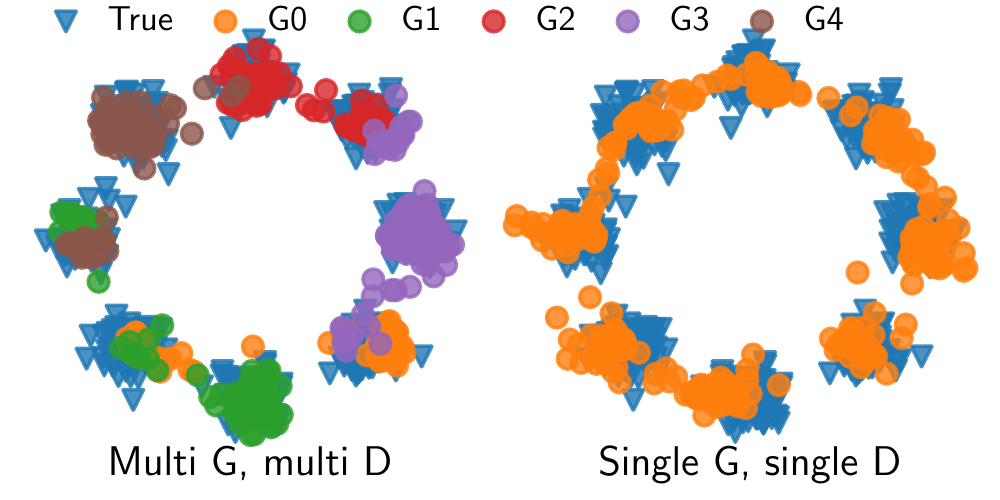}

    \caption{Training mixtures of GANs over a synthetic mixture of Gaussians in 2D. WFR-DA converges faster with models with low number of parameters, and similar performance with over-parametrized models. Mixtures naturally perform a form of clustering of the data. Errors bars show variance across 5 runs.\label{fig:synthetic_gan}}
\end{figure*}

\begin{figure*}[ht]
    \centering
    \includegraphics[scale=.54]{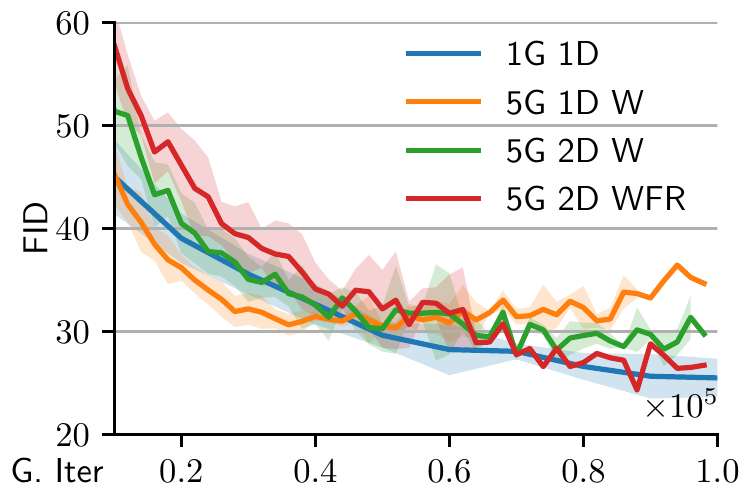}
    \includegraphics[scale=.54]{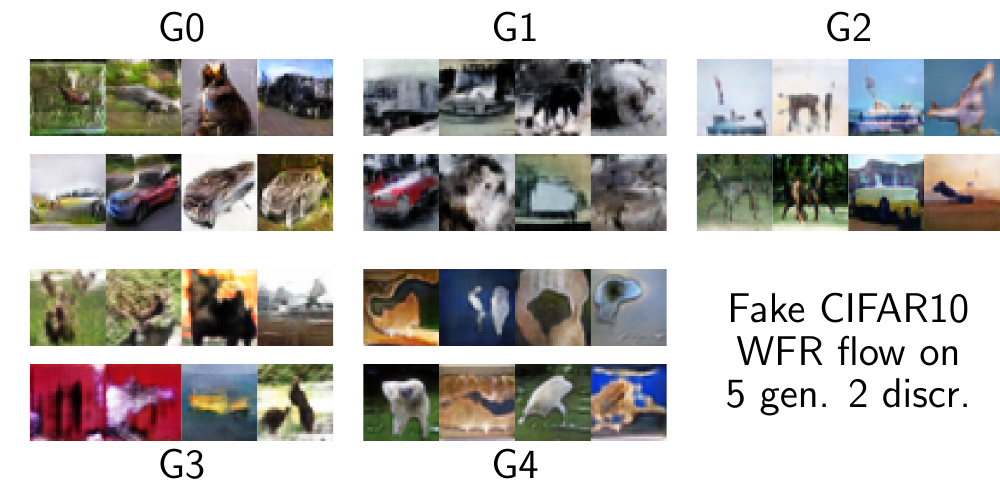}
    \includegraphics[scale=.54]{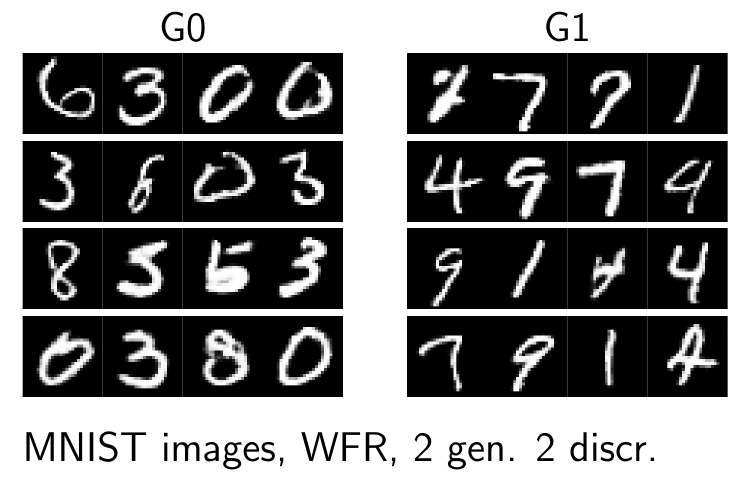}

    \caption{Training mixtures of GANs over CIFAR10. We
    compare the algorithm that updates the mixture weights and parameters
    (WFR-DA flow) with the algorithm that
    only updates parameters (W-DA flow). Using several discriminators and a WFR-DA flow brings more stable convergence. Each generator tends to specialize in a type of images. Errors bars show variance across 5 runs.\label{fig:gan}}
\end{figure*}

\subsection{Polynomial games on spheres}
We study two different games with losses $\ell_{a}, \ell_{b} : \mathcal{S}^{d-1} \times \mathcal{S}^{d-1} \rightarrow \mathbb{R}$ of the form
\begin{align}
\begin{split}
    \ell_a(x,y) &= x^\top A_0 x + x^\top A_1 y + y^\top A_2 y + y^\top A_3 (x^2) + a_0^\top x + a_1^\top y\\
    \ell_b(x,y) &= x^\top A_0^\top A_0 x + x^\top A_1 y + y^\top A_2^\top A_2 y + a_0^\top x + a_1^\top y,
\end{split}
\end{align}
where $A_0, A_1, A_2, A_3, a_0, a_1$ are matrices and vectors with components sampled from a normal distribution $\mathcal{N}(0,1)$, and $x^2$ is the vector given by component-wise multiplication of $x$. $\ell_b$ is a convex loss on the sphere, while $\ell_a$ is not.
We run Langevin Descent-Ascent (updates of positions) and WFR Descent-Ascent
(updates of weights and positions), and compare it with mirror descent (updates of weights).We note that the computation of the NI error \eqref{eq:nikaido_def} entails solving two optimization problems on measures, or equivalently in parameter space. We solve each of them by performing 2000 gradient acsent runs with random uniform initialization and selecting the highed minimum final value. This gives a lower bound on the NI error which is precise enough for our purposes. We perform
time averaging on the weights of mirror descent and WFR-DA, but not on the
positions of WFR-DA because that would incur an $O(t)$ overhead on memory.

\paragraph{Results.} Wirror descent
performs like WFR-DA in low dimensions, but suffers strongly from the curse of
dimensionality (\autoref{fig:polygames}). On the other hand, algorithms that
incorporate a transport term keep performing well in high dimensions. In
particular, WFR-DA is consistently the algorithm with lowest NI error.  Notice
that the errors in the $n=50$ and $n=100$ plots do not differ much, confirming
that we reach a mean-field regime.

\subsection{Training GAN mixtures}


 We now use WFR-DA to train mixtures of generator networks. We consider the Wasserstein-GAN \citep{arjovsky2017wasserstein} setting. We seek to approximate a distribution $\Pp_{\text{data}}$ with a distribution $\Gg_x$, defined as the push-forward of a noise distribution $\Nn(0,I)$ by a neural-network $g_x$. The discrepancy between $\Pp_{\text{data}}$ and $\Gg_x$ is estimated by a neural-network discriminator $f_y$, leading to the problem
\begin{equation}
    \min_{x} \max_{y} \ell(x, y) \triangleq \EE_{a \sim p_{\text{data}}}
    [f_y(a)] - \EE_{\varepsilon \sim \Nn(0, I)}
    [f_y(g_x(\varepsilon))].
\end{equation}
We lift this problem in the space of distributions over the parameters $x$ and $y$ (see \autoref{subsec:minimax}),
that we represent through weighted discrete distributions of
 $\sum_{i=1}^p w^{(i)}_x \delta_{x^{(i)}}$ and
 $\sum_{j=1}^q w^{(j)}_y \delta_{y^{(j)}}$. We solve
\begin{equation}\label{eq:gan}
    \min_{x^{(i)}, w_x \in \triangle_p} \max_{y^{(j)}, w_y \in \triangle^q} \sum_{i=1}^p \sum_{j=1}^q w^{(i)}_x w^{(j)}_y \ell(x^{(i)}, y^{(j)})~,
\end{equation}
using \autoref{alg:WFR_dynamics}, where $\triangle^q$ is the $q$-dimensional simplex.
 The optimal generation strategy corresponding
 to an equilibrium point $(x^{(i)})_i, w_x, (y^{(j)})_j, w_y$ is
 then to randomly select a generator $g_{x_I}$ with $I$ sampled
 among $[n]$ with probability $w^{(i)}_x$, and use it to generate $g_{x_I}(\varepsilon)$, with $\varepsilon \sim \Nn(0, I)$.
 Training mixtures of generators has been proposed by
 \citet{ghosh_multi-agent_2018}, with a tweaked discriminator loss. Our formulation only involves a lifting in the space of measures, and uses a new training algorithm.

\paragraph{Results on 2D GMMs.} We first set $\Pp_{\text{data}}$ to be an
$8$-mode mixture of Gaussians in two dimensions. We use the original W-GAN loss, with weight cropping for the
discriminators $(f_{y^{(j)}})_j$. We measure the interest of using mixtures when a single generator $g_{x^{(i)}}$ cannot fit $\Pp_{\text{data}}$ (single-layer MLP), and when it can (4-layer MLP).
We report results in \autoref{fig:synthetic_gan}, measuring the log likelihood of $\Gg_x$ for the GMM during training.  The WFR dynamic is stable even with few particles. When training under-parametrized generators, using mixtures permits faster convergence (in terms of generator updates). In the over-parametrized
setting, training a single generator or a mixture of generators perform similarly. WFR-DA is thus useful to train mixtures of simple generators. In this setting, each simple generator identifies modes in the
training data, doing data clustering at no cost (\autoref{fig:synthetic_gan} right).

\paragraph{Results on real data.} We train a mixture of ResNet generators on CIFAR10 and MNIST. We replace the position updates in \autoref{alg:WFR_dynamics} by extrapolated Adam steps \citep{gidel2018variational} to achieve faster convergence, and perform grid search over generator and discriminators learning rates. Convergence curves for the best learning rates are displayed in
\autoref{fig:gan} right, measuring test FID \citep{heusel2017gans}. With a sufficient number of generators and
discriminators $(G > 5, D > 2)$, the model trains as fast as a normal
GAN. WFR-DA is thus stable and efficient even with a reasonable number of particles.
Using the discretized WFR versus the Wasserstein flow provides a slight
improvement over updating parameters only. As with GMMs, each generator trained with WFR-DA becomes specialised in generating a fraction of the target data, thereby identifying clusters.
Those could be used for unsupervised
conditional generation of images.


%% file: future_work.tex
\section{Conclusions and future work}
We have explored non-convex-non-concave, high-dimensional 
games from the perspective of optimal transport. 
As with non-convex optimization, framing the problem in terms of measures provides geometric benefits, at the expense of moving into non-Euclidean metric spaces over measures. 
Our theoretical results establish approximate mean-field 
convergence for two setups: Langevin Descent-Ascent 
and WFR D-A, and directly applies to GANs, for mixtures of generators and discriminators. 

Despite the positive convergence guarantees our results 
are qualitative in nature, i.e. without rates.
In the entropic case, the unfavorable tradeoff between temperature and convergence of the associated McKean-Vlasov scheme deserves further study, 
maybe through log-Sobolev-type inequalities \citep{Markowich99onthe}.  
In the WFR case, we lack a local convergence analysis explaining the benefits of transport observed empirically, perhaps leveraging sharpness Polyak-Łojasiewicz results such as those in \citep{chizat2019sparse} or \citep{sanjabi2018solving}.
 Finally, in our GAN formulation, each generator is associated to a single particle in a high-dimensional product space of all network parameters, which is not scalable to large population sizes that would approximate their mean-field limit. 
 A natural question is to understand to what extent our framework could be combined with specific choices of architecture, as recently studied in \citep{lei2019sgd}.
 \vfill

%% file: broader_impact.tex
\section*{Broader impact}
We study algorithms designed to find equilibria in games, provide theoretical guarantees of convergence and test their performance empirically. Among other applications, our results give insight into training algorithms for generative adversarial networks (GANs), which are useful for many relevant tasks such as image generation, image-to-image or text-to-image translation and video prediction. 
As always, we note that machine learning improvements like ours come in the form of ``building machines to do X better''.
For a sufficiently malicious or ill-informed choice of X, such as surveillance or recidivism prediction, almost any progress in machine learning might indirectly lead to a negative outcome, and our work is not excluded from that.

\section*{Funding disclosure}
C. Domingo-Enrich thanks J. De Dios Pont for conversations on the subject.
This work is partially supported by the Alfred P. Sloan Foundation, NSF RI-1816753, NSF CAREER CIF 1845360, NSF CHS-1901091, Samsung Electronics, and the Institute for Advanced Study. The work of C. Domingo-Enrich is partially supported by the La Caixa Fellowship.
The work of A. Mensch is supported by the European Research Council (ERC project NORIA). The work of G. Rotskoff is supported by the James S. McDonnell Foundation. 

%% file: app_lifted.tex
\printcontents[sections]{l}{1}{\setcounter{tocdepth}{2}}

\section{Lifted dynamics for the Interacting Wasserstein-Fisher-Rao Gradient Flow} \label{subsec:lifted_WFR}
Recall the IWFRGF in \eqref{eq:WFR_eq}, which we reproduce here for convenience.
\begin{align}
    \begin{split} 
        \begin{cases}
        \partial_t \mu_x &= \gamma \nabla_x \cdot (\mu_x \nabla_x V_x(\mu_y,x)) - \alpha \mu_x(V_x(\mu_y,x) - \mathcal{L}(\mu_x,\mu_y)), \quad \mu_x(0) = \mu_{x,0}\\
        \partial_t \mu_y &= -\gamma \nabla_y \cdot (\mu_y \nabla_y V_y(\mu_x,y)) + \alpha \mu_y(V_y(\mu_x,y) -  \mathcal{L}(\mu_x,\mu_y)), \quad \mu_y(0) = \mu_{y,0}
        \end{cases}
    \end{split}
\end{align}

Given $\nu_x \in \mathcal{P}(\mathcal{X} \times \mathbb{R}^+)$ define $\mu_x = \int_{\mathcal{X}} w_x \ d\nu_x(\cdot, w_x) \in \mathcal{P}(\mathcal{X})$, that is
\begin{align} \label{eq:h_def}
    \int_\mathcal{X} \phi(\xx) \ d\mu_x(\xx) = \int_\mathcal{X \times \mathbb{R}^{+}} w_x \phi(\xx) \ d\nu_x(\xx,w_x),
\end{align}
for all $\phi \in C(\mathcal{X})$. Given $\nu_y \in \mathcal{P}(\mathcal{Y} \times \mathbb{R}^+)$, define $\mu_y = \int_{\mathcal{X}} w_y \ d\nu_y(\cdot, w_y) \in \mathcal{P}(\mathcal{Y})$ analogously. We say that $\nu_x, \nu_y$ are ``lifted'' measures of $\mu_x, \mu_y$, and reciprocally $\mu_x, \mu_y$ are ``projected'' measures of $\nu_x, \nu_y$. 

By \autoref{lem:lifted_dynamics} below, we can view a solution of \eqref{eq:WFR_eq} as the projection of a solution of the following dynamics on the lifted domains $\mathcal{X} \times \mathbb{R}^{+}$ and $\mathcal{Y} \times \mathbb{R}^{+}$: 
\begin{equation}
     \begin{split} \label{eq:interacting_flow_wfr2}
        \begin{cases} 
        \partial_t \nu_x = \nabla_{w_x,x} \cdot (\nu_x {g}_{\mu_y}(x,w_x)), \quad \nu_{x}(0) = \mu_{x,0}\times {\delta}_{w_x = 1}\\
        \partial_t \nu_y = -\nabla_{w_y,y} \cdot (\nu_y g_{\mu_x}(y,w_y)), \quad \nu_{y}(0) = \mu_{y,0}\times {\delta}_{w_y = 1}
        \end{cases}
    \end{split}
\end{equation}
where
\begin{align}
\begin{split}
g_{\mu_y}(x,w_x) &= (\alpha w_x (V_x(\mu_y, x)-\mathcal{L}(\mu_x,\mu_y)), \gamma \nabla_x V_x(\mu_y, x))),\\
g_{\mu_x}(y,w_y) &= (\alpha w_y (V_y(\mu_x, x)-\mathcal{L}(\mu_x,\mu_y)), \gamma \nabla_y V_y(\mu_x, y))).
\end{split}
\end{align}

\begin{lemma} \label{lem:lifted_dynamics}
For a solution $\nu_x : [0,T] \rightarrow \mathcal{P}(\mathcal{X}\times \mathbb{R}^{+}), \nu_y : [0,T] \rightarrow \mathcal{P}(\mathcal{Y} \times \mathbb{R}^{+})$ of \eqref{eq:interacting_flow_wfr2}, the projections $\mu_x, \mu_y$ are solutions of \eqref{eq:WFR_eq}.
\end{lemma}
That is, given any $\phi_x \in \mathcal{C}^1(\mathcal{X}), \phi_y \in \mathcal{C}^1(\mathcal{Y})$, we have
\begin{align}
    \begin{split} \label{eq:weak_nonlifted}
        \frac{d}{dt} \int_{\mathcal{X}} \phi_x(x) \ d\mu_x &= - \gamma \int_{\mathcal{X}} \nabla_x \phi_x(x) \cdot \nabla_x V_x(\mu_y,x) \ d\mu_x - \alpha \int_{\mathcal{X}}  \phi_x(x) (V_x(\mu_y,x) - \mathcal{L}(\mu_x,\mu_y)) \ d\mu_x,\\
        \frac{d}{dt} \int_{\mathcal{Y}} \phi_y(y) \ d\mu_y &= \gamma \int_{\mathcal{Y}} \nabla_y \phi_y(y) \cdot \nabla_y V_y(\mu_x,y)) \ d\mu_y + \alpha \int_{\mathcal{Y}} \phi_y(y) (V_y(\mu_x,y) -  \mathcal{L}(\mu_x,\mu_y)) \ d\mu_y,\\
        \mu_x(0) &= \mu_{x,0}, \quad \mu_y(0) = \mu_{y,0}
    \end{split}
\end{align}
From \eqref{eq:interacting_flow_wfr2} in the weak form, we obtain that given any $\psi_x \in \mathcal{C}^1(\mathcal{X} \times \mathbb{R}^{+}), \psi_y \in \mathcal{C}^1(\mathcal{Y} \times \mathbb{R}^{+})$,
\begin{align}
    \begin{split} \label{eq:weak_IWFR}
        \frac{d}{dt} \int_{\mathcal{X} \times \mathbb{R}^{+}} \psi_x(x,w_x) \ d\nu_x(x,w_x) &= \int_{\mathcal{X} \times \mathbb{R}^{+}} - \gamma \nabla_x \psi_x(x,w_x) \cdot \nabla_x V_x(\mu_y,x) \\ &- \alpha w_x \frac{d\psi_x}{dw_x}(x,w_x) (V_x(\mu_y,x) - \mathcal{L}(\mu_x,\mu_y)) \ d\mu_x,\\
        \frac{d}{dt} \int_{\mathcal{Y} \times \mathbb{R}^{+}} \psi_y(y,w_y) \ d\nu_y(y,w_y) &= \int_{\mathcal{Y} \times \mathbb{R}^{+}} \gamma \nabla_y \psi_y(y,w_y) \cdot \nabla_y V_y(\mu_x,y) \\ &+ \alpha w_y \frac{d\psi_y}{dw_y}(y,w_y) (V_y(\mu_x,y) -  \mathcal{L}(\mu_x,\mu_y)) \ d\mu_y,\\
        \nu_{x}(0) = \mu_{x,0}\times {\delta}_{w_x = 1}, \quad \nu_{y}(0) &= \mu_{y,0}\times {\delta}_{w_y = 1}. 
    \end{split}
\end{align}
Taking $\psi_x(x,w_x) = w_x \phi_x(x), \psi_y(y,w_y) = w_y \phi_y(y)$ yields
\begin{align}
    \begin{split} \label{eq:lifted_eq_lem}
        \frac{d}{dt} \int_{\mathcal{X} \times \mathbb{R}^{+}} w_x \phi_x(x) \ d\nu_x(x,w_x) &= \int_{\mathcal{X} \times \mathbb{R}^{+}} - \gamma w_x \nabla_x \phi_x(x) \cdot \nabla_x V_x(\mu_y,x) \\ &- \alpha w_x \phi_x(x) (V_x(\mu_y,x) - \mathcal{L}(\mu_x,\mu_y)) \ d\mu_x,\\
        \frac{d}{dt} \int_{\mathcal{Y} \times \mathbb{R}^{+}} w_y \psi_y(y,w_y) \ d\nu_y(y,w_y) &= \int_{\mathcal{Y} \times \mathbb{R}^{+}} \gamma w_y \nabla_y \phi_y(y) \cdot \nabla_y V_y(\mu_x,y) \\ &+ \alpha w_y \phi_y(y) (V_y(\mu_x,y) -  \mathcal{L}(\mu_x,\mu_y)) \ d\mu_y.
    \end{split}
\end{align}

Notice that \eqref{eq:lifted_eq_lem} is indeed \eqref{eq:weak_nonlifted}.

%% file: app_nikaido_continuity.tex
\section{Continuity and convergence properties of the Nikaido-Isoda error} \label{subsec:conv_ni}
\begin{lemma} \label{lem:NI_continuity}
The Nikaido-Isoda error $\text{NI} : \mathcal{P}(\mathcal{X}) \times \mathcal{P}(\mathcal{Y}) \rightarrow \mathbb{R}$ defined in \eqref{eq:nikaido_def} is continuous when we endow $\mathcal{P}(\mathcal{X}), \mathcal{P}(\mathcal{Y})$ with the topology of weak convergence. 
Specifically, it is $\text{Lip}(\ell)$-Lipschitz when we use the distance $\mathcal{W}_1(\mu_x,\mu'_x) + \mathcal{W}_1(\mu_y,\mu'_y)$ between $(\mu_x,\mu_y)$ and $(\mu'_x,\mu'_y)$ in $\mathcal{P}(\mathcal{X}) \times \mathcal{P}(\mathcal{Y})$.
\end{lemma}
\begin{proof}
For any $\mu_y$, the function $V_x(\mu_y, \cdot) : \mathcal{X} \rightarrow \mathbb{R}$ defined as $x \mapsto \int \ell(x, y) \ d\mu_y$ is continuous and it has the same Lipschitz constant $\text{Lip}(\ell)$ as $\ell$. Hence, for any $\mu_x,\mu'_x \in \mathcal{P}(\mathcal{X})$, 
\begin{align}
\begin{split}
    &\sup_{\mu_y \in \mathcal{P}(\mathcal{Y})} \mathcal{L}(\mu_x,\mu_y) - \sup_{\mu_y \in \mathcal{P}(\mathcal{Y})} \mathcal{L}(\mu'_x,\mu_y) = \sup_{\mu_y \in \mathcal{P}(\mathcal{Y})} \int V_x(\mu_y,x) d\mu_x - \sup_{\mu_y \in \mathcal{P}(\mathcal{Y})} \int V_x(\mu_y,x) d\mu'_x \\&\leq \sup_{\mu_y \in \mathcal{P}(\mathcal{Y})} \int V_x(\mu_y,x) d\mu'_x + \sup_{\mu_y \in \mathcal{P}(\mathcal{Y})} \int V_x(\mu_y,x) d(\mu_x - \mu'_x) - \sup_{\mu_y \in \mathcal{P}(\mathcal{Y})} \int V_x(\mu_y,x) d\mu'_x \\&= \sup_{\mu_y \in \mathcal{P}(\mathcal{Y})} \int V_x(\mu_y,x) d(\mu_x - \mu'_x) \leq \text{Lip}(\ell) \mathcal{W}_1(\mu_x,\mu'_x)
\end{split}
\end{align}
The same inequality interchanging the roles of $\mu_x, \mu'_x$ shows that $|\sup_{\mu_y \in \mathcal{P}(\mathcal{Y})} \mathcal{L}(\mu_x,\mu_y) - \sup_{\mu_y \in \mathcal{P}(\mathcal{Y})} \mathcal{L}(\mu'_x,\mu_y)| \leq \text{Lip}(\ell) \mathcal{W}_1(\mu_x,\mu'_x)$ holds.
An analogous reasoning for $\ell(\mu_x, \cdot) : \mathcal{Y} \rightarrow \mathbb{R}$ and the triangle inequality complete the proof.
\end{proof}
\begin{lemma} \label{lem:conv_ni}
    Suppose that $(\mu_x^n)_{n \in \mathbb{N}}$ is a sequence of random elements valued in $\mathcal{P}(\mathcal{X})$ such that 
    \begin{equation} 
    \mathbb{E}[\mathcal{W}_{2}^2(\mu_x^n, \mu_x)] \xrightarrow{n \rightarrow \infty} 0,
    \end{equation}
    where $\mu_x \in \mathcal{P}(X)$. Analogously, suppose that $(\mu_y^n)_{n \in \mathbb{N}}$ is a sequence of random elements valued in $\mathcal{P}(\mathcal{Y})$ such that 
    \begin{equation}
        \mathbb{E}[\mathcal{W}_{2}^2(\mu_y^n, \mu_y)] \xrightarrow{n \rightarrow \infty} 0,
    \end{equation}
    where $\mu_y \in \mathcal{P}(Y)$. 
    
    Then,
    \begin{equation}
        \mathbb{E}[|\text{NI}(\mu_x^n,\mu_y^n)-\text{NI}(\mu_x,\mu_y)|] \xrightarrow{n \rightarrow \infty} 0
    \end{equation}
\end{lemma}
\begin{proof}
First,
\begin{equation} \label{eq:2_CS}
    \mathbb{E}[\mathcal{W}_{1}(\mu_x^n, \mu_x)] \leq \mathbb{E}[\mathcal{W}_{2}(\mu_x^n, \mu_x)] \leq \left(\mathbb{E}[\mathcal{W}_{2}^2(\mu_x^n, \mu_x)]\right)^{1/2},
\end{equation}
which results from two applications of the Cauchy-Schwarz inequality on the appropriate scalar products. An analogous inequality holds for $\mathbb{E}[\mathcal{W}_{1}(\mu_y^n, \mu_y)]$. Hence, by \autoref{lem:NI_continuity},
\begin{align}
    \begin{split}
        \mathbb{E}[|\text{NI}(\mu_x^n,\mu_y^n)-\text{NI}(\mu_x,\mu_y)|] &\leq \text{Lip}(\ell) \mathbb{E}[\mathcal{W}_1(\mu_x^n,\mu_x) + \mathcal{W}_1(\mu_y^n,\mu_y)] \\ &\leq \text{Lip}(\ell) \left( \left(\mathbb{E}[\mathcal{W}_{2}^2(\mu_x^n, \mu_x)]\right)^{1/2} + \left(\mathbb{E}[\mathcal{W}_{2}^2(\mu_x^n, \mu_x)]\right)^{1/2}\right) \\ &\leq \text{Lip}(\ell) \sqrt{2} \left( \mathbb{E}[\mathcal{W}_{2}^2(\mu_x^n, \mu_x)] + \mathbb{E}[\mathcal{W}_{2}^2(\mu_x^n, \mu_x)]\right)^{1/2},
    \end{split}
\end{align}
where the second inequality uses \eqref{eq:2_CS} and the third inequality is another application of the Cauchy-Schwarz inequality. Since the right hand side converges to 0 by assumption, this concludes the proof.
\end{proof}

%% file: app_thm1.tex
\section{Proof of \autoref{thm:langevin_unique}} \label{sec:pf_thm1}
We restate 
\autoref{thm:langevin_unique}
for convenience.

\begingroup
\def\thetheorem{\ref{thm:langevin_unique}}
\begin{theorem} 
Suppose that \autoref{ass:continuous_game} holds, that $\ell \in C^{2,\alpha}(\mathcal{X} \times \mathcal{Y})$ for some $\alpha \in (0,1)$ and that the initial measures $\mu_{x,0}, \mu_{y,0}$ have densities in $L^1(\mathcal{X}), L^1(\mathcal{Y})$. If a solution $(\mu_x(t), \mu_y(t))$ of the ERIWGF \eqref{eq:entropic_interacting_flow} converges in time, it must converge to the point $(\hmu_x,\hmu_y)$ which is the unique fixed point of the problem 
\begin{equation} 
        \rho_x(x) = \frac{1}{Z_x} e^{- \beta \int \ell(x,y) \ d\mu_y(y)}, \quad
        \rho_y(y) = \frac{1}{Z_y} e^{\beta \int \ell(x,y) \ d\mu_x(x)}.
\end{equation}
$(\hmu_x,\hmu_y)$ is an $\epsilon$-Nash equilibrium of the game given by $\mathcal{L}$ when
$    \beta \geq \frac{4}{\epsilon} \log \left(2 \frac{1-V_{\delta}}{V_{\delta}} (2 K_{\ell}/\epsilon -1) \right),$
where $K_{\ell} := \max_{\xx,\yy} \ell(\xx,\yy) - \min_{\xx,\yy} \ell(\xx,\yy)$ is the length of the range of $\ell$, $\delta := \epsilon/(2\text{Lip}(\ell))$ 
and $V_{\delta}$ is a lower bound on the volume of a ball of radius $\delta$ in $\mathcal{X},\mathcal{Y}$.
\end{theorem}
\addtocounter{theorem}{-1}
\endgroup

\autoref{thm:langevin_unique} is a consequence of the following three results, which we prove separately.

\begin{theorem} \label{thm:fixed_point_problem}
Assume $\mathcal{X}, \mathcal{Y}$ are compact Polish metric spaces equipped with canonical Borel measures, and that $\ell$ is a continuous function on $\mathcal{X} \times \mathcal{Y}$. Let us consider the fixed point problem 
\begin{align} 
    \begin{split} 
    \begin{cases} 
        \rho_x(x) &= \frac{1}{Z_x} e^{- \beta \int \ell(x,y) \ d\mu_y(y)}, \\
        \rho_y(y) &= \frac{1}{Z_y} e^{\beta \int \ell(x,y) \ d\mu_x(x)},
    \end{cases}
    \end{split} 
\end{align}
where $Z_x$ and $Z_y$ are normalization constants and $\rho_x, \rho_y$ are the densities of $\mu_x, \mu_y$. This fixed point problem has a unique solution $(\hmu_x,\hmu_y)$ that is also the unique Nash equilibrium of the game given by $\mathcal{L}_{\beta}(\mu_x,\mu_y) \triangleq \mathcal{L}(\mu_x,\mu_y) +\beta^{-1}(H(\mu_y) -H(\mu_x))$.
\end{theorem}

\begin{theorem} \label{thm:epsilon_thm}
Let $K_{\ell} := \max_{\xx,\yy} \ell(\xx,\yy) - \min_{\xx,\yy} \ell(\xx,\yy)$ be the length of the range of $\ell$. Let $\epsilon > 0$, 
$\delta := \epsilon/(2\text{Lip}(\ell))$ 
and $V_{\delta}$ be a lower bound on the volume of a ball of radius $\delta$ in $\mathcal{X},\mathcal{Y}$. Then the solution $(\hmu_x,\hmu_y)$ of \eqref{eq:exp} is an $\epsilon$-Nash equilibrium of the game given by $\mathcal{L}$ when
\begin{equation}
    \beta \geq \frac{4}{\epsilon} \log \left(2 \frac{1-V_{\delta}}{V_{\delta}} (2 K_{\ell}/\epsilon -1) \right).
\end{equation}
\end{theorem}

\begin{theorem} \label{thm:stationary_point}
Suppose that \autoref{ass:continuous_game} holds and $\ell \in C^{2,\alpha}(\mathcal{X} \times \mathcal{Y})$ for some $\alpha \in (0,1)$, i.e. the second derivatives of $\ell$ are $\alpha$-Hölder.
Then, there exists only one stationary solution of the ERIWGF \eqref{eq:entropic_interacting_flow} and it is the solution of the fixed point problem \eqref{eq:exp}.
\end{theorem}

\subsection{Proof of \autoref{thm:fixed_point_problem}: Preliminaries}
\begin{definition}[Upper hemicontinuity]
A set-valued function $\phi: X \rightarrow 2^{Y}$ is upper hemicontinuous if for every open set $W \subset Y$, the set $\{x|\phi(x) \subset W \}$ is open.
\end{definition}
Alternatively, set-valued functions can be seen as correspondences $\Gamma: X \rightarrow Y$. The graph of $\Gamma$ is $\text{Gr}(\Gamma) = \{(a,b) \in X \times Y | b \in \Gamma(a)\}$. If $\Gamma$ is upper hemicontinuous, then $\text{Gr}(\Gamma)$ is closed. If $Y$ is compact, the converse is also true. 

\begin{definition}[Kakutani map]
Let $X$ and $Y$ be topological vector spaces and $\phi: X \rightarrow 2^Y$ be a set-valued function. If $Y$ is convex, then $\phi$ is termed a Kakutani map if it is upper hemicontinuous and $\phi(x)$ is non-empty, compact and convex for all $x \in X$.
\end{definition}

\begin{theorem}[Kakutani-Glicksberg-Fan] \label{thm:kakutani}
Let $S$ be a non-empty, compact and convex subset of a Hausdorff locally convex topological vector space. Let $\phi: S \rightarrow 2^S$ be a Kakutani map. Then $\phi$ has a fixed point.
\end{theorem}

\begin{definition}[Lower semi-continuity]
Suppose $X$ is a topological space, $x_{0}$ is a point in $X$ and 
$f:X \rightarrow \mathbb{R} \cup \{-\infty, \infty\}$ is an extended real-valued function. We say that $f$ is lower semi-continuous (l.s.c.) at 
$x_{0}$ if for every 
$\epsilon >0$ there exists a neighborhood $U$ of $x_{0}$ such that $f(x)\geq f(x_{0})-\epsilon$ for all $x$ in $U$ when $f(x_{0})<+\infty$, and $f(x)$ tends to $+\infty$  as $x$ tends towards $x_{0}$ when $f(x_{0})=+\infty$. 
\end{definition}

We can also characterize lower-semicontinuity in terms of level sets. A function is lower semi-continuous if and only if all of its lower level sets $\{x\in X:~f(x)\leq \alpha \}$ are closed. This property will be useful.

\begin{theorem}[Weierstrass theorem for l.s.c. functions] \label{thm:weierstrass}
Let $f : T \rightarrow (-\infty, +\infty]$ be a l.s.c. function on a compact Hausdorff topological space $T$. Then $f$ attains its infimum over $T$, i.e. there exists a minimum of $f$ in $T$.
\end{theorem}
\begin{proof}
Proof. Let $\alpha_0 = \inf f(T)$. If $\alpha_0 = +\infty$, then $f$ is infinite and the assertion trivially holds. Let $\alpha_0 < +\infty$. Then, for each real $\alpha > \alpha_0$, the set $\{f \leq \alpha\}$ is closed and nonempty. Any finite collection of such sets has a nonempty intersection. By compactness, also the set $\bigcap_{\alpha > \alpha_0} \{ f \leq \alpha \} = \{f \leq \alpha_0 \} = f^{-1}(\alpha_0)$ is nonempty. (In particular, this implies that $\alpha_0$ is
finite.)
\end{proof}

\begin{remark}
By Prokhorov's theorem, since $\mathcal{X}$ and $\mathcal{Y}$ are compact separable metric spaces, $\mathcal{P}(\mathcal{X})$ and $\mathcal{P}(\mathcal{Y})$ are compact in the topology of weak convergence. 
\end{remark}

\subsection{Proof of \autoref{thm:fixed_point_problem}: Existence}
\autoref{prop:p1} and \ref{prop:continuity_m1} are intermediate results, and \autoref{prop:existence} shows existence of the solution.
\begin{lemma} \label{prop:p1}
For any $\mu_y \in \mathcal{P}(\mathcal{Y})$, $\mathcal{L}_{\beta}(\cdot, \mu_y) : \mathcal{P}(\mathcal{X}) \rightarrow \mathbb{R}$ is lower semicontinuous, and it achieves a unique minimum in $\mathcal{P}(\mathcal{X})$. Moreover, the minimum $m_x(\mu_y)$ is absolutely continuous with respect to the Borel measure, it has full support and its density takes the form
\begin{equation} \label{eq:def_m_x}
     \frac{dm_x(\mu_y)}{d\xx}(\xx) = \frac{1}{Z_{\mu_y}} e^{-\beta \int L(\xx,\yy) d\mu_y},
\end{equation}
where $Z_{\mu_y}$ is a normalization constant.

Analogously, for any $\mu_x \in \mathcal{P}(\mathcal{X})$, $- \mathcal{L}_{\beta}(\mu_x, \cdot) : \mathcal{P}(\mathcal{Y}) \rightarrow \mathbb{R}$ is lower semicontinuous, and it achieves a unique minimum in $\mathcal{P}(\mathcal{Y})$. The minimum $m_y(\mu_x)$ is absolutely continuous with respect to the Borel measure, it has full support and its density takes the form
\begin{equation}
     \frac{dm_y(\mu_x)}{d\yy}(\yy) = \frac{1}{Z_{\mu_x}} e^{\beta \int L(\xx,\yy) d\mu_x},
\end{equation}
where $Z_{\mu_x}$ is a normalization constant.
\end{lemma}
\begin{proof}
We will prove the result for $\mathcal{L}_{\beta}(\cdot, \mu_y)$, as the other one is analogous. Let $dx$ denote the canonical Borel measure on $\mathcal{X}$, and let $\tilde{p}$ be the probability measure proportional to the canonical Borel measure, i.e. $\frac{d\tilde p}{dx} = \frac{1}{\text{vol}(\mathcal{X})}$. Notice that $\text{vol}(\mathcal{X})$ is by definition the value of the canonical Borel measure on the whole $\mathcal{X}$. We rewrite 
\begin{align}
\begin{split}
    \mathcal{L}_{\beta}(\mu_x,\mu_y) &= \iint \ell(\xx,\yy) d\mu_y d\mu_x + \beta^{-1} \int \log \left(\frac{d\mu_x}{dx}\right) d\mu_x + \beta^{-1} H(\mu_y)\\ &= \iint \ell(\xx,\yy) d\mu_y d\mu_x + \beta^{-1} \int \log \left(\frac{d\mu_x}{d\tilde{p}} \frac{d\tilde{p}}{dx}\right) d\mu_x + \beta^{-1} H(\mu_y)\\
    &= \iint \left(\ell(\xx,\yy) - \beta^{-1} \log \left(\text{vol}(\mathcal{X})\right)\right) d\mu_y d\mu_x + \beta^{-1} \int \log \left(\frac{d\mu_x}{d\tilde{p}}\right) d\mu_x + \beta^{-1} H(\mu_y)
\end{split}
\end{align}
Notice that the first term in the right hand side is a lower semi-continuous (in weak convergence topology) functional in $\mu_x$ when $\mu_y$ is fixed. That is because it is a linear functional in $\mu_x$ with a continuous integrand, which implies that it is continuous in the weak convergence topology. The second to last term can be seen as the relative entropy (or Kullback-Leibler divergence) between $\mu_x$ and $\tilde{p}$:
\begin{align}
    H_{\tilde{p}}(\mu_x) := \int \log \left(\frac{d\mu_x}{d\tilde{p}}\right) d\mu_x
\end{align}
The relative entropy $H_{\tilde{p}}(\mu_x)$ is a lower semi-continuous functional with respect to $\mu_x$ (see Theorem 1 of \citet{posner_random}, which proves a stronger statement: joint semi-continuity with respect to both measures). 

Therefore, we conclude that $\mathcal{L}_{\beta}(\cdot, \mu_y)$ (with $\mu_y \in \mathcal{P}(\mathcal{Y})$ fixed) is a l.s.c. functional on $\mathcal{P}(\mathcal{X})$. By Theorem \ref{thm:weierstrass} and using the compactness of $\mathcal{P}(\mathcal{X})$, there exists a minimum of $\mathcal{L}_\beta(\cdot, \mu_y)$ in $\mathcal{P}(\mathcal{X})$. 

Denote a minimum of $\mathcal{L}_\beta(\cdot, \mu_y)$ by $\hmu_x$. $\hmu_x$ must be absolutely continuous, because otherwise $-\beta^{-1}H(\hmu_x)$ would take an infinite value. By the Euler-Lagrange equations for functionals on probability measures, a necessary condition for $\hmu_x$ to be a minimum of $\mathcal{L}_\beta(\cdot, \mu_y)$ is that the first variation $\frac{\delta \mathcal{L}_\beta(\cdot, \mu_y)}{\delta \mu_x}(\hmu_x)(\xx)$ must take a constant value for all $\xx \in \text{supp}(\hmu_x)$ and values larger or equal outside of $\text{supp}(\hmu_x)$. The intuition behind this is that otherwise a zero-mean signed measure with positive mass on the minimizers of $\frac{\delta \mathcal{L}_\beta(\cdot, \mu_y)}{\delta \mu_x}(\hmu_x)$ and negative mass on the maximizers would provide a direction of decrease of the functional. We compute the first variation at $\hmu_x$:
\begin{align}
\begin{split}
    \frac{\delta \mathcal{L}_\beta(\cdot, \mu_y)}{\delta \mu_x}(\hat{\mu}_x) (\xx) &= \frac{\delta}{\delta \mu_x} \left(\int L(\xx,\yy) d\mu_y d\mu_x - \beta^{-1} H(\hat{\mu}_x) + \beta^{-1} H(\mu_y) \right) \\ &= \int L(\xx,\yy) d\mu_y + \beta^{-1} \log \left(\frac{d\hat{\mu}_x}{d\xx}(\xx) \right),
\end{split}
\end{align}
We equate $\int \ell(\xx,\yy) d\mu_y + \beta^{-1} \log(\frac{d\hat{\mu}_x}{dx}(\xx)) = K, \ \forall \xx \in \supp(\hmu_x)$, where $K$ is a constant. The first variation must take values larger or equal than $K$ outside of $\supp(\hmu_x)$, but since $\log(\frac{d\hat{\mu}_x}{dx}(\xx)) = -\infty$ outside of $\supp(\hmu_x)$, we obtain that $\supp(\hmu_x) = \mathcal{X}$. Then, for all $\xx \in \mathcal{X}$,
\begin{align} \label{eq:}
    \frac{d\hat{\mu}_x}{d\xx}(\xx) = e^{- \beta \int L(\xx,\yy) d\mu_y  + \beta K} = \frac{1}{Z_{\mu_y}}e^{- \beta \int L(\xx,\yy) d\mu_y}
\end{align}
where $Z_{\mu_y}$ is a normalization constant obtained from imposing $\int \frac{d\hat{\mu}_x}{d\xx}(\xx) \ d\xx = \int 1 \ d\hat{\mu}_x = 1$.
Since the necessary condition for optimality specifies a unique measure and the minimum exists, we obtain that $m_x(\mu_y) = \hmu_x$ is the unique minimum. An analogous argument holds for $m_y(\hmu_x)$ 
\end{proof}

\begin{lemma} \label{prop:continuity_m1}
Suppose that the measures ${(\mu_{y,n})}_{n \in \mathbb{N}}$ and $\mu_y$ are in $\mathcal{P}(\mathcal{Y})$. Recall the definition of $m_x : \mathcal{P}(\mathcal{Y}) \rightarrow \mathcal{P}(\mathcal{X})$ in equation \eqref{eq:def_m_x}. If ${(\mu_{y,n})}_{n \in \mathbb{N}}$ converges weakly to $\mu_y$, then ${(m_x(\mu_{y,n}))}_{n \in \mathbb{N}}$ converges weakly to $m_x(\mu_{y})$, i.e. $m_x$ is a continuous mapping when we endow $\mathcal{P}(\mathcal{Y})$ and $\mathcal{P}(\mathcal{X})$ with their weak convergence topologies.

The same thing holds for $m_y$ and measures ${(\mu_{x,n})}_{n \in \mathbb{N}}$ and $\mu_x$ on $\mathcal{X}$.
\end{lemma}

\begin{proof}
Given $x \in \mathcal{X}$, we have $\int \ell(\xx,\yy) d\mu_{y,n} \rightarrow \int \ell(\xx,\yy) d\mu_{y}$, because $\ell(\xx,\cdot)$ is a continuous bounded function on $\mathcal{Y}$. By continuity of the exponential function, we have that for all $x \in \mathcal{X}$, $e^{-\beta \int \ell(\xx,\yy) d\mu_{y,n}} \rightarrow e^{-\beta \int \ell(\xx,\yy) d\mu_{y}}$.
Using the dominated convergence theorem,
\begin{align}
    \int_{\mathcal{X}} e^{-\beta \int \ell(\xx,\yy) d\mu_{y,n}} d\xx \rightarrow \int_{\mathcal{X}} e^{-\beta \int \ell(\xx,\yy)  d\mu_{y}} d\xx
\end{align}
We need to find a dominating function. It is easy, because 
$\forall n \in \mathbb{N}, \ \forall \xx \in \mathcal{X}$,  $e^{-\beta \int \ell(\xx,\yy) d\mu_{y,n}} \leq e^{-\beta \min_{(\xx,\yy) \in \mathcal{X} \times \mathcal{Y}} \ell(\xx,\yy)}$. And $\int_{\mathcal{X}} e^{-\beta \min_{(\xx,\yy) \in \mathcal{X} \times \mathcal{Y}} \ell(\xx,\yy)} d\xx = e^{-\beta \min_{(\xx,\yy) \in \mathcal{X} \times \mathcal{Y}} \ell(\xx,\yy)} \text{vol}(\mathcal{X}) < \infty$.
By the Portmanteau theorem, we just need to prove that for all continuity sets $B$ of $m_x(\mu_y)$, we have $m_x(\mu_{y,n})(B) \rightarrow m_x(\mu_{y})(B)$. This translates to
\begin{align}
    \frac{\int_{B} e^{-\beta \int \ell(\xx,\yy) d\mu_{y,n}} d\xx}{\int_{\mathcal{X}} e^{-\beta \int \ell(\xx,\yy) d\mu_{y,n}} d\xx} \rightarrow \frac{\int_{B} e^{-\beta \int \ell(\xx,\yy)  d\mu_{y}} d\xx}{\int_{\mathcal{X}} e^{-\beta \int \ell(\xx,\yy)  d\mu_{y}} d\xx}
\end{align}
We have proved that the denominators converge appropriately, and the numerator converges as well using the same reasoning with dominated convergence. And both the numerators and the denominators are positive and the numerator is always smaller denominator, the quotient must converge. 
\end{proof}

\begin{lemma} \label{prop:existence}
There exists a solution of \eqref{eq:exp}, which is the Nash equilibrium of the game given by $\mathcal{L}_{\beta}$. 
\end{lemma}

\begin{proof}
We use \autoref{thm:kakutani} on the set $\mathcal{P}(\mathcal{X}) \times \mathcal{P}(\mathcal{Y})$, with the map $m : \mathcal{P}(\mathcal{X}) \times \mathcal{P}(\mathcal{Y}) \rightarrow \mathcal{P}(\mathcal{X}) \times \mathcal{P}(\mathcal{Y})$ given by $m(\mu_x,\mu_y) = (m_x(\mu_y), m_y(\mu_x))$. The only condition to check is upper hemicontinuity of $m$. By \autoref{prop:continuity_m1} we know that $m_x, m_y$ are continuous, and since continuous functions are upper hemicontinuous as set valued functions, this concludes the argument. Indeed, we could have used Tychonoff's theorem, which is similar to \autoref{thm:kakutani} but for single-valued functions. 
\end{proof}

\subsection{Proof of \autoref{thm:fixed_point_problem}: Uniqueness}
\begin{lemma}
The solution of \eqref{eq:exp} is unique. 
\end{lemma}
\begin{proof}
The argument is analogous to the proof of Theorem 2 of \citet{rosen_existence_1965}. Suppose $(\mu_{x,1},\mu_{y,1})$ and $(\mu_{x,2},\mu_{y,2})$ are two different solutions of \eqref{eq:exp}. We use the notation $F_1(\mu_x,\mu_y) = \mathcal{L}_\beta(\mu_x,\mu_y), F_2(\mu_x,\mu_y) = -\mathcal{L}_\beta(\mu_x,\mu_y)$. Hence, there exist constants $K_{x,1}, K_{y,1}, K_{x,2}, K_{y,2}$ such that
\begin{align}
    \begin{split}
        \frac{\delta F_1}{\delta \mu_x}(\mu_{x,1},\mu_{y,1})(\xx) + K_{x,1} &= 0,
        \frac{\delta F_2}{\delta \mu_y}(\mu_{x,1},\mu_{y,1})(\yy) + K_{y,1} = 0, \\
        \frac{\delta F_1}{\delta \mu_x}(\mu_{x,2},\mu_{y,2})(\xx) + K_{x,2} &= 0,
        \frac{\delta F_2}{\delta \mu_y}(\mu_{x,2},\mu_{y,2})(\yy) + K_{y,2} = 0
    \end{split}
\end{align}
On the one hand, we know that
\begin{align}
    \begin{split} \label{eq:is_0}
    &\int \frac{\delta F_1}{\delta \mu_x}(\mu_{x,1},\mu_{y,1})(\xx) \ d(\mu_{x,2}-\mu_{x,1}) + \int \frac{\delta F_2}{\delta \mu_y}(\mu_{x,1},\mu_{y,1})(\yy) \ d(\mu_{y,2}-\mu_{y,1}) \\
    &+\int \frac{\delta F_1}{\delta \mu_x}(\mu_{x,2},\mu_{y,2})(\xx) \ d(\mu_{x,1}-\mu_{x,2}) + \int \frac{\delta F_2}{\delta \mu_y}(\mu_{x,2},\mu_{y,2})(\yy) \ d(\mu_{y,1}-\mu_{y,2}) \\ = &- \int K_{x,1} \ d(\mu_{x,2}-\mu_{x,1}) - \int K_{y,1} \ d(\mu_{y,2}-\mu_{y,1}) \\ &- \int K_{x,2} \ d(\mu_{x,1}-\mu_{x,2}) - \int K_{y,2} \ d(\mu_{y,1}-\mu_{y,2}) = 0
    \end{split}
\end{align}
We will now prove that the left hand side of \eqref{eq:is_0} must be strictly larger than 0, reaching a contradiction.
We can write
\begin{align}
\begin{split}
    \frac{\delta F_1}{\delta \mu_x}(\mu_{x,2},\mu_{y,2})(\xx) - \frac{\delta F_1}{\delta \mu_x}(\mu_{x,1},\mu_{y,1})(\xx) &= \int L(\xx,\yy) \ d(\mu_{y,2}-\mu_{y,1}) \\ &+ \beta^{-1} (\log(\mu_{x,2}(\xx)) - \log(\mu_{x,1}(\xx))),\\
    \frac{\delta F_2}{\delta \mu_y}(\mu_{x,2},\mu_{y,2})(\xx) - \frac{\delta F_2}{\delta \mu_y}(\mu_{x,1},\mu_{y,1})(\xx) &= -\int L(\xx,\yy) \ d(\mu_{x,2}-\mu_{x,1}) \\ &+ \beta^{-1} (\log(\mu_{y,2}(\xx)) - \log(\mu_{y,1}(\xx)))
\end{split}
\end{align}
Hence, we rewrite the left hand side of \eqref{eq:is_0} as
\begin{align}
    \begin{split} \label{eq:rel_entropy}
        &\iint L(\xx,\yy) \ d(\mu_{y,2}-\mu_{y,1}) d(\mu_{x,2} - \mu_{x,1}) + \beta^{-1} \int (\log(\mu_{x,2}(\xx)) - \log(\mu_{x,1}(\xx))) \ d(\mu_{x,2} - \mu_{x,1}) \\
        -&\iint L(\xx,\yy) \ d(\mu_{x,2}-\mu_{x,1}) d(\mu_{y,2}-\mu_{y,1}) + \beta^{-1} \int (\log(\mu_{y,2}(\xx)) - \log(\mu_{y,1}(\xx))) \ d(\mu_{y,2}-\mu_{y,1}) \\ &= \beta^{-1} (H_{\mu_{x,1}}(\mu_{x,2}) + H_{\mu_{x,2}}(\mu_{x,1}) + H_{\mu_{y,1}}(\mu_{y,2}) + H_{\mu_{y,1}}(\mu_{y,2})).
    \end{split}
\end{align}
Since the relative entropy is always non-negative and zero only if the two measures are equal, we have reached the desired contradiction. 
\end{proof}

%% file: app_thm2.tex
\subsection{Proof of \autoref{thm:epsilon_thm}} \label{sec:pf_thm2}

We will use the shorthand $V_x(\xx) = V_x(\hmu_y)(\xx) = \int \mathcal{L}(\xx,\yy) d\hmu_y$, $V_y(\yy) = V_y(\hmu_x)(\yy) = \int \mathcal{L}(\xx,\yy) d\hmu_x$. Since $\ell: \mathcal{X} \times \mathcal{Y} \rightarrow \mathbb{R}$ is a continuous function on a compact metric space, it is uniformly continuous. Hence,
\begin{align}
    \forall \epsilon > 0, \exists \delta > 0 \text{ st. } \sqrt{d(\xx,\xx')^2+d(\yy,\yy')^2} < \delta \implies |\ell(\xx,\yy)-\ell(\xx',\yy')| < \epsilon 
\end{align}
Which means that 
\begin{align}
    d(\xx,\xx') < \delta \implies |V_x(\xx)-V_x(\xx')| = \bigg|\int (\ell(\xx,\yy)-\ell(\xx',\yy)) d\yy \bigg| < \epsilon 
\end{align}
This proves that $V_x$ is uniformly continuous on $\mathcal{X}$ (and $V_y$ is uniformly continuous on $\mathcal{Y}$ using the same argument). 

We can write the Nikaido-Isoda function of the game with loss $\mathcal{L}$ (equation \eqref{eq:nikaido_def}) evaluated at $(\hmu_x,\hmu_y)$ as
\begin{align}
\begin{split} \label{eq:NI_development}
    \text{NI}(\hmu_x,\hmu_y) &:= \mathcal{L}(\hmu_x,\hmu_y) - \min_{\mu_x'} \{ \mathcal{L}(\mu_x',\hmu_y)\} + (-\mathcal{L}(\hmu_x,\hmu_y) + \max_{\mu_y'} \{ \mathcal{L}(\hmu_x,\mu_y')\}) \\ &= \frac{\int V_x(\xx) e^{-\beta V_x(\xx)} d\xx}{\int e^{-\beta V_x(\xx)} d\xx} - \min_{\xx \in \mathcal{C}_1} V_x(\xx) + \frac{-\int V_y(\yy) e^{\beta V_y(\yy)} d\yy}{\int e^{\beta V_y(\yy)} d\yy} + \max_{\yy \in \mathcal{C}_2} V_y(\yy)
\end{split}
\end{align}
The second equality follows from the definitions of $\mathcal{L}, V_x, V_y$. We observe that in the right-most expression the first two terms and the last two terms are analogous. Let us show the first two terms can be made smaller than an arbitrary $\epsilon > 0$ by taking $\beta$ large enough; the last two will be dealt with in an analogous manner. Let us define $\tilde{V}_x(\xx) = V_x(\xx)-\min_{\xx' \in \mathcal{C}_1} V_x(\xx')$. 
\begin{align}
\begin{split} \label{eq:epsilon_beta0}
    & \frac{\int V_x(\xx) e^{-\beta V_x(\xx)} d\xx}{\int e^{-\beta V_x(\xx)} d\xx} - \min_{\xx \in \mathcal{C}_1} V_x(\xx) = \frac{\int (V_x(\xx)-\min_{\xx' \in \mathcal{C}_1} V_x(\xx')) e^{-\beta V_x(\xx)} d\xx}{\int e^{-\beta V_x(\xx)} d\xx} \\ &= \frac{\int \tilde{V}_x(\xx) e^{-\beta V_x(\xx)} \left(\mathds{1}_{\{\tilde{V}_x(\xx) \leq \epsilon/2 \}}  + \mathds{1}_{\{\epsilon/2 < \tilde{V}_x(\xx) \leq \epsilon \}} + \mathds{1}_{\{\epsilon < \tilde{V}_x(\xx)\}} \right) d\xx}{\int e^{-\beta V_x(\xx)} \mathds{1}_{\{\tilde{V}_x(\xx) \leq \epsilon/2 \}} d\xx + \int e^{-\beta V_x(\xx)} \mathds{1}_{\{\epsilon/2 < \tilde{V}_x(\xx) \leq \epsilon \}} d\xx + \int e^{-\beta V_x(\xx)} \mathds{1}_{\{\epsilon < \tilde{V}_x(\xx)\}} d\xx}
\end{split}
\end{align}
Let us define 
\begin{align}
    q_{\{\tilde{V}_x(\xx) \leq \epsilon/2 \}} = \int e^{-\beta V_x(\xx)} \mathds{1}_{\{\tilde{V}_x(\xx) \leq \epsilon/2 \}} d\xx,
\end{align}
and $q_{\{\epsilon/2 < \tilde{V}_x(\xx) \leq \epsilon \}}$ and $q_{\{\epsilon < \tilde{V}_x(\xx)\}}$ analogously.

Similarly, let 
\begin{align}
        r_{\{\tilde{V}_x(\xx) \leq \epsilon/2 \}} &= \int \tilde{V}_x(\xx) e^{-\beta V_x(\xx)} \mathds{1}_{\{\tilde{V}_x(\xx) \leq \epsilon/2 \}} d\xx, 
\end{align}
and $r_{\{\epsilon/2 < \tilde{V}_x(\xx) \leq \epsilon \}}$ and $r_{\{\epsilon < \tilde{V}_x(\xx)\}}$ analogously.

Let 
\begin{align}
\tilde{p} = \frac{q_{\{\epsilon/2 < \tilde{V}_x(\xx) \leq \epsilon \}}}{q_{\{\tilde{V}_x(\xx) \leq \epsilon/2 \}} + q_{\{\epsilon/2 < \tilde{V}_x(\xx) \leq \epsilon \}} + q_{\{\epsilon < \tilde{V}_x(\xx)\}}}
\end{align}
Then, we can rewrite the right-most expression of \eqref{eq:epsilon_beta0} as
\begin{align} 
\begin{split} \label{eq:epsilon_beta1}
    &\frac{r_{\{\tilde{V}_x(\xx) \leq \epsilon/2 \}} + r_{\{\epsilon/2 < \tilde{V}_x(\xx) \leq \epsilon \}} + r_{\{\epsilon < \tilde{V}_x(\xx)\}}}{q_{\{\tilde{V}_x(\xx) \leq \epsilon/2 \}} + q_{\{\epsilon/2 < \tilde{V}_x(\xx) \leq \epsilon \}} + q_{\{\epsilon < \tilde{V}_x(\xx)\}}} \\ &= \tilde{p} \frac{r_{\{\epsilon/2 < \tilde{V}_x(\xx) \leq \epsilon \}}}{q_{\{\epsilon/2 < \tilde{V}_x(\xx) \leq \epsilon \}}} + (1-\tilde{p}) \frac{r_{\{\tilde{V}_x(\xx) \leq \epsilon/2 \}} + r_{\{\epsilon < \tilde{V}_x(\xx)\}}}{q_{\{\tilde{V}_x(\xx) \leq \epsilon/2 \}} + q_{\{\epsilon < \tilde{V}_x(\xx)\}}}
\end{split}
\end{align}
Since $\tilde{V}(\xx) \leq \epsilon$ in the set $\{\xx|\epsilon/2 < \tilde{V}_x(\xx) \leq \epsilon \}$, $r_{\{\epsilon/2 < \tilde{V}_x(\xx) \leq \epsilon \}}/q_{\{\epsilon/2 < \tilde{V}_x(\xx) \leq \epsilon \}} \leq \epsilon$. 

Let $\xx_{\text{min}}$ be such that $V(\xx_{\text{min}}) = \min_{\xx \in C_1} V(\xx)$ (possibly not unique). By uniform continuity of $V_x$, we know there exists $\delta > 0$ (dependent only on $\epsilon$) such that $B(\xx_{\text{min}}, \delta) \subseteq \{\xx | \tilde{V}_x(\xx) \leq \epsilon/2 \}$. The following inequalities hold:
\begin{align} \label{eq:epsilon_beta2}
    \begin{split}
        r_{\{\tilde{V}_x(\xx) \leq \epsilon/2\}} & \leq \frac{\epsilon}{2} q_{\{\tilde{V}_x(\xx) \leq \epsilon/2\}}, \\
        r_{\{\epsilon < \tilde{V}_x(\xx)\}} & \leq (\max_{\xx \in \mathcal{C}_1} V_x(\xx) - \min_{\xx \in \mathcal{C}_1} V_x(\xx)) q_{\{\epsilon < \tilde{V}_x(\xx)\}} \leq (\max_{\xx,\yy} L(\xx,\yy) - \min_{\xx,\yy} L(\xx,\yy)) q_{\{\epsilon < \tilde{V}_x(\xx)\}} \\&= K_L q_{\{\epsilon < \tilde{V}_x(\xx)\}}.
    \end{split}
\end{align}
where we define $K_\ell = \max_{\xx,\yy} \ell(\xx,\yy) - \min_{\xx,\yy} \ell(\xx,\yy)$. Using \eqref{eq:epsilon_beta2}, we obtain
\begin{align}
    \frac{r_{\{\tilde{V}_x(\xx) \leq \epsilon/2 \}} + r_{\{\epsilon < \tilde{V}_x(\xx)\}}}{q_{\{\tilde{V}_x(\xx) \leq \epsilon/2 \}} + q_{\{\epsilon < \tilde{V}_x(\xx)\}}} \leq \frac{\frac{\epsilon}{2} q_{\{\tilde{V}_x(\xx) \leq \epsilon/2 \}} + K_L q_{\{\epsilon < \tilde{V}_x(\xx)\}}}{q_{\{\tilde{V}_x(\xx) \leq \epsilon/2 \}} + q_{\{\epsilon < \tilde{V}_x(\xx)\}}}. 
\end{align}
If the right-hand side is smaller or equal than $\epsilon$, then equation \eqref{eq:epsilon_beta1} would be smaller than $\epsilon$ and the proof would be concluded. For that to happen, we need $(K_\ell - \epsilon) q_{\{\epsilon < \tilde{V}_x(\xx)\}} \leq \frac{\epsilon}{2} q_{\{\tilde{V}_x(\xx) \leq \epsilon/2 \}} \iff q_{\{\tilde{V}_x(\xx) \leq \epsilon/2 \}}/q_{\{\epsilon < \tilde{V}_x(\xx)\}} \geq 2(K_\ell/\epsilon -1)$. The following bounds hold:
\begin{align}
    \begin{split}
        q_{\{\tilde{V}_x(\xx) \leq \epsilon/2 \}} &\geq \text{Vol}(B(\xx_{\text{min}}, \delta)) e^{-\beta (\min_{\xx \in \mathcal{C}_1} V_x(\xx) + \epsilon/2)},\\
        q_{\{\epsilon < \tilde{V}_x(\xx)\}} &\leq (1-\text{Vol}(B(\xx_{\text{min}}, \delta))) e^{-\beta (\min_{\xx \in \mathcal{C}_1} V_x(\xx) + \epsilon)}. 
    \end{split}
\end{align}
Thus, the following condition is sufficient:
\begin{align}
    \frac{\text{Vol}(B(\xx_{\text{min}}, \delta))}{1-\text{Vol}(B(\xx_{\text{min}}, \delta))} e^{\beta \epsilon/2} \geq 2(K_L/\epsilon -1). 
\end{align}
Hence, if we take 
\begin{align} \label{eq:beta_x}
    \beta \geq \frac{2}{\epsilon} \log \left(2 \frac{1-\text{Vol}(B(\xx_{\text{min}}, \delta))}{\text{Vol}(B(\xx_{\text{min}}, \delta))} (K_L/\epsilon -1) \right)
\end{align}
then $(\hmu_x,\hmu_y)$ is an $\epsilon$-Nash equilibrium. Since we have only bound the first two terms in the right hand side of \eqref{eq:NI_development} and the other two are bounded in the same manner, the statement of the theorem results from setting $\epsilon = \epsilon/2$ in \eqref{eq:beta_x}.

%% file: app_thm3.tex
\subsection{Proof of \autoref{thm:stationary_point}} \label{sec:pf_thm3}

First, we show that any pair $\hmu_x, \hmu_y$ such that 
\begin{equation} 
        \frac{d\hmu_x}{d\xx}(\xx) = \frac{1}{Z_x} e^{- \beta \int \ell(\xx,\yy) \ d\hmu_y(\yy)}, \quad \frac{d\hmu_y}{d\yy}(\yy) = \frac{1}{Z_y} e^{\beta \int \ell(\xx,\yy) \ d\hmu_x(\xx)}
\end{equation}
is a stationary solution of \eqref{eq:entropic_interacting_flow}. Denoting the Radon-Nikodym derivatives $\frac{d\hmu_x}{d\xx}, \frac{d\hmu_y}{d\yy}$ by $\hat{\rho}_x, \hat{\rho}_y$, it is sufficient to see that
\begin{align}
    \begin{split} \label{eq:fixed_is_stationary}
        \begin{cases}
        0 = \nabla_x \cdot (\hat{\rho}_x \nabla_x V_x(\mu_y,x)) + \beta^{-1} \Delta_x \hat{\rho}_x\\
        0 = -\nabla_y \cdot (\hat{\rho}_y \nabla_y V_y(\mu_x,y)) + \beta^{-1} \Delta_y \hat{\rho}_y
        \end{cases}
    \end{split}
\end{align}
holds weakly. And
\begin{align}
\begin{split}
    \nabla_x \hat{\rho}_x = \frac{1}{Z_x} e^{- \beta \int \ell(\xx,\yy) \ d\hmu_y(\yy)} \left(-\beta \nabla_x \int \ell(\xx,\yy) \ d\hmu_y(\yy) \right) = -\hat{\rho}_x \nabla_x V_x(\hmu_y,x), \\
    \nabla_y \hat{\rho}_y = \frac{1}{Z_y} e^{\beta \int \ell(\xx,\yy) \ d\hmu_x(\xx)} \left(\beta \nabla_y \int \ell(\xx,\yy) \ d\hmu_x(\xx) \right) = \hat{\rho}_y \nabla_y V_y(\hmu_x,y),
\end{split}
\end{align}
implies that \eqref{eq:fixed_is_stationary} holds.

Now we will prove the converse. Suppose that $\hmu_x, \hmu_y$ are (weak) stationary solutions of \eqref{eq:entropic_interacting_flow}. That is, if $\phi_x \in C^2(\mathcal{X}), \phi_y \in C^2(\mathcal{Y})$ are arbitrary twice continuously differentiable functions, the following holds
\begin{align} 
\begin{split} \label{eq:weak_stationary_measures}
    0 & = \int_{\mathcal{X}} \left( -\int_{\mathcal{Y}} \nabla_x \phi_x(\xx) \cdot \nabla_x \ell(\xx,\yy) \ d\hmu_y  + \beta^{-1} \Delta_{x} \phi_x(\xx) \right) \ d\hmu_x \\
    0 & = \int_{\mathcal{Y}} \left( \int_{\mathcal{X}} - \nabla_y \phi_y(\yy) \cdot \nabla_y \ell(\xx,\yy) \ d\hmu_x  - \beta^{-1} \Delta_{y} \phi_y(\xx,\yy) \right) \ d\hmu_y 
\end{split}
\end{align}
\eqref{eq:weak_stationary_measures} can be seen as two measure-valued stationary Fokker-Planck equations. We want to see that they have densities and that the densities satisfy the corresponding classical stationary Fokker-Planck equations \eqref{eq:fixed_is_stationary}. Works in the theory of PDEs have studied sufficient conditions for measure-valued stationary Fokker-Planck equations to correspond to weak stationary Fokker-Planck equations, and further to classical stationary Fokker-Planck equations. See page 3 of \citet{huang2015steady} for a more detailed explanation on the two steps. That measure-valued stationary correspond to weak stationary solutions is shown in Theorem 2.2 of \cite{bogachev2001onregularity}. That weak stationary solutions are classical stationary solutions requires that the drift term is in $C^{1,\alpha}_{\text{loc}}$ (locally $\alpha$-Hölder continuous with exponent 1), meaning that it is in $C^1$ and that its derivatives are $\alpha$-Hölder in compact sets. The result follows from the theory of Schauder estimates. Differentiating under the integral sign, the drift terms $-\int_{\mathcal{Y}} \nabla_x \ell(\xx,\yy) \ d\hmu_y, \int_{\mathcal{X}} \nabla_y \ell(\xx,\yy) \ d\hmu_x$ fulfill the condition if $\ell \in C^{2,\alpha}$.

%% file: app_thm4.tex
\section{Proof of \autoref{thm:wfr_main_thm}} \label{sec:proof_thm4}
Recall the expression of an \textit{Interacting Wasserstein-Fisher-Rao Gradient Flow (IWFRGF)} in \eqref{eq:WFR_eq}:
\begin{align}
    \begin{split} 
        \begin{cases}
        \partial_t \mu_x &= \gamma \nabla \cdot (\mu_x \nabla_x V_x(\mu_y,x)) \\ &- \alpha \mu_x(V_x(\mu_y,x) - \mathcal{L}(\mu_x,\mu_y)), \quad \mu_x(0) = \mu_{x,0}\\
        \partial_t \mu_y &= -\gamma \nabla \cdot (\mu_y \nabla_y V_y(\mu_x,y))\\ &+ \alpha \mu_y(V_y(\mu_x,y) -  \mathcal{L}(\mu_x,\mu_y)), \quad \mu_y(0) = \mu_{y,0}
        \end{cases}
    \end{split}
\end{align}

The aim is to obtain a global convergence result like the one in Theorem 3.8 of \citet{chizat2019sparse}. First, we will rewrite Lemma 3.10 of \citet{chizat2019sparse} in our case.

\begin{lemma} \label{thm:wfr_ni_error}
Let $\mu_{x}, \mu_{y}$ be the solution of the IWFRGF in \eqref{eq:WFR_eq}. Let $\mu^\star_x, \mu^\star_y$ be arbitrary measures on $\mathcal{X}, \mathcal{Y}$. Let $\bar{\mu}_{x}(t) = \frac{1}{t} \int_{0}^t \mu_{x}(s) \ ds$ and $\bar{\mu}_{y}(t) = \frac{1}{t} \int_{0}^t \mu_{y}(s) \ ds$.
Let $\|\cdot\|_{\text{BL}}$ be the bounded Lipschitz norm, i.e. $\|f\|_\text{BL} = \|f\|_{\infty} + \text{Lip}(f)$. Let
\begin{align} \label{eq:calQ_def}
    \mathcal{Q}_{\mu^\star,\mu_0}(\tau) = \inf_{\mu \in \mathcal{P}(\Theta)} \|\mu^\star - \mu\|_{\text{BL}}^* + \frac{1}{\tau} \mathcal{H}(\mu,\mu_0)
\end{align}
with $\Theta = \mathcal{X}$ or $\mathcal{Y}$. Let
\begin{align} \label{eq:B_def}
    B = \frac{1}{2}\left(\max_{\xx \in \mathcal{X},\yy \in \mathcal{Y}} \ell(\xx,\yy) - \min_{\xx \in \mathcal{X},\yy \in \mathcal{Y}} \ell(\xx,\yy) \right) + \text{Lip}(\ell)
\end{align}
Then,
\begin{align} \label{eq:thm3.10}
    \mathcal{L}(\bar{\mu}_{x}(t), \mu^\star_y) - \mathcal{L}(\mu^\star_x, \bar{\mu}_{y}(t)) \leq B \mathcal{Q}_{\mu^\star_x,\mu_{x,0}}(\alpha B t) + B \mathcal{Q}_{\mu^\star_y,\mu_{y,0}}(\alpha B t) + \gamma B^2t
\end{align}
\end{lemma}

\begin{proof}
The proof is as in Lemma 3.10 of \citet{chizat2019sparse}, but in this case we have to do everything twice. Namely, we define the dynamics
\begin{align}
    \begin{split}
       \frac{d\mu_x^{\epsilon}}{dt} & = \gamma \nabla \cdot(\mu_x^{\epsilon} \nabla V_x(\mu_y,\xx)) \\
       \frac{d\mu_y^{\epsilon}}{dt} & = -\gamma \nabla \cdot(\mu_y^{\epsilon} \nabla V_y(\mu_x,\yy))
    \end{split}
\end{align}
initialized at $\mu_x^{\epsilon}(0) = \mu_{x,0}^{\epsilon}, \mu_y^{\epsilon}(0) = \mu_{y,0}^{\epsilon}$ arbitrary such that $\mu_{x,0}^{\epsilon}$ and $\mu_{y,0}^{\epsilon}$ are absolutely continuous with respect to $\mu_{x,0}$ and $\mu_{y,0}$ respectively.

Let us show that
\begin{align} \label{eq:derivative_entropy0}
        \frac{1}{\alpha}\frac{d}{dt} \mathcal{H}(\mu_x^{\epsilon}, \mu_x) = \int \frac{\delta \mathcal{L}}{\delta \mu_x}(\mu_x,\mu_y)(\xx) \ d(\mu_x^{\epsilon} - \mu_x) 
\end{align}
where $\mathcal{H}(\mu_x^{\epsilon}, \mu_x)$ is the relative entropy, i.e.
\begin{align}
    \begin{split} \label{eq:derivative_entropy}
        \frac{d}{dt} \mathcal{H}(\mu_x^{\epsilon}, \mu_x) = \frac{d}{dt} \int \log \left(\rho_x^{\epsilon}\right) \ d\mu_x^{\epsilon},
    \end{split}
\end{align}
$\rho_x^{\epsilon}$ being the Radon-Nikodym derivative $d\mu_x^{\epsilon}/d\mu_x$. 

Assume to begin with that $\mu_x^{\epsilon}$ remains absolutely continuous with respect to $\mu_x$ through time. We can write
\begin{align}
    \frac{d}{dt} \int \phi_x(\xx) \rho_x^{\epsilon}(\xx) d\mu_x(\xx) = \frac{d}{dt} \int \phi(\xx) d\mu^{\epsilon}_x(\xx)  
\end{align}
We can develop the left hand side into
\begin{align}
\begin{split}
    \frac{d}{dt} \int \phi_x(\xx) \rho_x^{\epsilon}(\xx) d\mu_x(\xx) &= \int - \gamma \nabla (\phi_x(\xx) \rho_x^{\epsilon}(\xx)) \cdot \nabla V_x(\mu_y,\xx) d\mu_x(\xx) \\ &+ \int - \alpha \phi_x(\xx) \rho_x^{\epsilon}(\xx) (V_x(\mu_y,x) - \mathcal{L}(\mu_x,\mu_y)) d\mu_x(\xx) \\ &+ \int \phi_x(\xx) \frac{\partial \rho_x^{\epsilon}}{\partial t}(\xx) d\mu_x(\xx) \\&= \int - \gamma (\nabla \phi_x(\xx) \rho_x^{\epsilon}(\xx) + \phi_x(\xx) \nabla \rho_x^{\epsilon}(\xx)) \cdot \nabla V_x(\mu_y,\xx) \ d\mu_x(\xx) \\&+ \int - \alpha \phi_x(\xx) \rho_x^{\epsilon}(\xx) (V_x(\mu_y,x) - \mathcal{L}(\mu_x,\mu_y)) d\mu_x(\xx) \\&+ \int \phi_x(\xx) \frac{\partial \rho_x^{\epsilon}}{\partial t}(\xx) d\mu_x(\xx)
\end{split}
\end{align}
and the right hand side into
\begin{align}
    \begin{split}
        \frac{d}{dt} \int \phi(\xx) d\mu^{\epsilon}_x(\xx) = \int - \gamma \nabla \phi_x(\xx) \cdot \nabla V_x(\mu_y,\xx) d\mu^{\epsilon}_x(\xx)
    \end{split}
\end{align}
Note that comparing terms, we obtain
\begin{align}
\begin{split}
    &\int - \gamma \phi_x(\xx) \nabla \rho_x^{\epsilon}(\xx) \cdot \nabla V_x(\mu_y,\xx) \ d\mu_x(\xx) \\ &= \int \alpha \phi_x(\xx) \rho_x^{\epsilon}(\xx) (V_x(\mu_y,x) - \mathcal{L}(\mu_x,\mu_y)) - \phi_x(\xx) \frac{\partial \rho_x^{\epsilon}}{\partial t}(\xx) \ d\mu_x(\xx)
\end{split}
\end{align}
Since $\phi_x$ is arbitrary, it must be that 
\begin{align} \label{eq:rho_epsilon_def}
-\gamma \nabla \rho_x^{\epsilon}(\xx) \cdot \nabla V_x(\mu_y,\xx)  = \alpha \rho_x^{\epsilon}(\xx) (V_x(\mu_y,x) - \mathcal{L}(\mu_x,\mu_y)) - \frac{\partial}{\partial t}  \rho_x^{\epsilon}(\xx)
\end{align}
holds $\mu_x$-almost everywhere. Now, 
\begin{align}
\begin{split}
    \frac{d}{dt} \int \log \left(\rho_x^{\epsilon}\right) \ d\mu_x^{\epsilon} &= - \gamma \int \nabla \left(\log \left(\rho_x^{\epsilon}(\xx)\right) \right) \cdot \nabla V_x(\mu_y,\xx) \ d\mu_x^{\epsilon}(\xx) \\ &= - \gamma \int \frac{1}{\rho_x^{\epsilon}(\xx)} \nabla \left(\rho_x^{\epsilon}(\xx)\right) \cdot \nabla V_x(\mu_y,\xx) \ d\mu_x^{\epsilon}(\xx) \\&= \alpha \int (V_x(\mu_y,x) - \mathcal{L}(\mu_x,\mu_y)) \ d\mu_x^{\epsilon}(\xx) - \int \frac{1}{\rho_x^{\epsilon}(\xx)} \frac{\partial}{\partial t}  \rho_x^{\epsilon}(\xx) d\mu_x^{\epsilon}(\xx)
\end{split}
\end{align}
Here,
\begin{align}
    \int \frac{1}{\rho_x^{\epsilon}(\xx)} \frac{\partial}{\partial t}  \rho_x^{\epsilon}(\xx) d\mu_x^{\epsilon}(\xx) = \int \frac{\partial}{\partial t}  \rho_x^{\epsilon}(\xx) d\mu_x(\xx) = 0
\end{align}
And since
\begin{align}
    \mathcal{L}(\mu_x,\mu_y) = \int \frac{\delta \mathcal{L}}{\delta \mu_{x}}(\mu_x,\mu_y)(\xx) \ d\mu_x,
\end{align}
the first term yields \eqref{eq:derivative_entropy0}. We assumed that $\rho_x^\epsilon$ existed and was regular enough. To make the argument precise, we can define the density of $\mu_x^\epsilon$ with respect to $\mu_x$ to be a solution $\rho_x^\epsilon$ of \eqref{eq:rho_epsilon_def}, and thus specify $\mu_x^\epsilon$. 

Now, recall that $\mu_x^\star$ is an arbitrary measure in $\mathcal{P}(\mathcal{X})$. By linearity of $\mathcal{L}$ with respect to $\mu_x$,
\begin{align}
\begin{split} \label{eq:bound_bl}
    \int \frac{\delta \mathcal{L}}{\delta \mu_x}(\mu_x,\mu_y)(\xx) \ d(\mu_x^{\epsilon} - \mu_x) = \int \frac{\delta \mathcal{L}}{\delta \mu_x}(\mu_x,\mu_y)(\xx) \ d(\mu_x^\star - \mu_x) + \int \frac{\delta \mathcal{L}}{\delta \mu_x}(\mu_x,\mu_y)(\xx) \ d(\mu_x^{\epsilon} - \mu_x^\star) \\ \leq - (\mathcal{L}(\mu_x,\mu_y)-\mathcal{L}(\mu_x^\star,\mu_y)) + \|\frac{\delta \mathcal{L}}{\delta \mu_x}(\mu_x,\mu_y)\|_{\text{BL}} \|\mu_x^{\epsilon} - \mu_x^\star\|_\text{BL}^*
\end{split}
\end{align}
Notice that we can take $\|\frac{\delta \mathcal{L}}{\delta \mu_x}(\mu_x,\mu_y)\|_{\text{BL}}$ to be smaller than $B$ (defined in \eqref{eq:B_def}). If we integrate \eqref{eq:derivative_entropy0} and \eqref{eq:bound_bl} from 0 to $t$ and divide by $t$, we obtain
\begin{align}
\begin{split} \label{eq:integrated_eq}
    &\frac{1}{t}\int_0^t \mathcal{L}(\mu_x(s),\mu_y(s)) \ ds- \frac{1}{t}\int_0^t \mathcal{L}(\mu_x^\star,\mu_y(s)) \ ds \\ &\leq \frac{1}{\alpha t}(\mathcal{H}(\mu_{x,0}^{\epsilon}, \mu_{x,0}) -\mathcal{H}(\mu_{x}^{\epsilon}(t), \mu_{x}(t))) + \frac{B}{t} \int_0^t \|\mu_x^{\epsilon} - \mu_x^\star\|_\text{BL}^* \ ds
\end{split}
\end{align}
We bound the last term on the RHS:
\begin{align} \label{eq:bound_l_integral}
    \frac{B}{t} \int_0^t \|\mu_x^{\epsilon} - \mu_x^\star\|_\text{BL}^* \ ds \leq B\|\mu_{x,0}^{\epsilon} - \mu_x^\star\|_\text{BL}^* + \frac{B}{t} \int_0^t \|\mu_{x,0}^{\epsilon} - \mu_x^\epsilon\|_\text{BL}^* \ ds
\end{align}
And
\begin{align}
\begin{split} \label{eq:bound_lipschitz}
    \|\mu_{x}^{\epsilon}(t) - \mu_{x,0}^\epsilon\|_\text{BL}^* = \sup_{\|f\|_\text{BL} \leq 1, f \in C^2(\mathcal{X})} \int f \ d(\mu_{x}^{\epsilon}(t) - \mu_{x,0}^\epsilon) = \sup_{\|f\|_\text{BL} \leq 1, f \in C^2(\mathcal{X})} \int_0^t \frac{d}{ds} \int f \ d\mu_{x}^{\epsilon}(s) \ ds\\ = \sup_{\|f\|_\text{BL} \leq 1, f \in C^2(\mathcal{X})} - \int_0^t \int \gamma \nabla f(\xx) \cdot \nabla \frac{\delta \mathcal{L}}{\delta \mu_x}(\mu_x^{\epsilon},\mu_y)(\xx) \ d\mu_{x}^{\epsilon}(s) \ ds \leq \int_0^t \int \gamma B \ d\mu_{x}^{\epsilon}(s) \ ds = \gamma B t
\end{split}
\end{align}
Also, by linearity of $\mathcal{L}$ with respect to $\mu_y$, 
\begin{align} \label{eq:concavity_L}
    - \frac{1}{t}\int_0^t \mathcal{L}(\mu_x^\star,\mu_y(s)) \ ds = - \mathcal{L}(\mu_x^\star,\bar{\mu}_y(t))
\end{align}
If we use \eqref{eq:bound_l_integral}, \eqref{eq:bound_lipschitz} and \eqref{eq:concavity_L} and the non-negativeness of the relative entropy on \eqref{eq:integrated_eq}, we obtain:
\begin{align} \label{eq:eq_for_x}
    \frac{1}{t}\int_0^t \mathcal{L}(\mu_x(s),\mu_y(s)) \ ds- \mathcal{L}(\mu_x^\star,\bar{\mu}_y(t)) \leq \frac{\mathcal{H}(\mu_{x,0}^{\epsilon}, \mu_{x,0})}{4\alpha t} + B\|\mu_{x,0}^{\epsilon} - \mu_x^\star\|_\text{BL}^* + \frac{B^2 \gamma}{2} t \\
    -\frac{1}{t}\int_0^t \mathcal{L}(\mu_x(s),\mu_y(s)) \ ds + \mathcal{L}(\bar{\mu}_x(t),\mu_y^\star) \leq \frac{\mathcal{H}(\mu_{y,0}^{\epsilon}, \mu_{y,0})}{4\alpha t} + B\|\mu_{y,0}^{\epsilon} - \mu_y^\star\|_\text{BL}^* + \frac{B^2 \gamma}{2} t \label{eq:same_for_y}
\end{align}
Equation \eqref{eq:same_for_y} is obtained by performing the same argument switching the roles of $x$ and $y$, and $\mathcal{L}$ by $-\mathcal{L}$. By adding equations \eqref{eq:eq_for_x} and \eqref{eq:same_for_y} and considering the definition of $\mathcal{Q}$ in \eqref{eq:calQ_def}, we obtain the inequality \eqref{eq:thm3.10}.

\end{proof}

Notice that by taking the supremum wrt $\mu^\star_x, \mu^\star_y$ on \eqref{eq:thm3.10} we obtain a bound on the Nikaido-Isoda error of $(\bar{\mu}_{x}(t),\bar{\mu}_{y}(t))$ (see \eqref{eq:nikaido_def}).

Next, we will obtain a result like Lemma E.1 from \citet{chizat2019sparse} in which we bound $\mathcal{Q}$. The proof is a variation of the argument in Lemma E.1 from \citet{chizat2019sparse}, as in our case no measures are necessarily sparse. 

\begin{lemma} \label{thm:q_bound}
Let $\Theta$ be a Riemannian manifold of dimension $d$. Assume that $\text{Vol}(B_{\theta,\epsilon}) \geq e^{-K} \epsilon^d$ for all $\theta \in \Theta$, where the volume is defined of course in terms of the Borel measure\footnote{The metric of the manifold gives a natural choice of a Borel (volume) measure, the one given by integrating the canonical volume form.} of $\Theta$. If $\rho := \frac{d \mu_0}{d \theta}$ is the Radon-Nikodym derivative of $\mu_0$ with respect to the Borel measure of $\Theta$, assume that $\rho(\theta) \geq e^{-K'}$ for all $\theta \in \Theta$. The function $\mathcal{Q}_{\mu^\star,\mu_0}(\tau)$ defined in \eqref{eq:calQ_def} can be bounded by
\begin{align}
    \mathcal{Q}_{\mu^\star,\mu_0}(\tau) \leq \frac{d}{\tau} (1- \log d + \log \tau) + \frac{1}{\tau}(K + K')
\end{align}
\end{lemma}

\begin{proof}
We will choose $\mu^{\epsilon}$ in order to bound the infimum.
For $\theta \in \Theta, \epsilon > 0$, let $\xi_{\theta,\epsilon}$ be a probability measure on $\Theta$ with support on the ball $B_{\theta,\epsilon}$ of radius $\epsilon$ centered at $\theta$ and proportional to the Borel measure for all subsets of the ball. Let us define the measure
\begin{align}
    \mu^{\epsilon}(A) = \int_{\Theta} \xi_{\theta,\epsilon}(A) \ d\mu^\star(\theta)
\end{align}
for all Borel sets $A$ of $\mathcal{X}$. Now, we can bound $\|\mu^{\epsilon} - \mu^\star\|_{\text{BL}}^* \leq W_1(\mu^{\epsilon},\mu^\star)$. Let us consider the coupling $\gamma$ between $\mu^{\epsilon}$ and $\mu^\star$ defined as:
\begin{align}
    \gamma(A \times B) = \int_{A} \xi_{\theta,\epsilon}(B) \ d\mu^\star(\theta)
\end{align}
for $A, B$ arbitrary Borel sets of $\Theta$. Notice that $\gamma$ is indeed a coupling between $\mu^{\epsilon}$ and $\mu^\star$, because $\gamma(A \times \Theta) = \mu^\star(A)$ and $\gamma(\Theta \times B) = \mu^{\epsilon}(B)$. Hence,
\begin{align} \label{eq:w1_bound}
W_1(\mu^{\epsilon},\mu^\star) \leq \int_{\Theta \times \Theta} d_\Theta(\theta,\theta') \ d\gamma(\theta,\theta') = \int_{\Theta} \frac{1}{\text{Vol}(B_{\theta',\epsilon})}\int_{B_{\theta',\epsilon}} d_\Theta(\theta,\theta') \ d\theta \ d\mu^\star(\theta')
\end{align}
where the inner integral is with respect to the Borel measure on $\Theta$. Since $d_\Theta(\theta,\theta') \leq \epsilon$ for all $\theta \in B_{\theta',\epsilon}$, we conclude from that \eqref{eq:w1_bound} that $W_1(\mu^{\epsilon},\mu^\star) \leq \epsilon$. 

Next, let us bound the relative entropy term. Define $\rho_\epsilon$ as the Radon-Nikodym derivative of $\mu^{\epsilon}$ with respect to the Borel measure of $\Theta$, i.e.
\begin{align}
\rho_\epsilon(\theta) := \frac{d \mu^{\epsilon}}{d \theta}(\theta) = \int_\Theta \frac{1}{\text{Vol}(B_{\theta',\epsilon})} {\mathds{1}}_{B_{\theta',\epsilon}}(\theta) \ d\mu^\star(\theta').
\end{align}
Also, recall that $\rho := \frac{d \mu_0}{d \theta}$. Then, we write
\begin{align} \label{eq:rel_entropy_decomposition}
    \mathcal{H}(\mu^\epsilon,\mu_0) = \int_\Theta \log \frac{\rho_\epsilon}{\rho} d\mu^{\epsilon} = \int_\Theta \log (\rho_\epsilon) \rho_\epsilon d\theta - \int_\Theta \log (\rho) \rho_\epsilon d\theta.
\end{align}
On the one hand, we use the convexity of the function $x \rightarrow x \log x$:
\begin{align}
\begin{split}
    \rho_\epsilon(\theta) \log \rho_\epsilon(\theta) &= \left(\int_\Theta \frac{1}{\text{Vol}(B_{\theta',\epsilon})} {\mathds{1}}_{B_{\theta',\epsilon}}(\theta) \ d\mu^\star(\theta') \right) \log \left( \int_\Theta \frac{1}{\text{Vol}(B_{\theta',\epsilon})} {\mathds{1}}_{B_{\theta',\epsilon}}(\theta) \ d\mu^\star(\theta') \right) \\ &\leq \int_\Theta \left( \frac{1}{\text{Vol}(B_{\theta',\epsilon})} {\mathds{1}}_{B_{\theta',\epsilon}}(\theta) \right) \log \left( \frac{1}{\text{Vol}(B_{\theta',\epsilon})} {\mathds{1}}_{B_{\theta',\epsilon}}(\theta) \right) d\mu^\star(\theta'). 
\end{split}
\end{align}
We use Fubini's theorem:
\begin{align}
\begin{split} \label{eq:bound_entropy_epsilon}
    &\int_\Theta \rho_\epsilon(\theta) \log \rho_\epsilon(\theta) \ d\theta \leq \int_\Theta \int_\Theta \left( \frac{1}{\text{Vol}(B_{\theta',\epsilon})} {\mathds{1}}_{B_{\theta',\epsilon}}(\theta) \right) \log \left( \frac{1}{\text{Vol}(B_{\theta',\epsilon})} {\mathds{1}}_{B_{\theta',\epsilon}}(\theta) \right) \ d\theta \ d\mu^\star(\theta') \\ &= \int_\Theta \frac{1}{\text{Vol}(B_{\theta',\epsilon})} \int_{B_{\theta',\epsilon}} - \log \left( \text{Vol}(B_{\theta',\epsilon}) \right) \ d\theta \ d\mu^\star(\theta') = - \int_\Theta \log \left( \text{Vol}(B_{\theta',\epsilon}) \right) d\mu^\star(\theta') \\ &\leq - d \log \epsilon + K 
\end{split}
\end{align}
where $d$ is the dimension of $\Theta$ and $K$ is a constant such that $\text{Vol}(B_{\theta',\epsilon}) \geq e^{-K} \epsilon^d$ for all $\theta' \in \Theta$.

On the other hand, 
\begin{align}
\begin{split} \label{eq:bound_cross_entropy}
    - \int_\Theta \log (\rho(\theta)) \rho_\epsilon(\theta) \ d\theta &= \int_\Theta \frac{1}{\text{Vol}(B_{\theta',\epsilon})} \int_{\text{Vol}(B_{\theta',\epsilon})} -\log (\rho(\theta)) \ d\theta \ d\mu^{\star}(\theta') \\&\leq \int_\Theta \frac{1}{\text{Vol}(B_{\theta',\epsilon})} \int_{\text{Vol}(B_{\theta',\epsilon})} K' \ d\theta \ d\mu^{\star}(\theta') = K'
\end{split}
\end{align}
where $K'$ is defined such that $\rho(\theta) \geq e^{-K'}$ for all $\theta \in \Theta$.

By plugging \eqref{eq:bound_entropy_epsilon} and \eqref{eq:bound_cross_entropy} into \eqref{eq:rel_entropy_decomposition} we obtain:
\begin{align}
    \|\mu^\star - \mu^\epsilon\|_{\text{BL}}^* + \frac{1}{\tau} \mathcal{H}(\mu^\epsilon,\mu_0) \leq \epsilon +\frac{1}{\tau} (- d \log \epsilon + K + K').
\end{align}
If we optimize the bound with respect to $\epsilon$ we obtain the final result.
\end{proof}

\begingroup
\def\thetheorem{\ref{thm:wfr_main_thm}}
\begin{theorem}
Let $\epsilon > 0$ arbitrary. Suppose that $\mu_{x,0}, \mu_{y,0}$ are such that their Radon-Nikodym derivatives with respect to the Borel measures of $\mathcal{X},\mathcal{Y}$ are lower-bounded by $e^{-K'_x}, e^{-K'_y}$ respectively. For any $\delta \in (0,1/2)$, there exists a constant $C_{\delta, \mathcal{X},\mathcal{Y},K'_x,K'_y} > 0$ depending on the dimensions of $\mathcal{X}, \mathcal{Y}$, their curvatures and $K'_x, K'_y$, such that if $\gamma/\alpha < 1$ and
\begin{align} \label{eq:thm_bound_alpha_beta}
    \frac{\gamma}{\alpha} \leq \left(\frac{\epsilon}{C_{\delta, \mathcal{X},\mathcal{Y},K'_x,K'_y}}\right)^{\frac{2}{1-\delta}}
\end{align}
Then, at $t_0 = (\alpha \gamma)^{-1/2}$ we have
\begin{align}
    \text{NI}(\bar{\mu}_{x}(t_0),\bar{\mu}_{y}(t_0)) := \sup_{\mu^\star_x,\mu^\star_y} \mathcal{L}(\bar{\mu}_{x}(t_0), \mu^\star_y) - \mathcal{L}(\mu^\star_x, \bar{\mu}_{y}(t_0)) \leq \epsilon
\end{align}
\end{theorem}
\addtocounter{theorem}{-1}
\endgroup

\begin{proof}
We plug the bound of Theorem~\ref{thm:q_bound} into the result of Theorem~\ref{thm:wfr_ni_error}, obtaining
\begin{align} \label{eq:thm3.10_2}
\begin{split}
    \mathcal{L}(\bar{\mu}_{x}(t), \mu^\star_y) - \mathcal{L}(\mu^\star_x, \bar{\mu}_{y}(t)) &\leq \frac{d_x}{\alpha t} (1- \log d_x + \log (\alpha B t)) \\&+\frac{d_y}{\alpha t} (1- \log d_y + \log (\alpha B t)) \\&+ \frac{1}{\alpha t}(K_x + K'_x + K_y + K'_y) + \gamma B^2 t
\end{split}
\end{align}
Now, we set $t = (\alpha \gamma)^{-1/2}$, and thus the right hand side becomes
\begin{align} 
\begin{split}\label{eq:bound_alpha_beta}
    \sqrt{\frac{\gamma}{\alpha}} \left(d_x \left(1- \log \frac{d_x}{B} + \log \sqrt{\frac{\alpha}{\gamma}} \right) + d_y \left(1- \log \frac{d_y}{B} + \log  \sqrt{\frac{\alpha}{\gamma}} \right) + K_x + K'_x + K_y + K'_y + B^2 \right)
\end{split}
\end{align}
Let $\epsilon > 0$ arbitrary. We want \eqref{eq:bound_alpha_beta} to be lower or equal than $\epsilon$. For any $\delta$ such that $0 < \delta < 1/2$, there exists $C_{\delta}$ such that $\log(x) \leq C_\delta x^\delta$. This yields
\begin{align} 
\begin{split} \label{eq:bound_alpha_beta_2}
    &\sqrt{\frac{\gamma}{\alpha}} \left(d_x \left(1- \log \frac{d_x}{B} + C_\delta \left(\frac{\alpha}{\gamma} \right)^{-\delta/2} \right) +d_y \left(1- \log \frac{d_y}{B} + C_\delta \left(\frac{\alpha}{\gamma} \right)^{-\delta/2} \right) \right) \\ &+ \sqrt{\frac{\gamma}{\alpha}} \left(K_x + K'_x + K_y + K'_y + B^2 \right)
\end{split}
\end{align}
If we set $\gamma < \alpha$, $(\gamma/\alpha)^{-\delta/2} > 1$ then \eqref{eq:bound_alpha_beta_2} is upper-bounded by
\begin{align} \label{eq:bound_alpha_beta_3}
    \left(\frac{\gamma}{\alpha}\right)^{\frac{1-\delta}{2}}  \left(d_x (1- \log \frac{d_x}{B} + C_\delta) + d_y (1- \log \frac{d_y}{B} + C_\delta) + K_x + K'_x + K_y + K'_y + B^2 \right)
\end{align}
If we bound this by $\epsilon$, we obtain the bound in \eqref{eq:thm_bound_alpha_beta}.
\end{proof}

\begin{corollary} \label{cor:thm4_dim}
Let ${(\mathcal{X}_{d_x},\mathcal{Y}_{d_y}, l_{d_x,d_y})}_{d_x \in \mathbb{N}, d_y \in \mathbb{N}}$ be a family indexed by $\mathbb{N}^2$. Assume that $\mu_{x,0}, \mu_{y,0}$ are set to be the Borel measures in $\mathcal{X}_{d_x}, \mathcal{Y}_{d_y}$, that $\mathcal{X}_{d_x}, \mathcal{Y}_{d_y}$ are locally isometric to the $d_x,d_y$-dimensional Euclidean spaces, and that the volumes of $\mathcal{X}_{d_x}$, $\mathcal{Y}_{d_y}$ grow no faster than exponentially on the dimensions $d_x, d_y$. Assume that $l_{d_x,d_y}$ are such that $B$ is constant. Then, we can rewrite \eqref{eq:thm_bound_alpha_beta} as 
\begin{align}
    \frac{\gamma}{\alpha} \leq O\left( \left(\frac{\epsilon}{(d_x+d_y)\log(B) + d_x \log(d_x) + d_y \log(d_y) + B^2} \right)^{\frac{2}{1-\delta}}\right)
\end{align}
\end{corollary}

\begin{proof}
The volume of $n$-dimensional ball of radius $r$ in $n$-dimensional Euclidean space is 
\begin{equation}
    V_n(r) = \frac{\pi^{n/2}}{\Gamma(\frac{n}{2} + 1)} R^n,
\end{equation}
and hence, if $\mathcal{X}, \mathcal{Y}$ are locally isometric to the $d_x$ and $d_y$-dimensional Euclidean spaces we can take
\begin{align}
\begin{split}
    K_x &= \log \Gamma \left(\frac{d_x}{2} + 1 \right) - \frac{d_x}{2} \log(\pi) \leq \left(\frac{d_x}{2} + 1 \right) \log \left(\frac{d_x}{2} + 1 \right) - \frac{d_x}{2} \log(\pi) \leq O(d_x \log d_x)\\
    K_y &= \log \Gamma(\frac{d_y}{2} + 1) - \frac{n}{2} \log(\pi) \leq O(d_x \log d_x)
\end{split}
\end{align}
If the volumes of $\mathcal{X}, \mathcal{Y}$ grow no faster than an exponential of the dimensions $d_x, d_y$ and we take $\mu_{x,0}, \mu_{y,0}$ to be the Borel measures, we can take $K'_x = \log(\text{Vol}(\mathcal{X})), K'_y = \log(\text{Vol}(\mathcal{Y}))$ to be constant with respect to the dimensions $d_x, d_y$.
\end{proof}

%% file: app_thm5.tex
\section{Proof of \autoref{thm:convergence_mean_field}(i)} \label{sec:LGV_prop}
\subsection{Preliminaries}
Throughout the section we will use the techniques shown in \autoref{sec:sdes_manifolds} to deal with SDEs on manifolds. Effectively, this means that for SDEs we have additional drift terms $\mathbf{\hat{h}}_x$ or $\mathbf{\hat{h}}_x$ induced by the geometry of the manifold, and that we must project the variations of the Brownian motion onto the tangent space. 

Define the processes $\mathbf{X}^n = (X^{1}, \dots, X^{n})$ and $\mathbf{Y}^n = (Y^{1}, \dots, Y^{n})$ such that for all $i \in \{1, \dots, n\},$
\begin{align}
\begin{split} \label{eq:particle_system_LGV2}
    dX_t^{i} &= \left(- \frac{1}{n} \sum_{j=1}^n \nabla_x \ell(X_t^{i}, Y_t^{j}) + \mathbf{\hat{h}}_x(X_t^{i}) \right)\ dt + \sqrt{2 \beta^{-1}} \ \text{Proj}_{T_{X_t^{i}} \mathcal{X}} (dW_t^i), \quad X_0^{n,i} = \xi^{i} \sim \mu_{x,0} \\
    dY_t^{i} &= \left(\frac{1}{n} \sum_{j=1}^n \nabla_y \ell(X_t^{j}, Y_t^{i}) + \mathbf{\hat{h}}_y(Y_t^{i})\right)\ dt + \sqrt{2 \beta^{-1}} \ \text{Proj}_{T_{Y_t^{i}} \mathcal{Y}}(d\bar{W}_t^i), \quad Y_0^{n,i} = \bar{\xi}^{i} \sim \mu_{y,0}\\
\end{split}
\end{align}
where $\mathbf{W}_t = (W_t^1, \dots, W_t^n)$, and
$\mathbf{\bar{W}}_t = (\bar{W}_t^1, \dots, \bar{W}_t^n)$ are Brownian motions on $\mathbb{R}^{n D_x}$ and $\mathbb{R}^{n D_y}$ respectively. Notice that $\mathbf{X}_t$ is valued in $\mathcal{X}^n \subseteq \mathbb{R}^{n D_x}$ and $\mathbf{Y}_t$ is valued in $\mathcal{Y}^n \subseteq \mathbb{R}^{n D_y}$.
\eqref{eq:particle_system_LGV2} can be seen as a system of $2n$ interacting particles in which each particle of one player interacts with all the particles of the other one. It also corresponds to noisy continuous-time mirror descent on parameter spaces for an augmented game in which there are $n$ replicas of each player, choosing $\frac{1}{2}\|\cdot\|_2^2$ for the mirror map.

Now, define $\tilde{\mathbf{X}} = (\tilde{X}^{1}, \dots, \tilde{X}^{n})$ and $\tilde{\mathbf{Y}} = (\tilde{Y}^{1}, \dots, \tilde{Y}^{n})$ for all $i \in \{1, \dots, n\}$ let
\begin{align}
\begin{split} \label{eq:mean_field_syst}
    d\tilde{X}_t^{i} &= \left(- \int_{\mathcal{Y}} \nabla_x \ell(\tilde{X}_t^{i}, \yy) \ d\mu_{y,t} + \mathbf{\hat{h}}_x(\tilde{X}_t^{i})\right)\ dt + \sqrt{2 \beta^{-1}} \ \text{Proj}_{T_{\tilde{X}_t^{i}} \mathcal{X}} (dW_t^i), \\
    d\tilde{Y}_t^{i} &= \left( \int_{\mathcal{X}} \nabla_y \ell(\xx, \tilde{Y}_t^{i}) \ d\mu_{x,t} + \mathbf{\hat{h}}_y(\tilde{Y}_t^{i})\right)\ dt + \sqrt{2 \beta^{-1}} \ \text{Proj}_{T_{\tilde{Y}_t^{i}} \mathcal{Y}}(d\bar{W}_t^i),\\
    \tilde{X}_0^{i} &= \xi^{i} \sim \mu_{x,0}, \quad \mu_{y,t} = \text{Law}(\tilde{Y}_t^{i}), \quad \tilde{Y}_0^{i} = \bar{\xi}^{i} \sim \mu_{y,0}, \quad \mu_{x,t} = \text{Law}(\tilde{X}_t^{i})
\end{split}
\end{align}

\begin{lemma}[Forward Kolmogorov equation] \label{lem:forward_kolmogorov_entropy}
The laws $(\mu_x)_{t \in [0,T]}, (\mu_y)_{t \in [0,T]}$ of a solution $\tilde{X},\tilde{Y}$ of \eqref{eq:mean_field_syst} with $n=1$ (seen as elements of $\mathcal{C}([0,T],\mathcal{P}(\mathcal{X})),\mathcal{C}([0,T],\mathcal{P}(\mathcal{Y}))$) are a solution of \eqref{eq:entropic_interacting_flow2}.
\begin{align}
    \begin{split} \label{eq:entropic_interacting_flow2}
        \begin{cases} 
        \partial_t \mu_x = \nabla_x \cdot (\mu_x \nabla_x V_x(\mu_y,x)) + \beta^{-1} \Delta_x \mu_x, \quad \mu_x(0) = \mu_{x,0} \\
        \partial_t \mu_y = -\nabla_y \cdot (\mu_y \nabla_y V_y(\mu_x,y)) + \beta^{-1} \Delta_y \mu_y, \quad \mu_y(0) = \mu_{y,0}
        \end{cases}
    \end{split}
\end{align}
\end{lemma}
\begin{proof}
We sketch the derivation for the forward Kolmogorov equation on manifolds. First, we define the semigroups
\begin{align}
        P_t^x \phi_x(\xx) = \mathbb{E}[\phi_x(\tilde{X}_t) | \tilde{X}_0=x], \quad
        P_t^y \phi_y(\yy) = \mathbb{E}[\phi_y(\tilde{Y}_t) | \tilde{Y}_0=y],
\end{align}
where $\tilde{X},\tilde{Y}$ are solutions of \eqref{eq:mean_field_syst} with $n=1$. We obtain that if 
$\mathcal{L}_{t}^{x}, \mathcal{L}_{t}^{y}$ are the infinitesimal generators (i.e., $\mathcal{L}_{t}^{x} \phi_x(\xx) = \lim_{t \rightarrow 0^+} \frac{1}{t}(P_t^x \phi_x(\xx) -\phi_x(\xx))$), the backward Kolmogorov equations $\frac{d}{dt} P_t^x \phi_x(\xx) = \mathcal{L}_{t}^{x} P_t^x \phi_x(\xx), \frac{d}{dt} P_t^y \phi_y(\yy) = \mathcal{L}_{t}^{y} P_t^y \phi_y(\yy)$ hold for $\phi_x, \phi_y$ in the domains of the generators. Since $\mathcal{L}_{t}^{x}$ and $P_t^x$ commute for these choices of $\phi_x$, we have $\frac{d}{dt} P_t^x \phi_x(\xx) = P_t^x \mathcal{L}_{t}^{x} \phi_x(\xx), \frac{d}{dt} P_t^y \phi_y(\yy) = P_t^y \mathcal{L}_{t}^{y} \phi_y(\yy)$. By integrating these two equations over the initial measures $\mu_{x,0}, \mu_{y,0}$, we get
\begin{align}
    \frac{d}{dt} \int \phi_x(\xx) \ d\mu_{x,t} = \int \mathcal{L}_{t}^{x} \phi_x(\xx) \ d\mu_{x,t}, \quad \frac{d}{dt} \int \phi_y(\yy) \ d\mu_{y,t} = \int \mathcal{L}_{t}^{y} \phi_y(\yy) \ d\mu_{y,t}.  
\end{align}
We can write an explicit form for $\mathcal{L}_{t}^{x} P_t^x \phi_x(\xx)$ by using It\^{o}'s lemma on \eqref{eq:mean_field_syst}:
\begin{align}
    \mathcal{L}_{t}^{x} \phi_x(\xx) = \left( \int_{\mathcal{Y}} \nabla_x \ell(\xx, \yy) \ d\mu_{y,s} \ ds - \mathbf{\hat{h}}_x(\xx)\right) \nabla_x \phi_x(\xx) + \beta^{-1} \text{Tr}\left(\left(\text{Proj}_{T_{\xx} \mathcal{X}} \right)^\top H \phi_x(\xx) \ \text{Proj}_{T_{\xx} \mathcal{X}} \right),
\end{align}
where we use $\text{Proj}_{T_{\tilde{X}_t^{i}} \mathcal{X}}$ to denote its matrix in the canonical basis.

Let 
$\{ \xi_k\}$ be a partition of unity for $\mathcal{X}$ (i.e. a set of functions such that $\sum_{k} \xi_k(x) = 1$) in which each $\xi_k$ is regular enough and supported on a patch of $\mathcal{X}$. We can write
\begin{align}
\begin{split}
    \frac{d}{dt} \int_{\mathcal{X}} \phi_x(\xx) \ d\mu_{x,t}(\xx) &= \frac{d}{dt} \int_{\mathcal{X}} \phi_x(\xx) \ d\mu_{x,t}(\xx) = \sum_{k} \frac{d}{dt} \int_{\mathcal{X}} \xi_k(\xx) \phi_x(\xx) \ d\mu_{x,t}(\xx) \\ &= \sum_{k} \int \mathcal{L}_{t}^{x} (\xi_k(\xx) \phi_x(\xx)) \ d\mu_{x,t}
\end{split}
\end{align}
Now, let $\tilde{\phi}_x^k(\xx) = \xi_k(\xx) \phi_x(\xx)$.
\begin{align}
\begin{split}
    &\int_{\mathcal{X}} \mathcal{L}_{t}^{x} \tilde{\phi}_x^k(\xx) \ d\mu_{x,t} \\ &= \int_{\mathcal{X}} \left( \nabla_x V_x(\mu_{y,s}, \xx) - \mathbf{\hat{h}}_x(\xx)\right) \nabla_x \tilde{\phi}_x^k(\xx) + \beta^{-1} \text{Tr}\left(\left(\text{Proj}_{T_{\xx} \mathcal{X}} \right)^\top H \tilde{\phi}_x^k(\xx) \ \text{Proj}_{T_{\xx} \mathcal{X}} \right) \ d\mu_{x,t}
    \end{split}
\end{align}
Notice that this equation is analogous to \eqref{eq:KFP_manifold}. We reverse the argument made in \autoref{sec:sdes_manifolds}. Using the fact that the support of $\tilde{\phi}_x^k(\xx)$ is contained on some patch of $\mathcal{X}$ given by the mapping $\psi_k : U_{\mathbb{R}^d} \subseteq \mathbb{R}^d \rightarrow U \subseteq \mathcal{X} \subseteq \mathbb{R}^D$, the corresponding Fokker-Planck on $U_{\mathbb{R}^d}$ is 
\begin{align}
\begin{split}
    &\frac{d}{dt} \int_{U_{\mathbb{R}^d}} \tilde{\phi}_x^k(\psi_k(q)) \ d(\psi_k^{-1})_*\mu_{x,t}(q) \\ &= \int_{U_{\mathbb{R}^d}} \nabla V_x(\mu_{y,s}, \psi_k(q)) \cdot \nabla \tilde{\phi}_x^k(\psi_k(q))  + \beta^{-1} \Delta \tilde{\phi}_x^k(\psi_k(q))  \ d(\psi_k^{-1})_*\mu_{x,t}(q),
\end{split}
\end{align}
where the gradients and the Laplacian are in the metric inherited from the embedding (as in \autoref{sec:sdes_manifolds}). The pushforward definition implies
\begin{align}
    \frac{d}{dt} \int_{\mathcal{X}} \tilde{\phi}_x^k(\xx) \ d\mu_{x,t}(\xx) = \int_{U_{\mathbb{R}^d}} \nabla V_x(\mu_{y,s}, \xx) \cdot \nabla \tilde{\phi}_x^k(\xx)  + \beta^{-1} \Delta \tilde{\phi}_x^k(\xx)  \ d\mu_{x,t}(\xx),
\end{align}
By substituting $\tilde{\phi}_x^k(\xx) = \xi_k(\xx) \phi_x(\xx)$, summing for all $k$ and using $\sum_k \xi_k(\xx) = 1$, we obtain:
\begin{align}
\begin{split}
    \frac{d}{dt} \int_{\mathcal{X}} \phi_x(\xx) \ d\mu_{x,t}(\xx) &= \int_{\mathcal{X}} \nabla_x V_x(\mu_{y,s}, \xx) \cdot \nabla_x \phi_x(\xx) + \beta^{-1} \Delta_x \phi_x(\xx) \ d\mu_{x,t}(\xx)
\end{split}
\end{align}
which is the same as the first equation in \eqref{eq:entropic_interacting_flow}. The second equation is obtained analogously.
\end{proof}

Let $\mu_x^n = \frac{1}{n} \sum_{i=1}^n \delta_{X^{i}}$ be a $\mathcal{P}(\mathcal{C}([0,T],\mathcal{X}))$-valued random element that corresponds to the empirical measure of a solution $\mathbf{X}^n$ of \eqref{eq:particle_system_LGV2}. Analogously, let $\mu_y^n = \frac{1}{n} \sum_{i=1}^n \delta_{Y^{i}}$ be a $\mathcal{P}(\mathcal{C}([0,T],\mathcal{Y}))$-valued random element corresponding to the empirical measure of $\mathbf{Y}^n$.

Define the 2-Wasserstein distance on $\mathcal{P}(\mathcal{C}([0,T],\mathcal{X}))$ as 
\begin{equation} \label{eq:2_wass_C}
    \mathcal{W}_{2}^2(\mu, \nu) := \inf_{\pi \in \Pi(\mu, \nu)} \int_{C([0,T],\mathcal{X})^2} d(x,y)^2 \ d\pi(x, y)
\end{equation}
where $d(x,y) = \sup_{t \in [0,T]} d_{\mathcal{X}}(x(t),y(t))$. Define it analogously on $\mathcal{P}(\mathcal{C}([0,T],\mathcal{Y}))$.

We state a stronger version of the law of large numbers in the first statement of \autoref{thm:convergence_mean_field}(i).
\begin{theorem} \label{thm:propagation_chaos_LGV}
There exists a solution of the coupled McKean-Vlasov SDEs \eqref{eq:mean_field_syst}. Pathwise uniqueness and uniqueness in law hold. Let $\mu_x \in \mathcal{P}(\mathcal{C}([0,T],\mathcal{X})), \mu_y \in \mathcal{P}(\mathcal{C}([0,T],\mathcal{Y}))$ be the unique laws of the solutions for $n=1$ (all pairs have the same solutions). Then,
\begin{align}
    \mathbb{E}[\mathcal{W}_{2}^2(\mu_x^n, \mu_x) + \mathcal{W}_{2}^2(\mu_y^n, \mu_y)] \xrightarrow{n \rightarrow \infty} 0
\end{align}
\end{theorem}

Let us comment on why \autoref{thm:propagation_chaos_LGV} implies the first statement in \autoref{thm:convergence_mean_field}(i). We make use of the mapping $\mathcal{P}(\mathcal{C}([0,T],\mathcal{X})) \ni \mu \mapsto (\mu_t)_{t \in [0,T]} \in \mathcal{C}([0,T],\mathcal{P}(\mathcal{X}))$ into the time marginals. By the definition \eqref{eq:2_wass_C}, $\sup_{t \in [0,t]} \mathcal{W}_{2}^2(\mu_{x,t}^n, \mu_{x,t}) \leq \mathcal{W}_{2}^2(\mu_x^n, \mu_x)$ and the same holds for $\mu_y^n, \mu_y$. At this point, \autoref{lem:forward_kolmogorov_entropy} states that $(\mu_x)_{t \in [0,T]}, (\mu_y)_{t \in [0,T]}$ is a solution of the mean-field ERIWGF \eqref{eq:entropic_interacting_flow2} and concludes the argument.
The proof of \autoref{thm:propagation_chaos_LGV} uses a propagation of chaos argument, originally due to \citet{sznitmantopics1991} in the context of interacting particle systems. Our argument follows Theorem 3.3 of \citet{lackermean2018}.

\subsection{Existence and uniqueness}
We prove existence and uniqueness of the system given by
\begin{align}
\begin{split} \label{eq:mean_field_syst2}
    \tilde{X}_t &= \int_0^t \left(- \int_{\mathcal{Y}} \nabla_x \ell(\tilde{X}_s, \yy) \ d\mu_{y,s} \ ds + \mathbf{\hat{h}}_x(\tilde{X}_s)\right)\ ds + \sqrt{2 \beta^{-1}} \ \int_0^t \text{Proj}_{T_{\tilde{X}_s} \mathcal{X}} (dW_s), \\
    \tilde{Y}_t &= \int_0^t \left( \int_{\mathcal{X}} \nabla_y \ell(\xx, \tilde{Y}_s) \ d\mu_{x,s} + \mathbf{\hat{h}}_y(Y_s^{n,i})\right)\ ds + \sqrt{2 \beta^{-1}} \ \int_0^t \text{Proj}_{T_{\tilde{Y}_s} \mathcal{Y}}(d\bar{W}_s),\\
    \mu_{x,t} &= \text{Law}(\tilde{X}_t^{n}), \quad \mu_{y,t} = \text{Law}(\tilde{Y}_t^{n}), \quad \tilde{X}_0 = \xi \sim \mu_{x,0}, \quad \tilde{Y}_0 = \bar{\xi} \sim \mu_{y,0}.
\end{split}
\end{align}
Path-wise uniqueness means that given $W, \bar{W}, \xi, \bar{\xi}$, two solutions are equal almost surely. Uniqueness in law means that regardless of the Brownian motion and the initialization random variables chosen (as long as they are $\mu_{x,0}$ and $\mu_{y,0}$-distributed), the law of the solution is unique. We prove that both hold for \eqref{eq:mean_field_syst2}.

We have that for all $\xx, \xx' \in \mathcal{X}, \mu, \nu \in \mathcal{P}(\mathcal{Y})$, 
\begin{align} \label{eq:bound_lipschitz_vol}
    \bigg| \int \nabla_x \ell(\xx, \yy) \ d\mu - \int \nabla_x \ell(\xx', \yy) \ d\nu \bigg| \leq L(d(\xx,\xx') + \mathcal{W}_2(\mu, \nu))
\end{align}
This is obtained by adding and subtracting the term $\int \nabla_x \ell(\xx'\yy) \ d\mu$, by using the triangle inequality and the inequality $\mathcal{W}_1(\mu, \nu)) \leq \mathcal{W}_2(\mu, \nu))$ (which is proven using the Cauchy-Schwarz inequality). Hence,
\begin{align} \label{eq:lipschitz_transport}
    \bigg| \int \nabla_x \ell(\xx, \yy) \ d\mu - \int \nabla_x \ell(\xx', \yy) \ d\nu \bigg|^2 \leq 2 L^2 (d(\xx,\xx')^2 + \mathcal{W}^2_2(\mu, \nu))
\end{align}
On the other hand, using the regularity of the manifold, there exists $\mathcal{L}_{\mathcal{X}}$ such that
\begin{align}
\begin{split}
    |\mathbf{\hat{h}}_x(\xx) - \mathbf{\hat{h}}_x(\xx')| \leq L_{\mathcal{X}} d(\xx,\xx'), \\
    |\text{Proj}_{T_\xx \mathcal{X}} - \text{Proj}_{T_{\xx'} \mathcal{X}}| \leq L_{\mathcal{X}} d(\xx,\xx')
\end{split}
\end{align}
where $\text{Proj}_{T_\xx \mathcal{X}}$ denotes its matrix in the canonical basis and the norm in the second line is the Frobenius norm. Also, let $\|\xx-\xx'\|$ be the Euclidean norm of $\mathcal{X}$ in $\mathbb{R}^{D_x}$ (the Euclidean space where $\mathcal{X}$ is embedded) and let $K_{\mathcal{X}} > 1$ be such that $d(\xx, \xx') \leq K_{\mathcal{X}}\|\xx-\xx'\|$.

Let $\mu_y, \nu_y \in \mathcal{P}(\mathcal{C}([0,T],\mathcal{X}))$ and let $X^{\mu_y},X^{\nu_y}$ be the solutions of the first equation of \eqref{eq:mean_field_syst2} when we plug $\mu_y$ ($\nu_y$ resp.) as the measure for the other player. $X^{\mu_y}$ and $X^{\nu_y}$ exist and are unique by the classical theory of SDEs (see Chapter 18 of \citet{kallenberg2002foundations}). 
Following the procedure in Theorem 3.3 of \citet{lackermean2018}, we obtain
\begin{align}
    \begin{split} \label{eq:before_gronwall}
        \mathbb{E}[\|X^{\mu_y}-X^{\nu_y}\|^2_t] &\leq 3t \mathbb{E}\bigg[\int_0^t \bigg| \int \nabla_x \ell(X^{\mu_y}, \yy) \ d\mu_{y,r} - \int \nabla_x \ell(X^{\nu_y}, \yy) \ d\nu_{y,r} \bigg|^2 \ dr \bigg] \\ &+ 3t\mathbb{E}\bigg[\int_0^t |\mathbf{\hat{h}}_x(X^{\mu_y}) - \mathbf{\hat{h}}_x(X^{\nu_y})|^2 \ dr \bigg] \\&+ 12\mathbb{E}\bigg[\int_0^t |\text{Proj}_{T_\xx \mathcal{X}} - \text{Proj}_{T_{\xx'} \mathcal{X}}|^2 \ dr \bigg]\\ &\leq 3(3t + 4) \tilde{L}^2 \mathbb{E}\bigg[\int_0^t (\|X^{\mu_y}-X^{\nu_y}\|^2_r + \mathcal{W}_2^2(\mu_{y,r},\nu_{y,r})) \ dr\bigg],
    \end{split}
\end{align}
where $\tilde{L}^2 = (L^2 + L_{\mathcal{X}}^2) K_{\mathcal{X}}^2$.
Using Fubini's theorem and Gronwall's inequality, we obtain
\begin{align} \label{eq:gron}
    \mathbb{E}[\|X^{\mu_y}-X^{\nu_y}\|^2_t] \leq 3(3T + 4) \tilde{L}^2 \exp(3(3T + 4) \tilde{L}^2) \int_0^t \mathcal{W}_2^2(\mu_{y,r},\nu_{y,r})) \ dr
\end{align}
Let $C_{T} := 3(3T + 4) \tilde{L}^2 \exp(3(3T + 4) \tilde{L}^2)$. 
For $\mu, \nu \in \mathcal{P}(C([0,T],\mathcal{X}))$, define 
\begin{align}
    \mathcal{W}_{2, t}^2(\mu, \nu) := \inf_{\pi \in \Pi(\mu, \nu)} \int_{C([0,T],\mathcal{X})^2} \sup_{r \in [0,t]} d(x(r),y(r)) \ \pi(dx, dy)
\end{align}
Hence, \eqref{eq:gron} and the bound $\mathcal{W}_2^2(\mu_{y,r},\nu_{y,r}) \leq \mathcal{W}_{2, r}^2(\mu_y, \nu_y)$ yield
\begin{align}
    \mathbb{E}[\|X^{\mu_y}-X^{\nu_y}\|^2_t] \leq  C_{T} \int_0^t \mathcal{W}_{2, r}^2(\mu_y, \nu_y) \ dr
\end{align}
Reasoning analogously for the other player, we obtain
\begin{align}
    \mathbb{E}[\|X^{\mu_y}-X^{\nu_y}\|^2_t + \|Y^{\mu_x}-Y^{\nu_x}\|^2_t] \leq  C_{T} \int_0^t \mathcal{W}_{2, r}^2(\mu_y, \nu_y) \ dr + C_{T} \int_0^t \mathcal{W}_{2, r}^2(\mu_x, \nu_x) \ dr
\end{align}
Given $\mu_y \in \mathcal{P}(C([0,T],\mathcal{Y}))$, define $\Phi_x(\mu_y) = \text{Law}(X^{\mu_y}) \in \mathcal{P}(C([0,T],\mathcal{X}))$, and define $\Phi_y$ analogously. Notice that $\mathcal{W}_{2, t}^2(\Phi_x(\mu_y), \Phi_x(\nu_y)) \leq \mathbb{E}[\|X^{\mu_y}-X^{\nu_y}\|^2_t], \mathcal{W}_{2, t}^2(\Phi_y(\mu_x), \Phi_y(\nu_x)) \leq \mathbb{E}[\|X^{\mu_x}-X^{\nu_x}\|^2_t]$. Hence, 
we obtain 
\begin{align}
    \mathcal{W}_{2, t}^2(\Phi_x(\mu_y),\Phi_x(\nu_y)) + \mathcal{W}_{2, t}^2(\Phi_y(\mu_x), \Phi_y(\nu_x)) \leq C_{T} \int_0^t \mathcal{W}_{2, r}^2(\mu_y, \nu_y) + \mathcal{W}_{2, r}^2(\mu_x, \nu_x) \ dr
\end{align}
Observe that $\mathcal{W}_{2, t}^2(\mu_x,\nu_x) + \mathcal{W}_{2, t}^2(\mu_y,\nu_y)$ is the square of a distance between $(\mu_x,\mu_y)$ and $(\nu_x,\nu_y)$ on $\mathcal{P}(C([0,T],\mathcal{X})) \times \mathcal{P}(C([0,T],\mathcal{Y}))$. Hence, we can apply the Piccard iteration argument to obtain the existence result and another application of Gronwall's inequality yields pathwise uniqueness. 

Uniqueness in law (i.e., regardless of the specific Brownian motions and initialization random variables) follows from the typical uniqueness in law result for SDEs (see Chapter 18 of \citet{kallenberg2002foundations} for example). The idea is that when we solve the SDEs with $W', \bar{W}', \xi', \bar{\xi}'$ plugging in the drift the laws of a solution for $W, \bar{W}, \xi, \bar{\xi}$, the solution has the same laws by uniqueness in law of SDEs. Hence, that new solution solves the coupled McKean-Vlasov for $W', \bar{W}', \xi', \bar{\xi}'$. 

\subsection{Propagation of chaos}
Let $\mu_{x}^n = \frac{1}{n} \sum_{i = 1}^n \delta_{X^{i}}, \mu_{y}^n = \frac{1}{n} \sum_{i = 1}^n \delta_{Y^{i}}$. Using the argument from existence and uniqueness on the $i$-th components of $\mathbf{X}, \mathbf{\tilde{X}}$,
\begin{align}
    \begin{split}
        \mathbb{E}[\|X^{i}-\tilde{X}^{i}\|^2_t] \leq 3(3T + 4) \tilde{L}^2 \mathbb{E}\bigg[\int_0^t (\|X^{i}-\tilde{X}^{i}\|^2_r + \mathcal{W}_2^2(\mu_{y,r}^n,\mu_{y,r})) \ dr\bigg]
    \end{split}
\end{align}
Arguing as before, we obtain
\begin{align}
    \mathbb{E}[\|X^{i}-\tilde{X}^{i}\|^2_t] \leq  C_{T} \mathbb{E}\bigg[\int_0^t \mathcal{W}_{2, r}^2(\mu_y^n, \mu_y) \ dr \bigg]
\end{align}
Let $\nu_x^n = \frac{1}{n} \sum_{i=1}^n \delta_{\tilde{X}^i}$ be the empirical measure of the mean field processes in \eqref{eq:mean_field_syst}. Notice that $\frac{1}{n} \sum_{i=1}^n \delta_{(X^{i},\tilde{X}^i)}$ is a coupling between $\nu_x^n$ and $\mu_x^n$, and so
\begin{align}
    \mathcal{W}_{2, t}^2(\mu_x^n, \nu_x^n) \leq \frac{1}{n} \sum_{i=1}^n \|X^{i}-\tilde{X}^{i}\|^2_t
\end{align}
Thus, we obtain
\begin{align}
    \mathbb{E}[\mathcal{W}_{2, t}^2(\mu_x^n, \nu_x^n)] \leq  C_{T} \mathbb{E}\bigg[\int_0^t \mathcal{W}_{2, r}^2(\mu_y^n, \mu_y) \ dr \bigg]
\end{align}
We use the triangle inequality
\begin{align}
\begin{split}
    \mathbb{E}[\mathcal{W}_{2, t}^2(\mu_x^n, \mu_x)] &\leq 2\mathbb{E}[\mathcal{W}_{2, t}^2(\mu_x^n, \nu_x^n)] + 2\mathbb{E}[\mathcal{W}_{2, t}^2(\nu_x^n, \mu_x)] \\ &\leq
    2 C_{T} \mathbb{E}\bigg[\int_0^t \mathcal{W}_{2, r}^2(\mu_y^n, \mu_y) \ dr \bigg] + 2\mathbb{E}[\mathcal{W}_{2, t}^2(\nu_x^n, \mu_x)]
\end{split}
\end{align}
At this point we follow an analogous procedure for the other player and we end up with
\begin{align}
\begin{split}
    \mathbb{E}[\mathcal{W}_{2, t}^2(\mu_x^n, \mu_x) + \mathcal{W}_{2, t}^2(\mu_y^n, \mu_y)] &\leq
    2 C_{T} \mathbb{E}\bigg[\int_0^t \mathcal{W}_{2, r}^2(\mu_y^n, \mu_y) + \mathcal{W}_{2, r}^2(\mu_x^n, \mu_x) \ dr \bigg] \\ &+ 2\mathbb{E}[\mathcal{W}_{2, t}^2(\nu_x^n, \mu_x) + \mathcal{W}_{2, t}^2(\nu_y^n, \mu_y)]
\end{split}
\end{align}
We use Fubini's theorem and Gronwall's inequality again.
\begin{align}
    \mathbb{E}[\mathcal{W}_{2, t}^2(\mu_x^n, \mu_x) + \mathcal{W}_{2, t}^2(\mu_y^n, \mu_y)] \leq 2 \exp(2 C_{T} T) \mathbb{E}[\mathcal{W}_{2, t}^2(\nu_x^n, \mu_x) + \mathcal{W}_{2, t}^2(\nu_y^n, \mu_y)] 
\end{align}
If we set $t=T$ we get
\begin{align}
    \mathbb{E}[\mathcal{W}_{2}^2(\mu_x^n, \mu_x) + \mathcal{W}_{2}^2(\mu_y^n, \mu_y)] \leq 2 \exp(2 C_{T} T) \mathbb{E}[\mathcal{W}_{2}^2(\nu_x^n, \mu_x) + \mathcal{W}_{2}^2(\nu_y^n, \mu_y)] 
\end{align}
and the factor $\mathbb{E}[\mathcal{W}_{2}^2(\nu_x^n, \mu_x) + \mathcal{W}_{2}^2(\nu_y^n, \mu_y)]$ goes to 0 as $n \rightarrow \infty$ by the law of large numbers (see Corollary 2.14 of \citep{lackermean2018}).

\subsection{Convergence of the Nikaido-Isoda error}
\begin{corollary} \label{cor:conv_ni_LGV}
For $t \in [0,T]$, if $\mu_{x,t}^n, \mu_{x,t}, \mu_{y,t}^n, \mu_{y,t}$ are the marginals of $\mu_{x}^n, \mu_{x}, \mu_{y}^n, \mu_{y}$ at time $t$, we have 
\begin{equation}
    \mathbb{E}[|\text{NI}(\mu_{x,t}^n, \mu_{y,t}^n)-\text{NI}(\mu_{x,t}, \mu_{y,t})|] \xrightarrow{n \rightarrow \infty} 0
\end{equation}
\end{corollary}
\begin{proof}
See \autoref{lem:conv_ni}.
\end{proof}

%% file: app_thm6.tex
\section{Proof of \autoref{thm:convergence_mean_field}(ii)} \label{sec:WFR_prop}
\subsection{Preliminaries}
Define the processes $\mathbf{X} = (X^{1}, \dots, X^{n}), \mathbf{w}_{x} = (w_{x}^{1}, \dots, w_{x}^{n})$ and $\mathbf{Y} = (Y^{1}, \dots, Y^{n}), \mathbf{w}_{y} = (w_{y}^{1}, \dots, w_{y}^{n})$ such that for all $i \in \{1, \dots, n\}$
\begin{align}
\begin{split} \label{eq:particle_system}
    \frac{dX_t^{i}}{dt} &= - \gamma \frac{1}{n}\sum_{j=1}^n w_{y,t}^{j} \nabla_x \ell(X_t^{i}, Y_t^{j}), \quad X_0^{i} = \xi^{i} \sim \mu_{x,0} \\
    \frac{dw_{x,t}^{i}}{dt} &= \alpha \left(- \frac{1}{n} \sum_{j=1}^n w_{y,t}^{j} \ell(X_t^{i}, Y_t^{j}) + \frac{1}{n^2} \sum_{k=1}^n \sum_{j=1}^n w_{y,t}^{j} w_{x,t}^{k} \ell(X_t^{i}, Y_t^{j})\right) w_{x,t}^{i}, \quad w_{x,0}^{i} = 1\\
    \frac{dY_t^{i}}{dt} &= \gamma \frac{1}{n} \sum_{j=1}^n w_{x,t}^{j} \nabla_y \ell(X_t^{j}, Y_t^{i}), \quad Y_0^{i} = \bar{\xi}^{i} \sim \mu_{y,0}\\
    \frac{dw_{y,t}^{i}}{dt} &= \alpha \left( \frac{1}{n} \sum_{j=1}^n w_{x,t}^{j} \ell(X_t^{i}, Y_t^{j}) - \frac{1}{n^2} \sum_{k=1}^n \sum_{j=1}^n w_{y,t}^{j} w_{x,t}^{k} \ell(X_t^{i}, Y_t^{j})\right) w_{x,t}^{i}, \quad w_{y,0}^{i} = 1\\
\end{split}
\end{align}
Let $\nu_{x,t}^n = \frac{1}{n}\sum_{i=1}^n \delta_{(X^{i}_t, w_{x,t}^{i})} \in \mathbb{P}(\mathcal{X} \times \mathbb{R}^+), \nu_{y,t}^n = \frac{1}{n}\sum_{i=1}^n \delta_{(Y^{i}_t, r^{n,i}_{y,t})} \in \mathbb{P}(\mathcal{Y} \times \mathbb{R}^+)$. 
Let $\mu_{x,t}^n = \frac{1}{n}\sum_{i=1}^n w_{x,t}^{i} \delta_{X^{i}_t} \in \mathbb{P}(\mathcal{X}), \mu_{y,t}^n = \frac{1}{n}\sum_{i=1}^n w_{y,t}^{i} \delta_{Y^{i}_t} \in \mathbb{P}(\mathcal{Y})$ be the projections of $\nu_{x,t}^n, \nu_{y,t}^n$. Notice that we have changed the notation with respect to the main text, multiplying $w^i_{x}$ by $n$: now $w^i_{x,0} =1$ and $\sum_i w^i_{x,t} =n, \forall t \geq 0$ instead of $w^i_{x,0} =1/n$ and $\sum_i w^i_{x,t} =1, \forall t \geq 0$. 

Let $h_x, h_y$ be the projection operators, i.e. $h_x \nu_x = \int_{\mathcal{R}^{+}} w_x \nu_x(\cdot, w_x)$. We also define the mean field processes $\mathbf{\tilde{X}}, \mathbf{\tilde{Y}}, \mathbf{\tilde{w}}_x, \mathbf{\tilde{w}}_y$ given component-wise by
\begin{align}
\begin{split} \label{eq:lifted_mean_system}
    \frac{d\tilde{X}_t^{i}}{dt} &= - \gamma \nabla_x \int \ell(\tilde{X}_t^{i},\yy) d\mu_{y,t}, \quad \tilde{X}_0^{i} = \xi^{i} \sim \mu_{x,0} \\
    \frac{d\tilde{w}_{x,t}^{i}}{dt} &= \alpha \left(- \int \ell(\tilde{X}_t^{i},\yy) d\mu_{y,t} + \mathcal{L}(\mu_{x,t},\mu_{y,t}) \right) \tilde{w}_{x,t}^{i}, \quad \tilde{w}_{x,0}^{i} = 1\\
    \frac{d\tilde{Y}_t^{i}}{dt} &= \gamma \nabla_y \int \ell(\xx,\tilde{Y}_t^{i}) d\mu_{x,t}, \quad \tilde{Y}_0^{i} = \bar{\xi}^{i} \sim \mu_{y,0}\\
    \frac{d\tilde{w}_{y,t}^{i}}{dt} &= \alpha \left(\int \ell(\xx,\tilde{Y}_t^{i}) d\mu_{x,t} - \mathcal{L}(\mu_{x,t},\mu_{y,t}) \right) \tilde{w}_{x,t}^{i}, \quad \tilde{w}_{y,0}^{i} = 1\\
    \mu_{x,t} &= h_x \text{Law}(\tilde{X}_t^{i},\tilde{w}_{x,t}^{i}), \quad \mu_{y,t} = h_y \text{Law}(\tilde{Y}_t^{i},\tilde{w}_{y,t}^{i})
\end{split}
\end{align}
for $i$ between 1 and $n$.

\begin{lemma}[Forward Kolmogorov equation] \label{lem:forward_kolmogorov_wfr}
If $\tilde{X},\tilde{w}_x,\tilde{Y},\tilde{w}_y$ is a solution of \eqref{eq:lifted_mean_system} with $n=1$, then its laws $\nu_x, \nu_y$ fulfill \eqref{eq:interacting_flow_wfr2}.
\end{lemma}
\begin{proof}
Let $\psi_x : \mathcal{X} \times \mathbb{R}^{+} \rightarrow \mathbb{R}$. Plug the laws $\nu_x, \nu_y$ of the solution $(\tilde{X},\tilde{w}_x),(\tilde{Y},\tilde{w}_y)$ into the ODE \eqref{eq:lifted_mean_system}. Let $\Phi_{x,t} = (X^{\Phi}_{x,t}, w^{\Phi}_{x,t}) : (\mathcal{X} \times \mathbb{R}^{+}) \rightarrow (\mathcal{X} \times \mathbb{R}^{+})$ denote the flow that maps an initial condition of the ODE \eqref{eq:lifted_mean_system} to the corresponding solution at time $t$. Then, we can write $\nu_{x,t} =  (\Phi_{x,t})_* \nu_{x,0}$, where $(\Phi_{x,t})_*$ is the pushforward. Hence, 
\begin{align}
\begin{split}
    &\frac{d}{dt} \int_{\mathcal{X}\times\mathbb{R}^{+}} \psi_x(\xx,w_x) \ d\nu_{x,t}(\xx,w_x) \\ &= \frac{d}{dt} \int_{\mathcal{X}\times\mathbb{R}^{+}} \psi_x(\Phi_{x,t}(\xx,w_x)) \ d\nu_{x,0}(\xx,w_x) \\ 
    &= \int_{\mathcal{X}\times\mathbb{R}^{+}} \left(\nabla_x \psi_x(\Phi_{x,t}(\xx,w_x)),
    \frac{d\psi_x}{dw_x} (\Phi_{x,t}(\xx,w_x)) \right) \cdot
    \frac{d}{dt} \Phi_{x,t}(\xx,w_x) \ d\nu_{x,0}(\xx,w_x) \\ &= \int_{\mathcal{X}\times\mathbb{R}^{+}} \nabla_x \psi_x(\Phi_{x,t}(\xx,w_x)) \cdot (-\gamma \nabla_x V_x(h_y \nu_{y,t},X^{\Phi}_{x,t})) \\ &+ \frac{d\psi_x}{dw_x} (\Phi_{x,t}(\xx,w_x)) \alpha (-V_x(h_y \nu_{y,t},X^{\Phi}_{x,t}) + \mathcal{L}(h_x \nu_{x,t},h_y \mu_{y,t})) \ d\nu_{x,0}(\xx,w_x) 
\end{split}
\end{align}
And we can identify the right hand side as the weak form of \eqref{eq:interacting_flow_wfr2}, shown in \eqref{eq:weak_IWFR}. The argument for $\nu_y$ is analogous.
\end{proof}

We state a stronger version of the law of large numbers in the first statement of
\autoref{thm:convergence_mean_field}(ii).
\begin{theorem} \label{thm:propagation_chaos_WFR}
There exists a solution of the coupled SDEs \eqref{eq:lifted_mean_system}. Pathwise uniqueness and uniqueness in law hold. Let $\nu_x \in \mathcal{P}(\mathcal{C}([0,T],\mathcal{X}\times \mathbb{R}^{+})), \nu_y \in \mathcal{P}(\mathcal{C}([0,T],\mathcal{Y} \times \mathbb{R}^{+}))$ be the unique laws of the solutions for $n=1$ (all pairs have the same solutions). Then,
\begin{align}
    \mathbb{E}[\mathcal{W}_{2}^2(\nu_x^n, \nu_x) + \mathcal{W}_{2}^2(\nu_y^n, \nu_y)] \xrightarrow{n \rightarrow \infty} 0
\end{align}
\end{theorem} 

\autoref{thm:propagation_chaos_WFR} is the law of large numbers for the WFR dynamics, and its proof follows the same argument of \autoref{thm:propagation_chaos_LGV}. The reason \autoref{thm:propagation_chaos_WFR} implies \autoref{thm:convergence_mean_field}(ii) is analogous to the reason for which \autoref{thm:propagation_chaos_LGV} implies \autoref{thm:convergence_mean_field}(i), with the additional step that $\mathcal{W}_{2}^2(\mu_{x,t}^n, \mu_{x,t}) = \mathcal{W}_{2}^2(h_x \nu_{x,t}^n, h_x \nu_{x,t}) \leq e^{4MT} \mathcal{W}_{2}^2(\nu_{x,t}^n, \nu_{x,t})$, and this inequality is shown in \eqref{eq:wasserstein_proj}. 

\subsection{Existence and uniqueness}
We choose to do an argument close to \citet{sznitmantopics1991} (see \citet{lackermean2018}), which yields convergence of the expectation of the square of the 2-Wasserstein distances between the empirical and the mean field measures.

First, to prove existence and uniqueness of the solution $(\mu_{x,t},\mu_{y,t})$ in the time interval $[0,T]$ for arbitrary $T$, we can use the same argument as in the \autoref{sec:LGV_prop}. Now, instead of \eqref{eq:mean_field_syst2} we have
\begin{align}
    \tilde{X}_t &= \xi - \gamma \int_{0}^t \int_{\mathcal{Y}} \nabla_x \ell(\tilde{X}_s, \yy) \ d\mu_{y,s} \ ds, \\
    \tilde{w}_{x,t} &= 1 + \alpha \int_0^t \left(- \int \ell(\tilde{X}_t,\yy) d\mu_{y,t} + \mathcal{L}(\mu_{x,t},\mu_{y,t}) \right)\tilde{w}_{x,s} \ ds, \\
    \tilde{Y}_t &= \bar{\xi} + \gamma \int_{0}^t \int_{\mathcal{X}} \nabla_y \ell(\xx, \tilde{Y}_s) \ d\mu_{x,s} \ ds, \\
    \tilde{w}_{y,t} &= 1 + \alpha \int_0^t \left(\int \ell(\xx,\tilde{Y}_t) d\mu_{x,t} - \mathcal{L}(\mu_{x,t},\mu_{y,t}) \right) \tilde{w}_{y,s} \ ds, \\
    \mu_{x,t} &= h_x \text{Law}(\tilde{X}_t,\tilde{w}_{x,t}), \quad \mu_{y,t} = h_y \text{Law}(\tilde{Y}_t,\tilde{w}_{y,t}),
\end{align}
where $\xi$ and $\bar{\xi}$ are arbitrary random variables with laws $\mu_{x,0}, \mu_{y,0}$ respectively. For $\xx, \xx' \in \mathcal{X}$, $r, r' \in \mathbb{R}^{+}$, $\mu_x, \mu'_x \in \mathcal{P}(\mathcal{X})$, $\mu_y, \mu'_y \in \mathcal{P}(\mathcal{Y})$, notice that using an argument similar to \eqref{eq:bound_lipschitz_vol} the following bound holds
\begin{align}
\begin{split}
    &\bigg|\left(- \int \ell(\xx,\yy) d\mu_{y} + \mathcal{L}(\mu_{x},\mu_{y}) \right) w - \left(- \int \ell(\xx',\yy) d\mu'_{y} + \mathcal{L}(\mu'_{x},\mu'_{y}) \right) w' \bigg| \\ &\leq 2M |w-w'| + |w'| \tilde{L}(|\xx-\xx'| + 3 \mathcal{W}_1(\nu, \mu)) \leq 2M |w-w'| + |w'| \tilde{L}(|\xx-\xx'| + 3 \mathcal{W}_2(\mu_y, \mu'_y))
\end{split}
\end{align}
\begin{align}
\begin{split}
    \implies &\bigg|\left(- \int \ell(\xx,\yy) d\mu_{y} + \mathcal{L}(\mu_{x},\mu_{y}) \right) r - \left(- \int \ell(\xx',\yy) d\mu'_{y} + \mathcal{L}(\mu'_{x},\mu'_{y}) \right) r' \bigg|^2 \\ &\leq 12M^2 |w-w'|^2 + 3|w'|^2 \tilde{L}^2(|\xx-\xx'|^2 + 9 \mathcal{W}_2^2(\mu_y, \mu'_y))
\end{split}
\end{align}
Recall that $M$ is a bound on the absolute value of $\ell$ and $\tilde{L}$ is the Lipschitz constant of the loss $\ell$. A simple application of Gronwall's inequality shows $|\tilde{w}_{x,t}|$ is bounded by $e^{2MT}$ for all $t \in [0,T]$. Hence, we can write
\begin{align}
    \begin{split}
        \mathbb{E}[\|X^{\mu_y} - X^{\mu'_y}\|^2_t + \|w_x^{\mu_y} - w_x^{\mu'_y}\|^2_t] \leq \gamma^2 t \mathbb{E}\bigg[\int_0^t \bigg|\nabla_x \int \ell(X^{\mu_y}_s,\yy) d\mu_{y,s} - \nabla_x \int \ell(X^{\mu'_y}_s,\yy) d\mu'_{y,s}\bigg|^2 \ ds \bigg] \\ + \alpha^2 t \mathbb{E}\bigg[\int_0^t \bigg|\left(- \int \ell(X^{\mu_y}_s,\yy) d\mu_{y} + \mathcal{L}(\mu_{x},\mu_{y}) \right) w_x^{\mu_y} - \left(- \int \ell(X^{\mu'_y}_s,\yy) d\mu'_{y} + \mathcal{L}(\mu'_{x},\mu'_{y}) \right) w_x^{\mu'_y} \bigg|^2 \ ds \bigg] \\ \leq  K t \mathbb{E}\bigg[ \int_0^t \|X^{\mu_y} - X^{\mu'_y}\|^2_s + \|w^{\mu_y} - w^{\mu'_y}\|^2_s \ ds \bigg] + K' t \mathbb{E}\bigg[ \int_0^t \mathcal{W}_{2}^2(\mu_{y,s}, \mu'_{y,s}) \ ds \bigg],
    \end{split}
\end{align}
where $K = \max \{12 \alpha^2 M^2, 2 L^2 \gamma^2 + 3 \tilde{L}^2 e^{4MT} \alpha^2\}, K' = 2 L^2 \gamma^2 + 27 \tilde{L}^2 e^{4MT} \alpha^2$. Notice that we have used \eqref{eq:lipschitz_transport} as well. This equation is analogous to equation \eqref{eq:before_gronwall}, and upon application of Fubini's theorem and Gronwall's inequality it yields
\begin{align} \label{eq:lifted_gronwall}
    \mathbb{E}[\|X^{\mu_y} - X^{\mu'_y}\|^2_t + \|w_x^{\mu_y} - w_x^{\mu'_y}\|^2_t] \leq T K' \exp(T K) \mathbb{E}\bigg[ \int_0^t \mathcal{W}_{2}^2(\mu_{y,s}, \mu'_{y,s}) \ ds \bigg]
\end{align}

Now we will prove that 
\begin{align} \label{eq:wasserstein_proj}
\mathcal{W}_{2}^2(h_x \nu_{x}, h_x \nu'_{x}) \leq e^{4MT} \mathcal{W}_{2}^2(\nu_{x}, \nu'_{x}),
\end{align}
where $\nu_x, \nu'_x \in \mathcal{P}(\mathcal{X}\times [0,e^{2MT}])$. 
Define the homogeneous projection operator $\tilde{h} : \mathcal{P}((\mathcal{X}\times \mathbb{R}^{+})^2) \rightarrow \mathcal{P}(\mathcal{X}^2)$ as $\forall f \in C(\mathcal{X}^2)$,
\begin{align}
\begin{split}
    \int_{\mathcal{X}^2} f(x,y) \ d(\tilde{h}\pi)(x,y) = \int_{(\mathcal{X}\times [0,e^{2MT}])^2} w_x w_y f(x,y) \ d\pi(x, w_x, y, w_y), \ \forall \pi \in \mathcal{P}((\mathcal{X}\times \mathbb{R}^{+})^2).
\end{split}
\end{align}
Let $\pi$ be a coupling between $h_x \nu_x, h_x \nu'_x$. Then $\tilde{h}\pi$ is a coupling between $h_x \nu_x, h_x \nu'_x$ and
\begin{align}
\begin{split}
    \int_{\mathcal{X}^2} \|x - y\|^2 \ d(\tilde{h}\pi)(x, y) 
    &= \int_{(\mathcal{X}\times [0,e^{2MT}])^2} w_x w_y \|x - y\|^2 \ d\pi(x, w_x, y, w_y) \\ &\leq e^{4MT} \int_{(\mathcal{X}\times [0,e^{2MT}])^2} \|x - y\|^2 \ d\pi(x, w_x, y, w_y) \\&\leq e^{4MT} \int_{(\mathcal{X}\times [0,e^{2MT}])^2} \|x - y\|^2 + |w_x-w_y|^2 \ d\pi'(x, w_x, y, w_y)
\end{split}
\end{align}
Taking the infimum with respect to $\pi$ on both sides we obtain the desired inequality. 

Let $\nu_{x,t} = \text{Law}(X^{\mu_y}_t,w_{x,t}^{\mu_y}), \nu'_{x,t} = \text{Law}(X^{\mu'_y}_t,w_{x,t}^{\mu'_y})$ and recall that $\mu_{x,t} = h_x \nu_{x,t}, \mu'_{x,t} = h_x \nu'_{x,t}$. Given $\nu_y \in \mathcal{P}(C([0,T],\mathcal{Y}\times\mathbb{R}^{+}))$, define $\Phi_x(\nu_y) = \text{Law}(X^{\nu_y},w_x^{\nu_y}) \in \mathcal{P}(C([0,T],\mathcal{X}))$ where we abuse the notation and use $(X^{\nu_y},w_x^{\nu_y})$ to refer to $(X^{\mu_y},w_x^{\mu_y})$. Notice also that
\begin{align} 
\begin{split} \label{eq:wass_mapping}
    \mathcal{W}_{2, t}^2(\Phi_x(\nu_y),\Phi_x(\nu'_y)) &\leq \mathbb{E}\bigg[\sup_{s \in [0,t]} \|X^{\mu_y}_s - X^{\mu'_y}_s\|^2 + \|w^{\mu_y}_{x,s} - w^{\mu'_y}_{x,s}\|^2 \bigg] \\ &\leq \mathbb{E}[\|X^{\mu_y} - X^{\mu'_y}\|^2_t + \|w_x^{\mu_y} - w_x^{\mu'_y}\|^2_t]
\end{split}
\end{align}
We use \eqref{eq:wasserstein_proj} and \eqref{eq:wass_mapping} on \eqref{eq:lifted_gronwall} to conclude
\begin{align}
    \mathcal{W}_{2, t}^2(\Phi_x(\nu_y),\Phi_x(\nu'_y)) \leq T K' \exp(T K) \mathbb{E}\bigg[ \int_0^t \mathcal{W}_{2,s}^2(\nu_{y}, \nu'_{y}) \ ds \bigg]
\end{align}
The rest of the argument is sketched in \autoref{sec:LGV_prop}.

\subsection{Propagation of chaos}
Following the reasoning in the existence and uniqueness proof, we can write
\begin{align}
    \begin{split}
        &\mathbb{E}[\|X^{i} - \tilde{X}^{i}\|^2_t + \|w_{x}^{i} - \tilde{w}_{x}^{i}\|^2_t] \\ &\leq K t \mathbb{E}\bigg[ \int_0^t \|X^{i} - \tilde{X}^{i}\|^2_s + \|w_{x}^{i} - \tilde{w}_{x}^{i}\|^2_s \ ds \bigg] + K' t \mathbb{E}\bigg[ \int_0^t \mathcal{W}_{2}^2(\mu_{y,s}^n, \mu_{y,s}) \ ds \bigg],
    \end{split}
\end{align}
Hence, we obtain
\begin{align}
    \begin{split}
        \mathbb{E}[\|X^{i} - \tilde{X}^{i}\|^2_t + \|w_{x}^{i} - \tilde{w}_{x}^{i}\|^2_t] \leq T K' \exp(T K) \mathbb{E}\bigg[ \int_0^t \mathcal{W}_{2}^2(\mu_{y,s}^n, \mu_{y,s}) \ ds \bigg]
    \end{split}
\end{align}
Let $\tilde{\nu}_{x,t}^n = \frac{1}{n}\sum_{i=1}^n \delta_{(\tilde{X}_t^{i}, \tilde{w}_t^{i})} \in \mathbb{P}(\mathcal{X} \times \mathbb{R}^+)$ be the marginal at time $t$ of the empirical measure of \eqref{eq:particle_system}. As in \autoref{sec:LGV_prop},
\begin{align}
    \mathcal{W}_{2,t}^2(\nu_{x}^n, \tilde{\nu}_{x}^n) \leq \frac{1}{n} \sum_{i=1}^n \sup_{s \in [0,t]} \|X_s^{i} - \tilde{X}_s^{i}\|^2 + |w_{x,s}^{i} - \tilde{w}_{x,s}^{i}|^2 \leq \frac{1}{n} \sum_{i=1}^n \|X^{i} - \tilde{X}^{i}\|^2_t + \|w_{x}^{i} - \tilde{w}_{x}^{i}\|^2_t
\end{align}
which yields 
\begin{align}
\begin{split}
    \mathbb{E}[\mathcal{W}_{2,t}^2(\nu_{x}^n, \tilde{\nu}_{x}^n)] &\leq T K' \exp(T K) \mathbb{E}\bigg[ \int_0^t \mathcal{W}_{2}^2(\mu_{y,s}^n, \mu_{y,s}) \ ds \bigg] \\&\leq T K' \exp((K + 4M)T) \mathbb{E}\bigg[ \int_0^t \mathcal{W}_{2, s}^2(\nu_{y}^n, \nu_{y}) \ ds \bigg]
\end{split}
\end{align}
The second inequality above follows from inequality \eqref{eq:wasserstein_proj}  $\mathcal{W}_{2}^2(\nu_{y,s}^n, \nu_{y,s}) \leq \mathcal{W}_{2,s}^2(\nu_{y}^n, \nu_{y})$. Now we use the triangle inequality as in \autoref{sec:LGV_prop}:
\begin{align}
\begin{split}
    \mathbb{E}[\mathcal{W}_{2,t}^2(\nu_{x}^n, \nu_{x})] &\leq 2\mathbb{E}[\mathcal{W}_{2,t}^2(\nu_{x}^n, \tilde{\nu}_{x}^n)] + 2\mathbb{E}[\mathcal{W}_{2,t}^2(\tilde{\nu}_{x}^n, \nu_{x})] \\ &\leq 2T K' \exp((K + 4M)T) \mathbb{E}\bigg[ \int_0^t \mathcal{W}_{2, s}^2(\nu_{y}^n, \nu_{y}) \ ds \bigg] + 2\mathbb{E}[\mathcal{W}_{2,t}^2(\tilde{\nu}_{x}^n, \nu_{x})]
\end{split}
\end{align}
If we denote $C := 2T K' \exp((K + 4M)T)$ and we make the same developments for the other player, we obtain
\begin{align}
\begin{split}
    \mathbb{E}[\mathcal{W}_{2,t}^2(\nu_{x}^n, \nu_{x}) + \mathcal{W}_{2,t}^2(\nu_{y}^n, \nu_{y})] &\leq C \mathbb{E}\bigg[ \int_0^t \mathcal{W}_{2, s}^2(\nu_{y}^n, \nu_{y}) + \mathcal{W}_{2, s}^2(\nu_{x}^n, \nu_{x}) \ ds \bigg] \\ &+ 2\mathbb{E}[\mathcal{W}_{2,t}^2(\tilde{\nu}_{x}^n, \nu_{x}) + \mathcal{W}_{2,t}^2(\tilde{\nu}_{y}^n, \nu_{y})]
\end{split}
\end{align}
From this point on, the proof works as in \autoref{sec:LGV_prop}.

\subsection{Convergence of the Nikaido-Isoda error} \label{subsec:conv_ni_WFR}
\begin{corollary} \label{cor:conv_ni2}
    For $t \in [0,T]$, let $\bar{\mu}_{x,t}^n = \frac{1}{t} \int_0^t h_x \nu_{x,r}^n \ dr, \bar{\mu}_{x,t} = \frac{1}{t} \int_0^t h_x \nu_{x,r} \ dr$ and define $\bar{\mu}_{y,t}^n, \bar{\mu}_{y,t}$ analogously.
    Then,
    \begin{equation}
        \mathbb{E}[|\text{NI}(\bar{\mu}_{x,t}^n,\bar{\mu}_{y,t}^n)-\text{NI}(\bar{\mu}_{x,t},\bar{\mu}_{y,t})|] \xrightarrow{n \rightarrow \infty} 0
    \end{equation}
\end{corollary}
\begin{proof}
Notice that since the integral over time and the homogeneous projection commute, we have $\bar{\mu}_{x,t}^n = h_x(\frac{1}{t}\int_0^t \nu_{x,r}^n \ dr), \bar{\mu}_{x,t} = h_x (\frac{1}{t}\int_0^t \nu_{x,r} \ dr)$. Since $\frac{1}{t}\int_0^t \nu_{x,r}^n \ dr$ and $\frac{1}{t}\int_0^t \nu_{x,r} \ dr$ belong to $\mathcal{P}(\mathcal{X} \times [0,e^{2MT}])$, \eqref{eq:wasserstein_proj} implies
\begin{align}
    \begin{split}
        \mathcal{W}_{2}^2\left(h_x \left(\frac{1}{t}\int_0^t \nu_{x,r}^n \ dr\right), h_x \left(\frac{1}{t}\int_0^t \nu_{x,r} \ dr\right)\right) \leq e^{4MT} \mathcal{W}_{2}^2\left(\frac{1}{t} \int_0^t \nu_{x,r}^n \ dr, \frac{1}{t} \int_0^t \nu_{x,r} \ dr\right)
    \end{split}
\end{align}
Notice that $\mathcal{W}_{2}^2(\frac{1}{t} \int_0^t \nu_{x,r}^n \ dr, \frac{1}{t} \int_0^t \nu_{x,r} \ dr) \leq \frac{1}{t} \int_0^t \mathcal{W}_{2}^2(\nu_{x,r}^n, \nu_{x,r}) \ dr$. Indeed, 
\begin{align}
\begin{split}
    \mathcal{W}_{2}^2\left(\frac{1}{t} \int_0^t \nu_{x,r}^n \ dr, \frac{1}{t} \int_0^t \nu_{x,r} \ dr\right) &= \max_{\phi \in \Psi_c(\mathcal{X})} \frac{1}{t} \int_0^t \int \phi \ d\nu_{x,r}^n \ dr + \frac{1}{t} \int_0^t \int \phi^c \ d\nu_{x,r}^n \ dr \\&\leq \frac{1}{t} \int_0^t \left(\max_{\phi \in \Psi_c(\mathcal{X})} \int \phi \ d\nu_{x,r}^n + \int \phi^c \ d\nu_{x,r}^n \right) \ dr \\ &= \frac{1}{t} \int_0^t \mathcal{W}_{2}^2(\nu_{x,r}^n, \nu_{x,r}) \ dr
\end{split}
\end{align}
Hence, using the inequality $\mathcal{W}_2^2(\nu_{x,r}^n,\nu_{x,r}) \leq \mathcal{W}_2^2(\nu_{x}^n,\nu_{x})$:
\begin{align}
\begin{split}
    \mathbb{E}\bigg[\mathcal{W}_{2}^2\left(h_x \left(\frac{1}{t} \int_0^t \nu_{x,r}^n \ dr \right), h_x \left(\frac{1}{t} \int_0^t \nu_{x,r} \ dr \right)\right)\bigg] &\leq e^{4MT} \mathbb{E} \bigg[\frac{1}{t} \int_0^t \mathcal{W}_{2}^2(\nu_{x,r}^n, \nu_{x,r}) \ dr \bigg] \\ &\leq e^{4MT} \mathbb{E}[\mathcal{W}_2^2(\nu_{x}^n,\nu_{x})]
\end{split}
\end{align}
Since the right hand side goes to zero as $n \rightarrow \infty$ by \autoref{thm:propagation_chaos_WFR}, we conclude by applying \autoref{lem:conv_ni}.
\end{proof}

\subsection{Hint of the infinitesimal generator approach}
Let $\phi_x : \mathcal{X} \rightarrow \mathbb{R}, \phi_y : \mathcal{Y} \rightarrow \mathbb{R}$ be arbitrary continuously differentiable functions, i.e. $\phi_x \in C^1(\mathcal{X}, \mathbb{R}), \phi_y \in C^1(\mathcal{Y}, \mathbb{R})$. Let us define the operators $\mathcal{L}^{(n)}_{x,t}:  C^1(\mathcal{X}, \mathbb{R}) \rightarrow C^0(\mathcal{X}, \mathbb{R}), \mathcal{L}^{(n)}_{y,t}:  C^1(\mathcal{Y}, \mathbb{R}) \rightarrow C^0(\mathcal{Y}, \mathbb{R})$ as
\begin{align}
\begin{split} \label{eq:inf_ops}
    \mathcal{L}^{(n)}_{x,t} \phi_x (\xx) &=
    - \gamma \nabla_x \int \ell(\xx,\yy) d\mu_{y,t}^n \cdot \nabla_x \phi_x(\xx) + \alpha \left(- \int \ell(\xx,\yy) d\mu_{y,t}^n + \mathcal{L}(\mu_{x,t}^n,\mu_{y,t}^n) \right)  \\
    \mathcal{L}^{(n)}_{y,t} \phi_y (\yy) &= \gamma \nabla_y \int \ell(\xx,\yy) d\mu_{x,t}^n \cdot \nabla_y \phi_y(\xx) + \alpha \left(\int \ell(\xx,\yy) d\mu_{x,t}^n - \mathcal{L}(\mu_{x,t}^n,\mu_{y,t}^n) \right) 
\end{split}
\end{align}

Notice that from \eqref{eq:particle_system} and \eqref{eq:inf_ops}, we have
\begin{align}
\begin{split} \label{eq:weak_particles_x}
    \frac{d}{dt} \int_{\mathcal{X}} \phi_x(\xx) \ d\mu^n_{x,t}(\xx) &= \frac{d}{dt} \int_{\mathcal{X}\times \mathbb{R}^{+}} w_x \phi_x(\xx) \ d\nu^n_{x,t}(\xx,w_x) = \frac{d}{dt} \sum_{i=1}^n w_{x,t}^{i} \phi_x(X_t^{i}) \\&= \sum_{i=1}^n \frac{dw_{x,t}^{i}}{dt} \phi_x(X_t^{i}) + \sum_{i=1}^n w_{x,t}^{i} \nabla_x \phi_x(X_t^{i}) \cdot \frac{dX_t^{i}}{dt} \\&= \int_{\mathcal{X}\times \mathbb{R}^{+}} w_x \mathcal{L}^{(n)}_{x,t} \phi_x (\xx) \ d\nu_{x,t}^n(\xx, w_x) = \int_{\mathcal{X}} \mathcal{L}^{(n)}_{x,t} \phi_x (\xx) \ d\mu_{x,t}^n(\xx)  
\end{split}
\end{align}
The analogous equation holds for $\mu_{x,t}^n$:
\begin{align} \label{eq:weak_particles_y}
    \frac{d}{dt} \int_{\mathcal{Y}} \phi_y(\yy) \ d\mu^n_{y,t}(\yy) = \int_{\mathcal{Y}} \mathcal{L}^{(n)}_{y,t} \phi_y (\yy) \ d\mu_{y,t}^n(\yy)  
\end{align}
Formally taking the limit $n \rightarrow \infty$ on \eqref{eq:weak_particles_x} and \eqref{eq:weak_particles_y} yields
\begin{align}
    \begin{split}
        \frac{d}{dt} \int_{\mathcal{X}} \phi_x(\xx) \ d\mu_{x,t}(\xx) &= \int_{\mathcal{X}} \mathcal{L}_{x,t} \phi_x (\xx) \ d\mu_{x,t}(\xx) \\
        \frac{d}{dt} \int_{\mathcal{Y}} \phi_y(\yy) \ d\mu_{y,t}(\yy) &= \int_{\mathcal{Y}} \mathcal{L}_{y,t} \phi_y (\yy) \ d\mu_{y,t}(\yy),
    \end{split}
\end{align}
where 
\begin{align}
\begin{split} \label{eq:inf_ops_lim}
    \mathcal{L}_{x,t} \phi_x (\xx) &=
    - \gamma \nabla_x \int \ell(\xx,\yy) d\mu_{y,t} \cdot \nabla_x \phi_x(\xx) + \alpha \left(- \int \ell(\xx,\yy) d\mu_{y,t} + \mathcal{L}(\mu_{x,t},\mu_{y,t}) \right)  \\
    \mathcal{L}_{y,t} \phi_y (\yy) &= \gamma \nabla_y \int \ell(\xx,\yy) d\mu_{x,t} \cdot \nabla_y \phi_y(\xx) + \alpha \left(\int \ell(\xx,\yy) d\mu_{x,t} - \mathcal{L}(\mu_{x,t},\mu_{y,t}) \right) 
\end{split}
\end{align}
and $\mu_{x,0}, \mu_{y,0}$ are set as in \eqref{eq:particle_system}.

To make the limit $n \rightarrow \infty$ rigorous, an argument analogous to Theorem 2.6 of \citet{chizat2018global} would result in almost sure convergence of the 2-Wasserstein distances between the empirical and the mean field measures. In our case almost sure convergence of the squared distance implies convergence of the expectation of the squared distance through dominated convergence, and hence the almost sure convergence result is stronger. Nonetheless, such an argument would require proving uniqueness of the mean field measure PDE through some notion of geodesic convexity, which is not clear in our case. 

%% file: app_misc.tex
\newpage
\section{Auxiliary material}
\subsection{$\epsilon$-Nash equilibria and the Nikaido-Isoda error}
\label{sec:epsilon_nash_NI}
Recall that an $\epsilon$-NE $(\mu_x, \mu_y)$ satisfies $\forall \mu_x^* \in \mathcal{P}(\mathcal{X}),\; \mathcal{L}(\mu_x,\mu_y)\leq \mathcal{L}(\mu_x^*,\mu_y) + \epsilon$ and $ \forall \mu_y^* \in \mathcal{P}(\mathcal{Y}),\; \mathcal{L}(\mu_x,\mu_y) \geq \mathcal{L}(\mu_x,\mu_y^*) - \epsilon$. That is, each player can improve its value by at most $\epsilon$ by deviating from the equilibrium strategy, supposing that the other player is kept fixed. 

Recall the Nikaido-Isoda error defined in \eqref{eq:nikaido_def}. This equation can be rewritten as:
\begin{align}
    \mathrm{NI}(\mu_x,\mu_y) = \sup_{\mu^{*}_y \in \mathcal{P}(\mathcal{Y})} \mathcal{L}(\mu_x,\mu^{*}_y) - \mathcal{L}(\mu_x,\mu_y) +\mathcal{L}(\mu_x,\mu_y) -\inf_{\mu^{*}_x \in \mathcal{P}(\mathcal{X})} \mathcal{L}(\mu^{*}_x,\mu_y)~.
\end{align}
The terms $\sup_{\mu^{*}_y \in \mathcal{P}(\mathcal{Y})} \mathcal{L}(\mu_x,\mu^{*}_y) - \mathcal{L}(\mu_x,\mu_y) > 0$ measure how much player $y$ can improve its value by deviating from $\mu_y$ while $\mu_x$ stays fixed. Analogously, the terms $\mathcal{L}(\mu_x,\mu_y) -\inf_{\mu^{*}_x \in \mathcal{P}(\mathcal{X})} \mathcal{L}(\mu^{*}_x,\mu_y) > 0$ measure how much player $x$ can improve its value by deviating from $\mu_x$ while $\mu_y$ stays fixed.

Notice that
\begin{align}
\begin{split}
    \forall \mu_x^* \in \mathcal{P}(\mathcal{X}),\; \mathcal{L}(\mu_x,\mu_y)\leq \mathcal{L}(\mu_x^*,\mu_y) + \epsilon &\iff \mathcal{L}(\mu_x,\mu_y) -\inf_{\mu^{*}_x \in \mathcal{P}(\mathcal{X})} \mathcal{L}(\mu^{*}_x,\mu_y) \leq \epsilon \\
    \forall \mu_y^* \in \mathcal{P}(\mathcal{Y}),\; \mathcal{L}(\mu_x,\mu_y) \geq \mathcal{L}(\mu_x,\mu_y^*) - \epsilon &\iff \sup_{\mu^{*}_y \in \mathcal{P}(\mathcal{Y})} \mathcal{L}(\mu_x,\mu^{*}_y) - \mathcal{L}(\mu_x,\mu_y) \leq \epsilon
\end{split}
\end{align}
Thus, an $\epsilon$-Nash equilibrium $(\mu_x,\mu_y)$ fulfills $\text{NI}(\mu_x,\mu_y) \leq 2\epsilon$, and any pair $(\mu_x,\mu_y)$ such that $\text{NI}(\mu_x,\mu_y) \leq \epsilon$ is an $\epsilon$-Nash equilibrium.

\subsection{Example: failure of the Interacting Wasserstein Gradient Flow} \label{subsec:failure_IWGF}

Let us consider the polynomial $f(x) = 5x^4 + 10x^2 - 2x$, which is an asymmetric double well as shown in \autoref{fig:f(x)}.
\begin{figure}[h] 
\begin{center}
\begin{tikzpicture} 
  \begin{axis}[ 
    xlabel=$x$,
    ylabel={$f(x) = 5x^4 - 10x^2 - 2x$}
  ] 
  \addplot [
    domain=-1.5:1.5, 
    samples=100, 
    color=blue,
    ]
    {5*x^4 - 10*x^2 - 2*x};
  \end{axis}
\end{tikzpicture}
\caption{Plot of the function $f(x) = 5x^4 + 10x^2 - 2x$.} \label{fig:f(x)}
\end{center}
\end{figure}
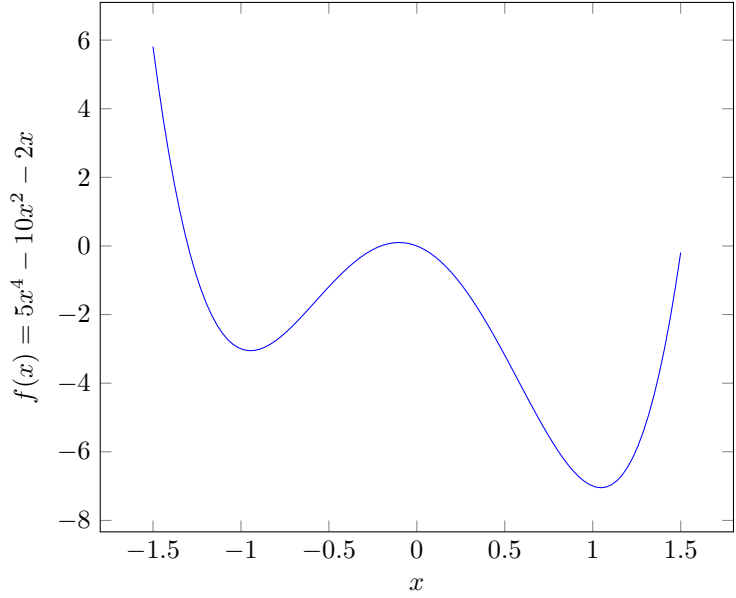

Let us define the loss $\ell : \mathbb{R} \times \mathbb{R} \rightarrow \mathbb{R}$ as $\ell(x,y) = f(x) - f(y)$. That is, the two players are non-interacting and hence we obtain $V_x(x,\mu_y) = f(x) + K$, $V_y(y,\mu_x) = -f(y) + K'$. This means that the IWGF in equation \eqref{eq:interacting_flow} becomes two independent Wasserstein Gradient Flows
\begin{equation}
    \begin{split} 
        \partial_t \mu_x = \nabla \cdot (\mu_x f'(x)), \quad \mu_x(0) = \mu_{x,0},\\
        \partial_t \mu_y = -\nabla \cdot (\mu_y f'(y)), \quad \mu_y(0) = \mu_{y,0}.
    \end{split}
\end{equation}
The particle flows in \eqref{eq:IWGF_particle_flow} become
\begin{equation}
    \frac{dx_{i}}{dt} = - f'(x_{i}), \quad \frac{dy_{i}}{dt} = f'(y_i).
\end{equation}
That is, the particles of player $x$ follow the gradient flow of $f$ and the particles of player $y$ follow the gradient flow of $-f$. It is clear from \autoref{fig:f(x)} that if the initializations $x_{0,i}, y_{0,i}$ are on the left of the barrier, they will not end up in the global minimum $f$ (resp., the global maximum of $-f$). And in this case, the pair of measures supported on the global minimum of $f$ is the only (pure) Nash equilibrium.

The game given by $\ell$ does not fall exactly in the framework that we describe in this work because $\ell$ is not defined on compact spaces. However, it is easy to construct very similar continuously differentiable functions on compact spaces that display the same behavior. 

\subsection{Link between Interacting Wasserstein Gradient Flow and interacting particle gradient flows} \label{subsec:link_IWGF}
Recall \eqref{eq:IWGF_particle_flow}: 
\begin{equation} \label{eq:IWGF_particle_flow2}
    \frac{dx_{i}}{dt} = - \frac{1}{n} \sum_{j=1}^{n} \nabla_x \ell(x_i, y_j), \quad \frac{dy_{i}}{dt} = \frac{1}{n} \sum_{j=1}^{n} \nabla_x \ell(x_j, y_i). \\
\end{equation}
Let $\Phi_t = (\Phi_{x,t}, \Phi_{y,t}) : \mathcal{X}^n \times \mathcal{Y}^n \rightarrow \mathcal{X}^n \times \mathcal{Y}^n$ be the flow mapping initial conditions $\mathbf{X}_0 = {(x_{i,0})}_{i \in [1:n]}, \mathbf{Y}_0 = {(y_{i,0})}_{i \in [1:n]}$ to the solution of \eqref{eq:IWGF_particle_flow}. Let $\mu_{x,t}^n = \frac{1}{n} \sum_{i=1}^n \delta_{\Phi^{(i)}_{x,t}(\mathbf{X}_0,\mathbf{Y}_0)}, \mu_{y,t}^n = \frac{1}{n} \sum_{i=1}^n \delta_{\Phi^{(i)}_{y,t}(\mathbf{X}_0,\mathbf{Y}_0)}$. For all $\psi_x \in \mathcal{C}(\mathcal{X})$, 
\begin{align}
\begin{split}
    \frac{d}{dt} \int_{\mathcal{X}} \psi_x(\xx) \ d\mu_{x,t}^n(\xx) &= \frac{1}{n} \sum_{i=1}^n \frac{d}{dt}\psi_x(\Phi^{(i)}_{x,t}(\mathbf{X}_0,\mathbf{Y}_0)) \\ &= \frac{1}{n} \sum_{i=1}^n \nabla_x \psi_x(\Phi^{(i)}_{x,t}(\mathbf{X}_0,\mathbf{Y}_0)) \cdot \left(- \frac{1}{n} \sum_{j=1}^{n} \nabla_x \ell(\Phi^{(i)}_{x,t}(\mathbf{X}_0,\mathbf{Y}_0), \Phi^{(j)}_{y,t}(\mathbf{X}_0,\mathbf{Y}_0))\right) \\ &= -\frac{1}{n} \sum_{i=1}^n \nabla_x \psi_x(\Phi^{(i)}_{x,t}(\mathbf{X}_0,\mathbf{Y}_0)) \cdot \nabla_x V_x(\mu_{y,t}^n, \Phi^{(i)}_{x,t}(\mathbf{X}_0,\mathbf{Y}_0)) \\ &= -\int_{\mathcal{X}} \nabla_x \psi_x(\xx) \cdot \nabla_x V_x(\mu_{y,t}^n, \xx) \ d\mu_{x,t}^n(x),
\end{split}
\end{align}
which is the first line of \eqref{eq:interacting_flow}. The second line follows analogously.

\subsection{Minimax problems and Stackelberg equilibria} \label{subsec:minimax}
Several machine learning problems, including GANs, are framed as a minimax problem 
\begin{equation}\label{eq:minmax_standard2}
    \min_{x\in  \mathcal{X}}\max_{y\in\mathcal{Y}}\;\ell(x,y).
\end{equation}
A minimax point (also known as a Stackelberg equilibrium or sequential equilibrium) is a pair $(\tilde{x},\tilde{y})$ at which the minimum and maximum of the problem \label{eq:minmax_standard} are attained, i.e.
\begin{equation}
\begin{split}
    \begin{cases}
    \min_{x \in \mathcal{X}}\max_{y \in \mathcal{Y}}\ell(x,y) = \max_{y\in \mathcal{Y}} \ell(\tilde{x},y)\\
    \max_{y\in \mathcal{Y}}\ell(\tilde{x},y) = \ell (\tilde{x},\tilde{y})
    \end{cases}.
\end{split}
\end{equation}
We consider the lifted version of the minimax problem~\eqref{eq:minmax_standard} in the space of probability measures.
\begin{equation}\label{eq:minmax_lifted}
    \min_{\mu_x\in \mathcal{P}(\mathcal{X})}\max_{\mu_y\in \mathcal{P}(\mathcal{Y})} \mathcal{L}(\mu_x,\mu_y). 
\end{equation}
By the generalized Von Neumann's minimax theorem, a Nash equilibrium of the game given by $\mathcal{L}$ is a solution of the lifted minimax problem~\eqref{eq:minmax_lifted} (see \autoref{lem:nash_stackelberg} in the case $\epsilon = 0$).

The converse is not true:
minimax points (solutions of \eqref{eq:minmax_lifted}) are not necessarily mixed Nash equilibria even in the case where the loss function is convex-concave. An example is $\mathcal{L} : \mathbb{R} \times \mathbb{R} \rightarrow \mathbb{R}$ given by $\mathcal{L}(\mu_x,\mu_y) = \iint (x^2 + 2xy) \ d\mu_x \ d\mu_y$. Let $\mathcal{M}$ be the set of measures $\mu \in \mathcal{P}(\mathbb{R})$ such that $\int x \ d\mu = 0$.
Notice that any pair $(\delta_0,\mu_y)$ with $\mu_y \in \mathcal{P}(\mathbb{R})$ is a minimax point. That is because 
\begin{align}
    \max_{\mu_y \in \mathcal{P}(\mathbb{R})} \mathcal{L}(\mu_x,\mu_y) =
    \begin{cases}
    +\infty &\text{if } \mu_x \notin \mathcal{M} \\
    \text{positive} &\text{if } \mu_x \in \mathcal{M} \setminus \{\delta_0\}\\
    0 &\text{if } \mu_x = \delta_0,
    \end{cases}
\end{align}
and hence $\delta_0 = \argmin_{\mu_x \in \mathcal{P}(\mathbb{R})} \max_{\mu_y \in \mathcal{P}(\mathbb{R})} \mathcal{L}(\mu_x,\mu_y)$. But if $\mu_x =\delta_0$, we have $\argmax_{\mu_y \in \mathcal{P}(\mathbb{R})} \mathcal{L}(\mu_x,\mu_y) = \mathcal{P}(\mathbb{R})$, because for all measures $\mu_y \in \mathcal{P}(\mathbb{R})$, $\mathcal{L}(\delta_0,\mu_y) = 0$. However, for $\mu_y \notin \mathcal{M}$, $\mathcal{L}(\mu_x,\mu_y)$ as a function of $\mu_x$ does not have a minimum at $\delta_0$, but at $\delta_{-\int y \ d\mu_y}$. Hence, the only mixed Nash equilibria are of the form $(\delta_0, \mu_y)$, with $\mu_y \in \mathcal{M}$.

The intuition behind the counterexample is that minimax points only require the minimizing player to be non-exploitable, but the maximizing player is only subject to a weaker condition. 

We define a $\varepsilon$-minimax point (or $\epsilon$-Stackelberg equilibrium) of an objective $\mathcal{L}(\mu_x,\mu_y)$ as a couple $(\tilde{\mu}_x,\tilde{\mu}_y)$ such that
\begin{equation}
\begin{split}
    \begin{cases}
    \min_{\mu_x\in \mathcal{P}(\mathcal{X})}\max_{\mu_y\in \mathcal{P}(\mathcal{Y})}\mathcal{L}(\mu_x,\mu_y)\geq\max_{\mu_y\in \mathcal{P}(\mathcal{Y})} \mathcal{L}(\tilde{\mu}_x,\mu_y)-\varepsilon\\
    \max_{\mu_y\in \mathcal{P}(\mathcal{Y})}\mathcal{L}(\tilde{\mu}_x,\mu_y) \leq \mathcal{L}(\tilde{\mu}_x,\tilde{\mu}_y) + \varepsilon
    \end{cases}.
\end{split}
\end{equation}

\begin{lemma} \label{lem:nash_stackelberg}
An $\varepsilon$-Nash equilibrium is a $2\varepsilon$-minimax point, and it holds that 
\begin{align}
\min_{\mu_x \in \mathcal{P}(\mathcal{X})} \max_{\mu_y \in \mathcal{P}(\mathcal{Y})} \mathcal{L}(\mu_x,\mu_y) - \epsilon \leq
\mathcal{L}(\hmu_x,\hmu_y) \leq \max_{\mu_y \in \mathcal{P}(\mathcal{Y})} \min_{\mu_x \in \mathcal{P}(\mathcal{X})} \mathcal{L}(\mu_x,\hmu_y) + \varepsilon \end{align}
\end{lemma}
\begin{proof}
Let $(\hmu_x, \hmu_y)$ be an $\varepsilon$-Nash equilibrium. Notice that $\max_{\mu_y \in \mathcal{P}(\mathcal{Y})} \min_{\mu_x \in \mathcal{P}(\mathcal{X})} \mathcal{L}(\tilde{\mu}_x,\mu_y) \leq \min_{\mu_x \in \mathcal{P}(\mathcal{X})} \max_{\mu_y \in \mathcal{P}(\mathcal{Y})} \mathcal{L}(\tilde{\mu}_x,\mu_y)$. Also,
\begin{align}
    \begin{split} \label{eq:chain_ineq}
        \min_{\mu_x \in \mathcal{P}(\mathcal{X})} \max_{\mu_y \in \mathcal{P}(\mathcal{Y})} \mathcal{L}(\mu_x,\mu_y) &\leq \max_{\mu_y \in \mathcal{P}(\mathcal{Y})} \mathcal{L}(\hat{\mu}_x,\mu_y) \leq \mathcal{L}(\hmu_x,\hmu_y) + \varepsilon \leq \min_{\mu_x \in \mathcal{P}(\mathcal{X})} \mathcal{L}(\mu_x,\hmu_y) + 2\varepsilon\\
        &\leq \max_{\mu_y \in \mathcal{P}(\mathcal{Y})} \min_{\mu_x \in \mathcal{P}(\mathcal{X})} \mathcal{L}(\mu_x,\hmu_y) + 2\varepsilon
    \end{split}
\end{align}
and this yields the chain of inequalities in the statement of the theorem.
The condition $\max_{\mu_y\in \mathcal{P}(\mathcal{Y})}\mathcal{L}(\tilde{\mu}_x,\mu_y) \leq \mathcal{L}(\tilde{\mu}_x,\tilde{\mu}_y) + \varepsilon$ of the definition of $\varepsilon$-minimax point follows directly from the definition of an $\epsilon$-Nash equilibrium. Using part of \eqref{eq:chain_ineq}, we get
\begin{align}
    \max_{\mu_y \in \mathcal{P}(\mathcal{Y})} \mathcal{L}(\hat{\mu}_x,\mu_y) - 2\epsilon \leq \max_{\mu_y \in \mathcal{P}(\mathcal{Y})} \min_{\mu_x \in \mathcal{P}(\mathcal{X})} \mathcal{L}(\mu_x,\hmu_y) \leq \min_{\mu_x \in \mathcal{P}(\mathcal{X})} \max_{\mu_y \in \mathcal{P}(\mathcal{Y})} \mathcal{L}(\tilde{\mu}_x,\mu_y), 
\end{align}
which is the first condition of a $2\epsilon$-minimax.
\end{proof}
\autoref{lem:nash_stackelberg} provides the link between approximate Nash equilibria and approximate Stackelberg equilibria, and it allows to translate our convergence results into minimax problems such as GANs.  

%% file: app_SDEs_manifolds.tex
\subsection{It\^{o} SDEs on Riemannian manifolds: a parametric approach} \label{sec:sdes_manifolds}
We provide a brief summary on how to deal with SDEs on Riemannian manifolds and their corresponding Fokker-Planck equations (see Chapter 8 of \citet{chirikjianstochastic}). While ODEs have a straightforward translation into manifolds, the same is not true for SDEs. Recall that the definitions of the gradient and divergence for Riemannian manifolds are
\begin{align}
    \nabla \cdot X = |g|^{-1/2} \partial_i (|g|^{1/2} X^i), \quad (\nabla f)^i = g^{ij} \partial_j f,
\end{align}
where $g_{ij}$ is the metric tensor, $g^{ij} = (g_{ij})^{-1}$ and $|g| = \text{det}(g_{ij})$. We use the Einstein convention for summing repeated indices. 

The parametric approach to SDEs in manifolds is to define the SDE for the variables $\mathbf{q} = (q_1,\cdots,q_d)$ of a patch of the manifold:
\begin{align} \label{eq:SDE_patch}
    d\mathbf{q} = \mathbf{h}(\mathbf{q},t) dt + H(\mathbf{q},t) d\mathbf{w}.
\end{align}
The corresponding forward Kolmogorov equation is
\begin{align} \label{eq:fokker_planck_manifold}
    \frac{\partial f}{\partial t} + |g|^{-1/2} \sum_{i=1}^d \frac{\partial}{\partial q_i} \left(|g|^{1/2} h_i f\right) = \frac{1}{2} |g|^{-1/2} \sum_{i,j=1}^d \frac{\partial^2}{\partial q_i \partial q_j} \left(|g|^{1/2} \sum_{k=1}^D H_{ik} H_{kj}^\top f \right),
\end{align}
which is to be understood in the weak form.


Assume that the manifold $\mathcal{M}$ embedded in $\mathbb{R}^D$. If $\phi : \mathcal{U}_{\mathbb{R}^d} \subseteq \mathbb{R}^d \rightarrow \mathcal{U} \subseteq \mathcal{M} \subseteq \mathbb{R}^D$ is the mapping corresponding to the patch $\mathcal{U}$ and \eqref{eq:SDE_patch} is defined on $\mathcal{U}_{\mathbb{R}^d}$, let us set $H(\mathbf{q}) = (D\phi(\mathbf{q}))^{-1}$. In this case, $\sum_k H_{ik} H_{kj}^\top = \sum_k (D\phi)^{-1}_{ik} ((D\phi)^{-1}_{kj})^\top = g^{ij}(\mathbf{q})$. Hence, the right hand side of \eqref{eq:fokker_planck_manifold} becomes
\begin{align}
\begin{split}
    &\frac{1}{2} |g|^{-1/2} \sum_{i,j=1}^d \frac{\partial^2}{\partial q_i \partial q_j} \left(|g|^{1/2} g^{ij} f \right) \\ &= |g|^{-1/2} \sum_{i=1}^d \frac{\partial}{\partial q_i} \left(|g|^{1/2} \tilde{h}_i f\right) + \frac{1}{2} |g|^{-1/2} \sum_{i,j=1}^d \frac{\partial}{\partial q_i} \left(|g|^{1/2} g^{ij} \frac{\partial}{\partial q_j}f \right) \\&= |g|^{-1/2} \sum_{i=1}^d \frac{\partial}{\partial q_i} \left(|g|^{1/2} \tilde{h}_i f\right) + \frac{1}{2} |g|^{-1/2} \sum_{i,j=1}^d \frac{\partial}{\partial q_i} \left(|g|^{1/2} g^{ij} \frac{\partial}{\partial q_j}f \right) \\&= \nabla \cdot (\mathbf{\tilde{h}} f) + \frac{1}{2}\nabla \cdot \nabla f
\end{split}
\end{align}
where
\begin{align}
    \tilde{h}_i(\mathbf{q}) = \frac{1}{2} \sum_{j=1}^d \left( |g(\mathbf{q})|^{-1/2} g^{ij}(\mathbf{q}) \frac{\partial |G(\mathbf{q})|^{1/2}}{\partial q_j} + \frac{\partial g^{ij}(\mathbf{q})}{\partial q_j} \right)
\end{align}
Hence, we can rewrite \eqref{eq:fokker_planck_manifold} as
\begin{align}
    \frac{\partial f}{\partial t} = \nabla \cdot ((-\mathbf{h} + \mathbf{\tilde{h}}) f) + \frac{1}{2}\nabla \cdot \nabla f
\end{align}
For this equation to be a Fokker-Planck equation with potential $E$ (i.e. with a Gibbs equilibrium solution), we need $-\mathbf{h} + \mathbf{\tilde{h}} = \nabla E$, which implies $\mathbf{h} = -\nabla E + \mathbf{\tilde{h}}$. 

We can convert an SDE in parametric form like \eqref{eq:SDE_patch} into an SDE on $\mathbb{R}^D$ by using Ito's lemma on $X = \phi(\mathbf{q})$:
\begin{align} \label{eq:ito_dxi}
    dX_i = d\phi_i(\mathbf{q}) = \left(D\phi_i(\mathbf{q}) \mathbf{h}(\mathbf{q}) + \frac{1}{2} \text{Tr}(H(\mathbf{q},t)^\top (H\phi_i)(\mathbf{q}) H(\mathbf{q},t)) \right) dt + D\phi_i(\mathbf{q})H(\mathbf{q},t)
    d\mathbf{w}  
\end{align}
If we set $H(\mathbf{q}) = (D\phi(\mathbf{q}))^{-1}$ as before, $D\phi(\mathbf{q})H(\mathbf{q},t)$ is the projection onto the tangent space of the manifold, i.e. $D\phi(\mathbf{q})H(\mathbf{q},t) v = \text{Proj}_{T_{\phi(\mathbf{q})}M} v, \ \forall v \in \mathbb{R}^D$. 
In the case $\mathbf{h} = \nabla E + \mathbf{\tilde{h}}$, $D\phi_i(\mathbf{q}) \mathbf{h}(\mathbf{q}) = D\phi_i(\mathbf{q}) \nabla E (\mathbf{q}) + D\phi_i(\mathbf{q}) \mathbf{\tilde{h}}(\mathbf{q})$. It is very convenient to abuse the notation and denote $D\phi(\mathbf{q}) \nabla E (\mathbf{q})$ by $\nabla E (\phi(\mathbf{q}))$. We also use $\mathbf{\hat{h}}(\phi(\mathbf{q})) := D\phi(\mathbf{q}) \mathbf{\tilde{h}}(\mathbf{q}) + \frac{1}{2} \text{Tr}(((D\phi(\mathbf{q}))^{-1})^\top (H\phi)(\mathbf{q}) (D\phi(\mathbf{q}))^{-1})$. Both definitions are well-defined because the variables are invariant by changes of coordinates. Hence, under these assumptions \eqref{eq:ito_dxi} becomes
\begin{align} \label{eq:SDE_manifold_embedded}
    dX = (-\nabla E (X) + \mathbf{\hat{h}}(X)) \ dt + \text{Proj}_{T_X M} (d\mathbf{w})
\end{align}
In short that means that we can treat SDEs on embedded manifolds as SDEs on the ambient space by projecting the Brownian motions to the tangent space and adding a drift term $\mathbf{\hat{h}}$ that depends on the geometry of the manifold. Notice that for ODEs on manifolds the additional drift term does not appear and \eqref{eq:SDE_manifold_embedded} reads simply $dX = \nabla E(X) dt$.  

Notice that the forward Kolmogorov equation for \eqref{eq:SDE_manifold_embedded} on $\mathbb{R}^D$ reads
\begin{align} \label{eq:KFP_manifold}
    \frac{d}{dt} \int f(x) \ d\mu_t(x) = \int (\nabla E (x) - \mathbf{\hat{h}}(x)) \cdot \nabla_x f(x) + \frac{1}{2}\text{Tr}((\text{Proj}_{T_x M})^\top H f(x) \text{Proj}_{T_x M})\ d\mu_t(x),
\end{align}
for an arbitrary $f$.